\documentclass{article}
\usepackage{iclr2026_conference,times}

\usepackage{amsmath,amsfonts,bm}

\def\eqref#1{equation~\ref{#1}}

\def\1{\bm{1}}

\DeclareMathAlphabet{\mathsfit}{\encodingdefault}{\sfdefault}{m}{sl}
\SetMathAlphabet{\mathsfit}{bold}{\encodingdefault}{\sfdefault}{bx}{n}

\usepackage{tabularx}
\usepackage{hyperref}
\usepackage{url}
\usepackage{amsmath,amssymb,amsthm}
\usepackage{graphicx}
\usepackage{listings}
\usepackage{float}
\usepackage{caption}
\usepackage{subcaption}
\usepackage{booktabs}
\usepackage{multirow}
\usepackage{array}
\usepackage{adjustbox}
\usepackage{threeparttable}
\usepackage{algorithmicx}
\usepackage{algpseudocode}
\usepackage{amsthm} 
\usepackage{todonotes}
\newtheorem{theorem}{Theorem}[section]
\newtheorem{lemma}{Lemma}
\newtheorem{proposition}{Proposition}[section]

\usepackage{xcolor}
\newcommand{\rev}[1]{\textcolor{black}{#1}}

\title{Mapping Semantic \& Syntactic Relationships with Geometric Rotation}

\author{Michael Freenor \& Lauren Alvarez   \\
TELUS Digital\\
\texttt{\{michael.freenor,lauren.alvarez\}@telusdigital.com}
}

\iclrfinalcopy 
\begin{document}

\maketitle

\begin{abstract}
Understanding how language and embedding models  encode semantic relationships is fundamental to model interpretability. 
While early word embeddings exhibited intuitive vector arithmetic (``king'' - ``man'' + ``woman'' = ``queen''), modern high-dimensional text representations lack straightforward interpretable geometric properties. 
We introduce Rotor-Invariant Shift Estimation (RISE), a geometric approach that represents semantic-syntactic transformations as consistent rotational operations in embedding space, leveraging the manifold structure of modern language representations. 
RISE operations have the ability to operate across both languages and models without reducing performance, suggesting the existence of analogous cross-lingual geometric structure.
We compare and evaluate RISE using two baseline methods, three embedding models, three datasets, and seven morphologically diverse languages in five major language groups. 
Our results demonstrate that RISE consistently maps discourse-level semantic-syntactic transformations with distinct grammatical features (e.g., negation and conditionality) across languages and models.
This work provides the first demonstration that discourse-level semantic-syntactic transformations correspond to consistent geometric operations in multilingual embedding spaces, empirically supporting the linear representation hypothesis at the sentence level.

\end{abstract}

\section{Introduction}
Understanding how contemporary language models encode and manipulate semantic knowledge has become a central challenge in deep learning interpretability. 
The ability to interpret (probe) and control (steer) these internal representations is fundamental to developing trustworthy, safe AI systems. 
In word2vec \citep{mikolov2013efficient} and similar models, semantic relationships could be captured with simple vector arithmetic in the embedding space (i.e. the famous ``king'' - ``man'' + ``woman'' = ``queen'' analogy).
This \rev{linear} transparency offered both interpretability and controllability, enabling researchers to navigate semantic space through intuitive mathematical operations.

However, this clarity has largely disappeared in modern transformer-based language models.
While large language models (LLMs) have achieved remarkable performance across diverse language tasks \citep{achiam2023gpt,touvron2023llama}, their internal workings remain largely opaque \citep{elhage2022toy,rogers2021primer}, limiting our ability to understand, predict, and control their behavior in critical applications.
Unlike the interpretable, \rev{linear}  directions found in static word embeddings, the geometry of modern text representations lacks the same straightforward correspondence to semantic operations. 
This opacity poses significant challenges for understanding how these models organize linguistic knowledge and limits our ability to \rev{interpret} their behavior in principled ways.

The central challenge lies in identifying which geometric operations correspond to meaningful semantic transformations in these complex representation spaces. 
Current approaches often rely on task-specific \textit{probes} \citep{rogers2021primer,hewitt2019structural,alain2017understanding} or \textit{steering vectors} \citep{zou2023representation, wang2023concept,turner2023activation,merullo2024language,trager2023linear}, but lack generalizable frameworks for systematically mapping semantic relationships to geometric structure. 
Without such principled methods, we cannot determine whether the geometric regularities that made static word embeddings interpretable persist in modern language or embedding models, albeit in more complex forms.

We address this gap by introducing Rotor-Invariant Shift Estimation (RISE), a geometric approach that represents \rev{semantic-syntactic} transformations as consistent rotational operations in embedding space, leveraging the manifold structure of modern language representations. RISE is a rotor-based alignment method that identifies cross-lingual and cross-model geometric transformations. \rev{Specifically, we demonstrate how RISE identifies  
three discourse-level semantic-syntactic changes (negation, conditionality, and politeness) across seven morphologically distinct languages and generalizes across three different embedding model architectures.
\textbf{The goal of this study is to develop a framework for identifying discourse-level semantic-syntactic changes that correspond to consistent geometric transformations, and determine how well these transformations can be cross-lingually mapped across model architectures.}} 
Our approach treats \rev{semantic-syntactic} transformations as rotations on the unit hypersphere, where sentence embeddings reside, enabling us to align different linguistic contexts into a common geometric framework. 
This paper presents evidence that certain \rev{semantic-syntactic} transformations exhibit generalizable geometric structure while others vary based on context-dependence, extending the linear representation hypothesis to cross-lingual discourse. 
We demonstrate this through empirical experiments across two baselines, three models, and seven languages -- revealing that negation, conditionality, and politeness transformations can be captured as consistent rotational operations \footnote{The link to our GitHub repository is \url{https://github.com/fuelix/RISE-steering}.}.

\section{Related Work}

\subsection{Linear Representation Hypothesis}

The linear representation hypothesis (LRH), or linear subspace hypothesis, has emerged as a promising theory for bridging the interpretability gap for embeddings \citep{mikolov2013linguistic,levy2014linguistic,bolukbasi2016man,ethayarajh2019contextual,parklinear,park2025}. 
The LRH posits that semantic concepts are encoded as linear structures within embedding spaces, meaning linear algebraic operations can be used for interpretation and control (e.g., ``king" - ``man" + ``woman" = ``queen'' presented  by \citet{mikolov2013linguistic}).
\citet{parklinear} formalized the LRH by unifying three distinct notions of linearity that had developed independently across the literature:

\begin{enumerate}
    \item word2vec-like embedding differences \citep{arora2016latent,mimno2017strange,ethayarajh2019towards,reif2019visualizing,li2020sentence,hewitt2019structural,chen2021probing,chang2022geometry,jiang2023uncovering,mitchell-lapata-2008-vector,baroni-zamparelli-2010-nouns}
    \item logistic probing \citep{alain2017understanding,kim2018interpretability,belinkov2022probing,li2022emergent,geva2022transformer,nanda2023emergent}
    \item steering vectors \citep{wang2023concept,turner2023activation,merullo2024language,trager2023linear}
\end{enumerate}

\citeauthor{parklinear}'s theoretical framework addresses a critical gap by synthesizing the first formalization of what ``linear representation" means \citep{parklinear}. 
However, while the LRH has been validated primarily within individual languages (monolingually), there remains a significant gap in understanding how \rev{semantic-syntactic} transformations generalize across linguistic contexts (cross-lingually). 
Most existing work examines static concept encodings \citep{park2025,parklinear} rather than dynamic \rev{semantic-syntactic} transformations that reflect real-world language use. 
\rev{Our work is the first to extend the LRH to multilingual contexts and embedding models. Although, the linear representations we consider are geodesic arcs and not Euclidean lines.} 

\subsection{Linear \& Geometric Representation Techniques}
The geometric foundations established by \citet{parklinear} are crucial for understanding when and why linear algebraic operations succeed in capturing semantic relationships. 
With traditional Euclidean geometry, it is hard to accept that arbitrary dot products or cosine similarities have semantic meaning. 
Moreover, \citet{parklinear} demonstrated that the choice of inner product fundamentally determines the interpretability of geometric operations, providing principled foundations for representation analysis.
Our work builds directly on recent advances in understanding linear representations in language models \citep{parklinear, li2023inference}. 
RISE implements a technique that respects semantic structure, similar to the geometric framework developed by \citet{parklinear}. 
\rev{While previous work focused primarily on categorical concepts and word-level transformations, RISE extends our understanding to sentence-level, discourse-level transformations through cross-lingual and cross-model analysis using seven morphologically diverse languages.}

\subsubsection{Steering Vectors \& Embedding Models}
The practical applications of linear representation theory have been explored through steering vector techniques. 
\citet{turner2023activation}, \citet{liu2023context}, and \citet{zou2023representation} demonstrated that targeted modifications to internal, latent space representations can systematically alter model behavior without parameter updates. 
The majority of steering vector research \citep{im2025unified,rimsky2023steering,zou2023representation,li2023inference} is connected to activation steering, only investigating the impact of steering vectors in the activation, hidden, and/or latent layer of an LLM.
Recently, \citet{pham2024householder} introduced Householder Pseudo-Rotation (HPR), which addresses activation norm consistency issues in LLM behavioral modification through direction-magnitude decomposition and pseudo-rotational transformations. 
\rev{Building on the insight that geometric approaches outperform additive methods, our work extends geometric reasoning to semantic transformations in embedding space through Riemannian operations.
To our knowledge, there is no work investigating the application of steering vectors to embedding models -- only completion models.
This study extends steering principles to embedding models on manifolds, not activation-level steering.}

\subsection{Generalization and Reliability Challenges}
Current knowledge about the generalization properties of linear representations reveals significant limitations. 
The taxonomy of generalization research in natural language processing (NLP) \citep{hupkes2023taxonomy} provides a framework for evaluating robustness, but systematic applications to representation-based techniques (i.e., steering, probing, or embedding manipulation) have been limited.
Recent empirical studies have revealed that steering vector effectiveness varies substantially across different inputs and contexts \citep{tan2024analysing}. 
Secondly, the relationship between local and global linearity represents a particularly critical gap in current understanding. 
There have been numerous demonstrations of local linear behavior within specific domains or prompt formats, but achieving global linearity (generalizable to multiple model architectures with different pre-training as required by strong versions of the LRH) remains challenging. 
While many studies demonstrate impressive results in controlled settings, they often fail to address the robustness needed in practical applications.
This study contributes to the literature gap by presenting a robust framework for geometrically identifying discourse-level \rev{semantic-syntactic} changes across typologically diverse languages and model architectures.

\section{Theoretical Motivation}
The limitations identified \rev{in the related literature} point toward a fundamental, theoretical challenge: existing approaches operate in Euclidean\rev{/linear} space while modern embeddings live on curved manifolds \rev{(spherical space)}. 
This geometric mismatch may explain why steering vector \rev{research shows} inconsistent cross-context performance and why linear methods struggle with robust generalization.
We \textbf{hypothesize} that discourse-level \rev{semantic-syntactic} transformations correspond to intrinsic geometric operations on the embedding manifold, rather than fixed directions derived from Euclidean computations. 
If semantic transformations can be characterized as consistent rotational operations on the unit hypersphere where embeddings reside, \rev{\textbf{this would provide theoretical support for the extension of the Linear Representation Hypothesis in curved spaces (through geodesics) and cross-lingual interpretability.}}
Testing this hypothesis requires robust evaluation across diverse languages and embedding architectures to determine whether geometric consistency reflects universal semantic properties or model-specific artifacts.

\section{Rotor-Invariant Shift Estimation (RISE)}
Modern sentence embeddings from multilingual encoders reside approximately on a unit hypersphere in high-dimensional space when the training objective enforces or fixes the $\ell_2$-norm constraints \citep{hirota2020emu}, the embeddings are normalized to unit length \citep{reimers2019sentencebert}, or the model is designed to produce isotropic embeddings \citep{li2020sentence,ethayarajh2019contextual}. 
Local semantic transformations (e.g., negation, politeness, conditionality) can be understood as rotational displacements on this sphere. 
The key insight is that these displacements can be interpreted by aligning different contexts to a common geometric frame.

For any neutral sentence embedding $n \in \mathbb{S}^{d-1}$ and its semantically 
transformed variant $v \in \mathbb{S}^{d-1}$, we can compute an orthogonal 
transformation (Clifford-algebraic rotor) $R(n)$ that aligns $n$ to a canonical reference direction $e_{1}$. 
By applying this same transformation to $v$, we express the semantic change 
in a standardized coordinate system:

\begin{equation}
\xi = R(n) \, \log_{n}(v),
\end{equation}

where $\log_{n}(v)$ denotes the Riemannian logarithm that computes the tangent 
vector from $n$ to $v$ on the hypersphere, and $R(n)$ aligns the tangent vector to the canonical reference direction.
Normalized embeddings reside on a unit hypersphere, where geodesics define the shortest paths between points, preserving the manifold's intrinsic geometry rather than imposing Euclidean distance measures. These geodesic paths represent the natural notion of a ``line'' in the embedding space, as they define the shortest distance between two points on the surface. By working with geodesics, we ensure our semantic transformations are consistent with the manifold structure.
To ``flatten'' out the curved arc to a straight vector, the Riemannian logarithmic map $\log_{n}(v)$ produces the vector from $n$ to $v$ on a tangent plane at $n$.
By operating within the tangent space at $n$, geodesic differences can be treated as ordinary vectors.

\subsection{The Rotor Alignment Algorithm}
RISE proceeds in three steps \rev{illustrated in Figure~\ref{fig:RISE}:}

\begin{figure}[h]
\centering
\includegraphics[width=0.65\linewidth ]{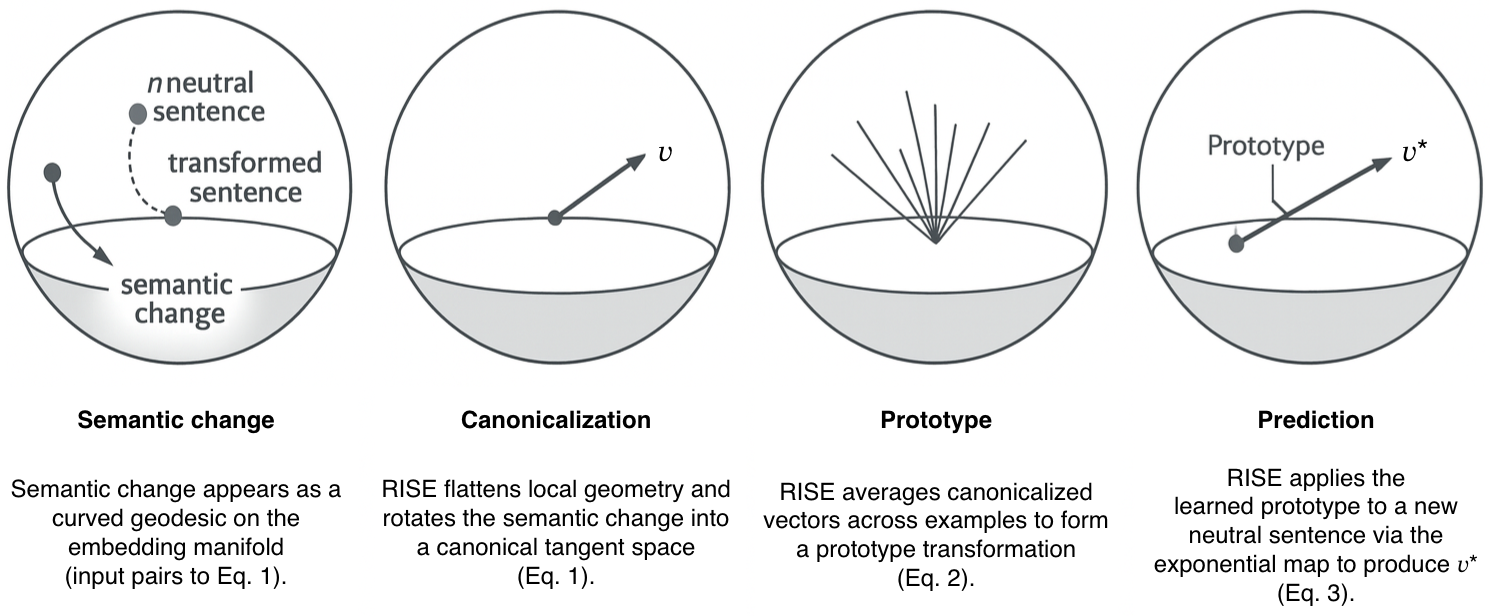}
\caption{\rev{RISE step-by-step illustration.}}
\label{fig:RISE}
\end{figure}

\textbf{Canonicalization.} For each neutral--transformed sentence pair $(n_i, v_i)$, 
compute a rotor $R(n_i)$ that maps $n_i$ to the reference direction $e_{1}$. 
We interpret canonicalization as controlling for the semantics present in the first elements of our pairs. 
By applying the canonical rotation to the second of the two the idea is that we have isolated the key differences between the elements in a fixed frame of reference.

\textbf{Prototype Learning.} Canonicalize all semantic changes into the reference frame 
and average all the tangent vectors to calculate one Prototype $\vec{p}$, where $M$ is the total amount of sentence pairs\footnote{For small angular differences, first-order equivalent to simply averaging the points and re-normalizing after the fact.}. This is a similar technique to mean-centering \citep{jorgensen2024improving}:
\begin{equation}
\vec{p} = \frac{1}{M} \sum_{i=1}^{M} R(n_i) \, \log_{n_i}(v_i).
\end{equation}

\textbf{Prediction.} To predict the semantic transformation for an unseen neutral embedding $n^{\ast}$, the prototype $\vec{p}$ can be used to predict the  
transformation embedding $v^{\ast}$ by converting the prototype $\vec{p}$ with the Riemannian exponential map and an application of the transpose of $n^{\ast}$'s canonicalizing rotor:

\begin{equation}
v^{\ast} = \exp_{n^{\ast}}\!\left(R(n^{\ast})\top \vec{p}\right).
\end{equation}

$R(n^{\ast})\top \vec{p}$ rotates $\vec{p}$ into the tangent space at $n^{\ast}$ . Then the Riemannian exponential $\exp_{n^{\ast}}(\vec{p})$ takes the tangent vector $\vec{p}$ and moves along the geodesic starting at $n^{\ast}$. The vector direction is which geodesic to follow and the length is how far along that arc to go (in radians).

\subsection{Differentiation from Related Work}
Our approach is related to recent advances in understanding linear representations in language models. 
As discussed in Section 2.2,
\citet{park2025} use a ``causal inner product'' that respects semantic structure  in a function space using the Riesz isomorphism. 
However, RISE uses Riemannian geometry to operate consistently on the curved manifolds.
Both methods take advantage of geometric properties, but the methods are distinctly different. 

Crucially, RISE transformations exhibit commutativity: applying multiple semantic transformations yields consistent results regardless of order (see Appendix \ref{app:rise_theory}).
This commutativity property provides strong evidence for the LRH, as it demonstrates that semantic transformations behave like vector additions in the tangent space—geodesics serve as the curved-space generalization of straight lines. 
The preservation of additive structure across semantic operations suggests that the geometric framework captures fundamental algebraic properties of meaning composition. 
We discuss more about the commutativity properties in Appendix \ref{app:rise_theory}.

Furthermore, the analysis in \citet{park2025} focused on categorical relationships in the unembedding space of language models; our work examines discourse-level transformations in sentence embeddings across multiple languages.
RISE effectively implements a non-Euclidean transformation that aligns with the natural curved manifold structure of the embedding space. 
This connection to high-dimensional  geometry provides theoretical grounding for why rotational operations can capture semantic transformations more effectively than simple vector additions, and extends the linear subspace hypothesis to curved/geodesic subspaces.

\section{Experimental Design}
\subsection{Discourse-level Semantic-Syntactic Changes \& Language Selection}
We focus on three discourse-level \rev{semantic-syntactic} transformations that vary in their context-dependence:

\textbf{Negation:} The logical reversal of the propositional content of a statement; where the proposition is ''P'' we take the negation to be ''not-P.'' Moreso, we are negating the predicate. This transformation is semantically precise and should exhibit high geometric consistency across contexts and languages.

\textbf{Conditionality:} Converting declarative statements into conditional constructions (``P'' → ``If P''). This introduces modal semantics that may interact with contextual factors.

\textbf{Politeness:} Increasing the social formality or deference level of utterances. This is highly context-dependent and culturally variable, making it a challenging test case for geometric consistency.

We selected seven morphologically diverse languages to ensure broad coverage of morphological, syntactic phenomena, and resource levels: English, Spanish, Japanese, Tamil, Thai, Arabic, and Zulu. 
This selection spans multiple language families (Indo-European, Sino-Tibetan, Dravidian, Afroasiatic, Niger-Congo) and different morphological types (analytic, agglutinative, fusional).
The languages also represent different levels of language model availability and resources. 
The diversity is crucial because different languages realize semantic transformations through distinct linguistic mechanisms. For instance, negation might be expressed through:
(1) Particles (i.e. English ``not''); (2)  Affixes (i.e. Tamil verb-internal negation, Japanese ``nai''); and (3) Auxiliary constructions (i.e. English ``does/has not'').
By testing across this range, we can determine whether geometric consistency reflects universal semantic properties or is merely an artifact of particular linguistic structures. 

\subsection{Datasets, Embedding Models, \& Linear Baselines}
We use three datasets and three models for evaluation. We  \rev{used two open-source, external datasets: \textbf{The Benchmark of Linguistic Minimal Pairs (BLiMP)} \citep{warstadt2020BLiMP} and \textbf{Sentences Involving Compositional Knowledge (SICK)} \citep{marelli2014sick}, and synthetically generated one dataset, referred to as the \textbf{Synthetic Multilingual} dataset.} For each language-transformation combination in the Synthetic Multilingual dataset, we generated 1,000 neutral-transformed sentence pairs using GPT-4.5 with carefully controlled prompts (see Appendix \ref{app:prompts}). 
To ensure robust analysis, we implemented several diversity controls (see Appendix \ref{app:data-generation}).

We compare three multilingual embedding models: \textbf{text-embedding-3-large} \citep{openai-text-embedding-3-large-2024}, \textbf{bge-m3 \footnote{Bge-m3  should be m3 as titled in the final version of \citep{chen2024bge}, but we referenced the model as bge-m3 in this paper and figures.}} \citep{chen2024bge}, and \textbf{mBERT} \citep{devlin2019bert}.
The text-embedding-3-large model produces 3072-dimensional vectors, bge-m3 produces 1024-dimensional vectors, and mBERT produces 768-dimensional vectors. All selected models produce constant-length embeddings that reside on a hypersphere making them suitable for our geometric analysis. This dimensional diversity allows us to test whether RISE effectiveness depends on embedding dimensionality. 
We calculate a \textit{rotor alignment score} where the scores represent mean cosine similarity between predicted embedding vectors and the semantically transformed pair on held-out test sets, with higher values indicating more consistent geometric structure.
\rev{Table~\ref{tab:cosine_interpretation} describes how the cosine similarity scores are interpreted.}

\rev{We include Mean Difference Vectors (MDV), and Procrustes alignment as baseline comparisons because they represent  standard linear approaches used to model transformations in embedding spaces. 
MDV test whether simple difference vectors can capture semantic or cross-lingual structure, while Procrustes evaluates whether a single global rotation can align transformed embeddings. 
MDV is the geometrically correct analogue of the Euclidean additive method for modern spherical embeddings, providing a stronger and fairer baseline for RISE.}

\begin{table}[h]
\centering
\small
\begin{tabularx}{\linewidth}{p{0.22\linewidth} X p{0.28\linewidth}}
\toprule
\textbf{Cosine Similarity Range} & 
\textbf{Interpretation} & 
\textbf{Supporting Literature} \\
\midrule
$\geq 0.80$ 
& Strong, consistent geometric structure 
& \cite{reimers2019sentencebert} \\

$0.65$--$0.80$ 
& Moderate, reliable structure 
& \cite{mikolov2013linguistic,ethayarajh2019contextual} \\

$0.50$--$0.65$ 
& Weak or variable structure 
& \cite{ethayarajh2019contextual,conneau2018you} \\

$< 0.30$ 
& Inconsistent or failing transformation 
& \cite{artetxe2018robust,conneau2018you} \\
\bottomrule
\end{tabularx}
\caption{Interpretation of cosine similarity magnitudes used throughout this work. Higher values indicate stronger geometric consistency between predicted and target embeddings. These thresholds are stricter than prior work but remain consistent with the established interpretations in the literature.}
\label{tab:cosine_interpretation}
\end{table}






\section{Results}

\subsection{Cross-Language Transfer Comparison}
This section discusses the comparison of \rev{embedding} models trained in one of the seven languages and tested on all seven. The results of this section demonstrate RISE multilingual performance computed by three embedding models. See Appendix \ref{app:heatmaps} for comprehensive results across all phenomena for each model. 

\textbf{Negation} emerges as the most robust discourse-level, \rev{semantic-syntactic} transformation, achieving the highest mean rotor alignment score \textbf{(0.788)} across all model-language combinations with performance ranging from 0.686 to 0.918. Figure \ref{fig:negation} demonstrates RISE performance on negation for each model.
RISE transformations for negation are most geometrically consistent in text-embedding-3-large. 
Negation's strong performance indicates that generalizable discourse-level, \rev{semantic-syntactic} changes are captured by RISE and best applied cross-lingually in text-embedding-3-large.

{\setlength{\intextsep}{0pt}
 \setlength{\abovecaptionskip}{2pt}
 \setlength{\belowcaptionskip}{0pt}
 \begin{figure}[H]
   \centering
   \includegraphics[width=\textwidth]{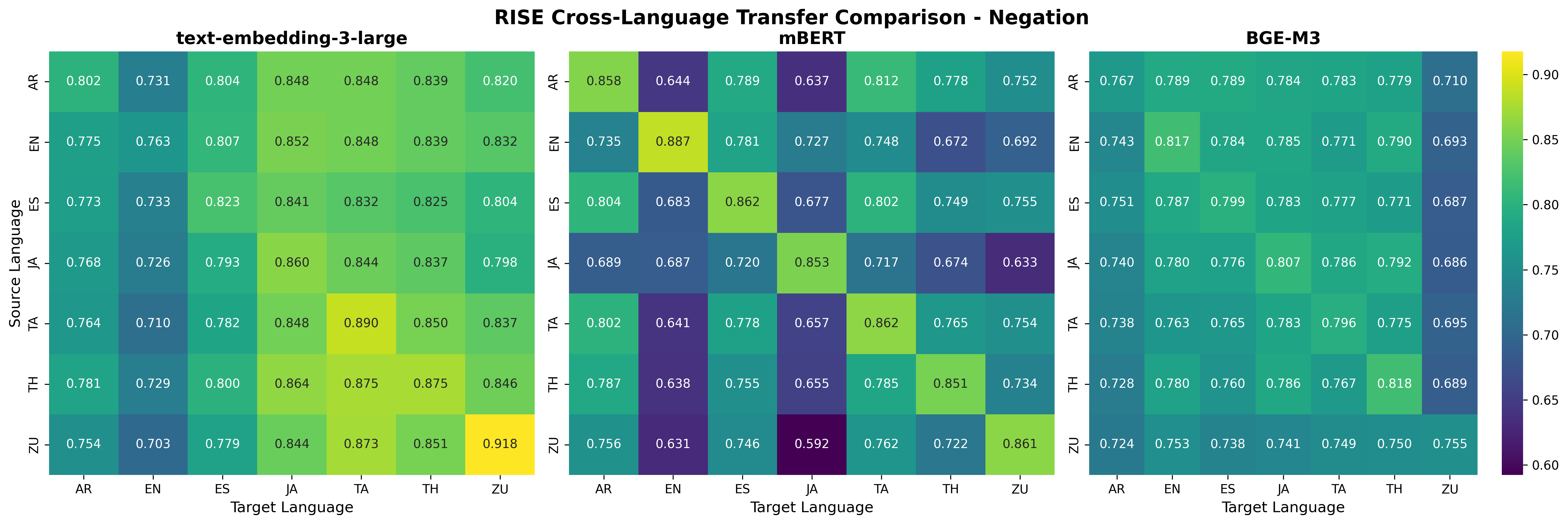}
   \caption{Embedding model heatmap cross-lingual transfer comparison on negation.}
   \label{fig:negation}
 \end{figure}
 \vspace{-\parskip} 
}

\textbf{Conditionality} demonstrates the highest stability and consistency across cross-language transfers, with the lowest performance variability (0.038) and most stable individual measurements (see Appendix \ref{app:heatmaps}). 
With the second highest, mean performance \textbf{(0.780)}, conditionality is particularly consistent results across all combinations. The strong transfer seen in bge-m3 and text-embedding-3-large suggests that conditional semantics are captured by stable geometric structure despite their modal complexity.

\textbf{Politeness} exhibits the most variable geometric structure, ranking third in performance \textbf{(0.762 mean)} with the highest performance variability (0.060) across combinations.  
This variability aligns with expectations, as politeness realizations depend heavily on cultural context and linguistic conventions, making cross-language transfer inherently more challenging.

The contrast across phenomena performance reflects an interesting insight. In the results, negation appears more robust, politeness is most variable, and conditionality sits between. 
This suggests embeddings encode logical semantic operators (negation and conditionality) with strong cross-lingual consistency. 
However, pragmatic operators (politeness) are less reliable due to inherent language-specific indicators and cultural conventions. 
Additionally, cross-language analysis revealed dimensionality does not directly predict cross-lingual performance. 
Despite having lower dimensionality, bge-m3 (1024-dim) demonstrated  the least variance in cross-language performance for all phenomena and languages. 
While text-embedding-3-large (3072-dim) showed highest cross-language performance (Figure~\ref{fig:text_embedding_3_large_heatmaps}),
mBERT (768-dim) showed strong monolingual performance, but exhibited high variability, particularly for politeness in cross-language settings. 
These results highlight that training methodology and architectural choices matter more than raw embedding dimensionality for cross-language semantic transfer.

\rev{The cross-language analysis fully presented in Appendix \ref{app:heatmaps} supports our hypothesis that discourse-level \rev{semantic-syntactic} transformations correspond to geometric operations on the embedding manifold.}
The variation across models, preservation of linguistic relationships across languages, and transformation patterns indicate that RISE successfully identifies \rev{semantic-syntactic} transformation on the embedding manifold.  
The limitations and future work are discussed further on.

\begin{figure}[htbp]
  \centering
  \includegraphics[width=0.30\linewidth]{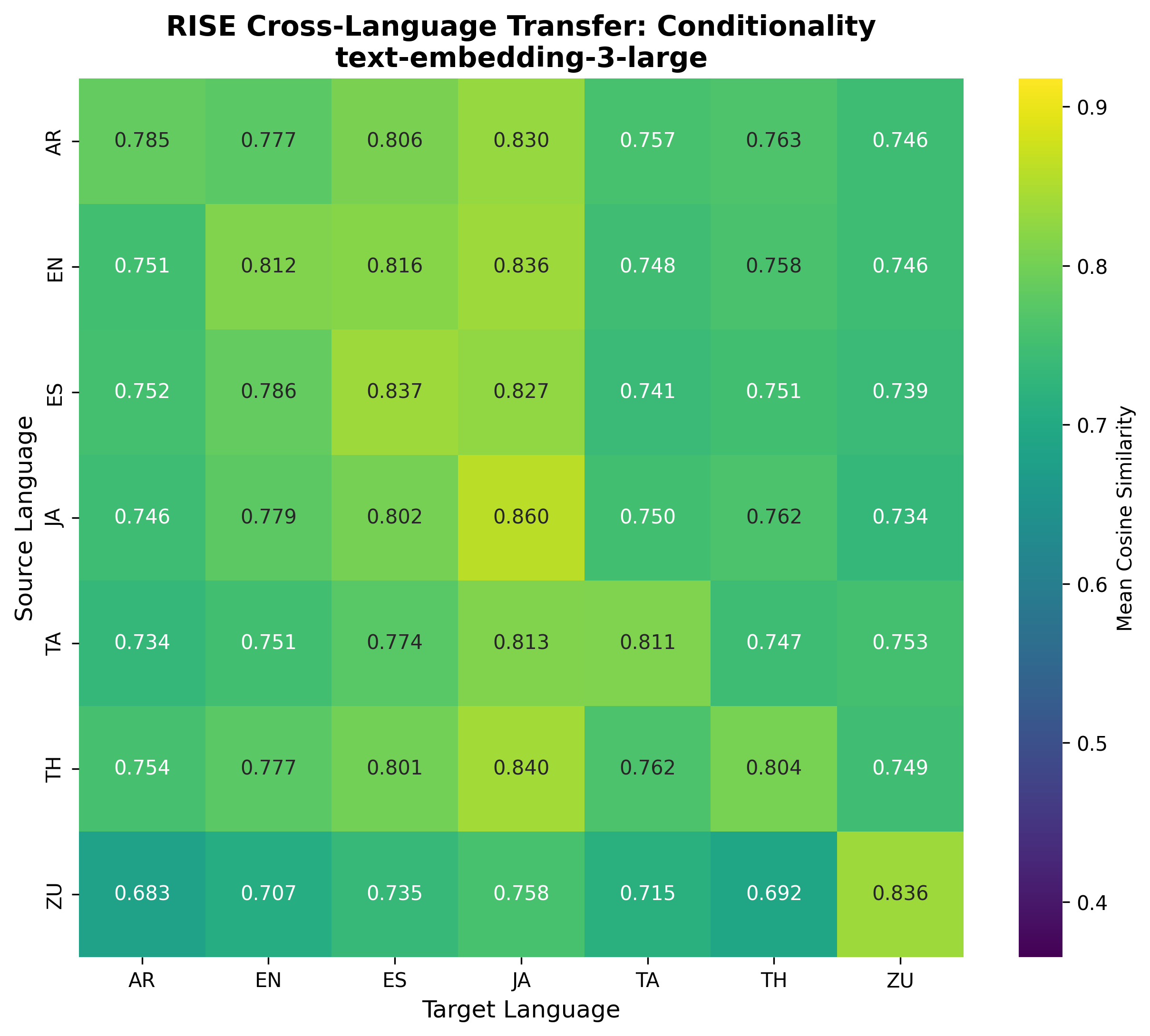}
  \includegraphics[width=0.30\linewidth]{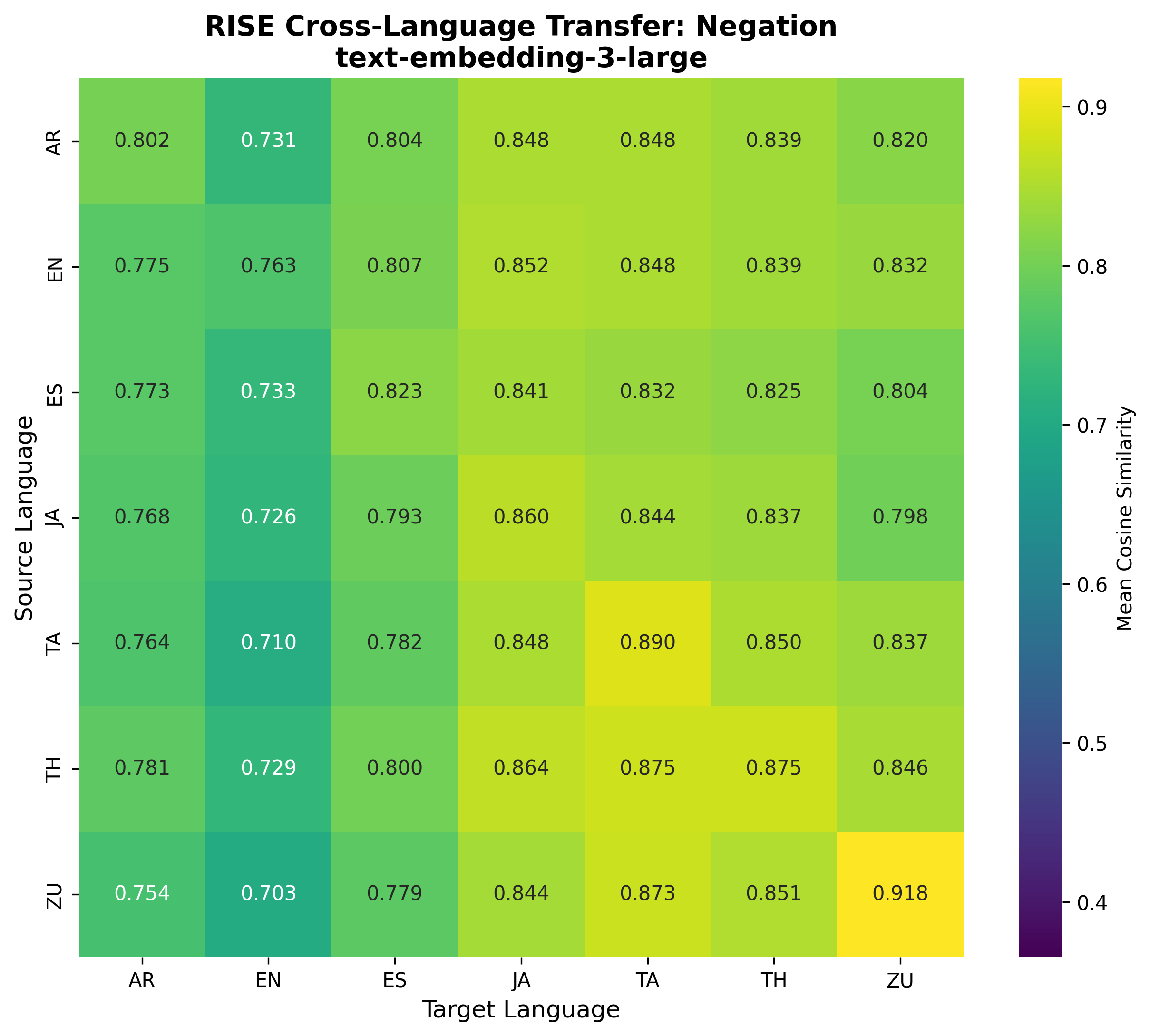}
  \includegraphics[width=0.30\linewidth]{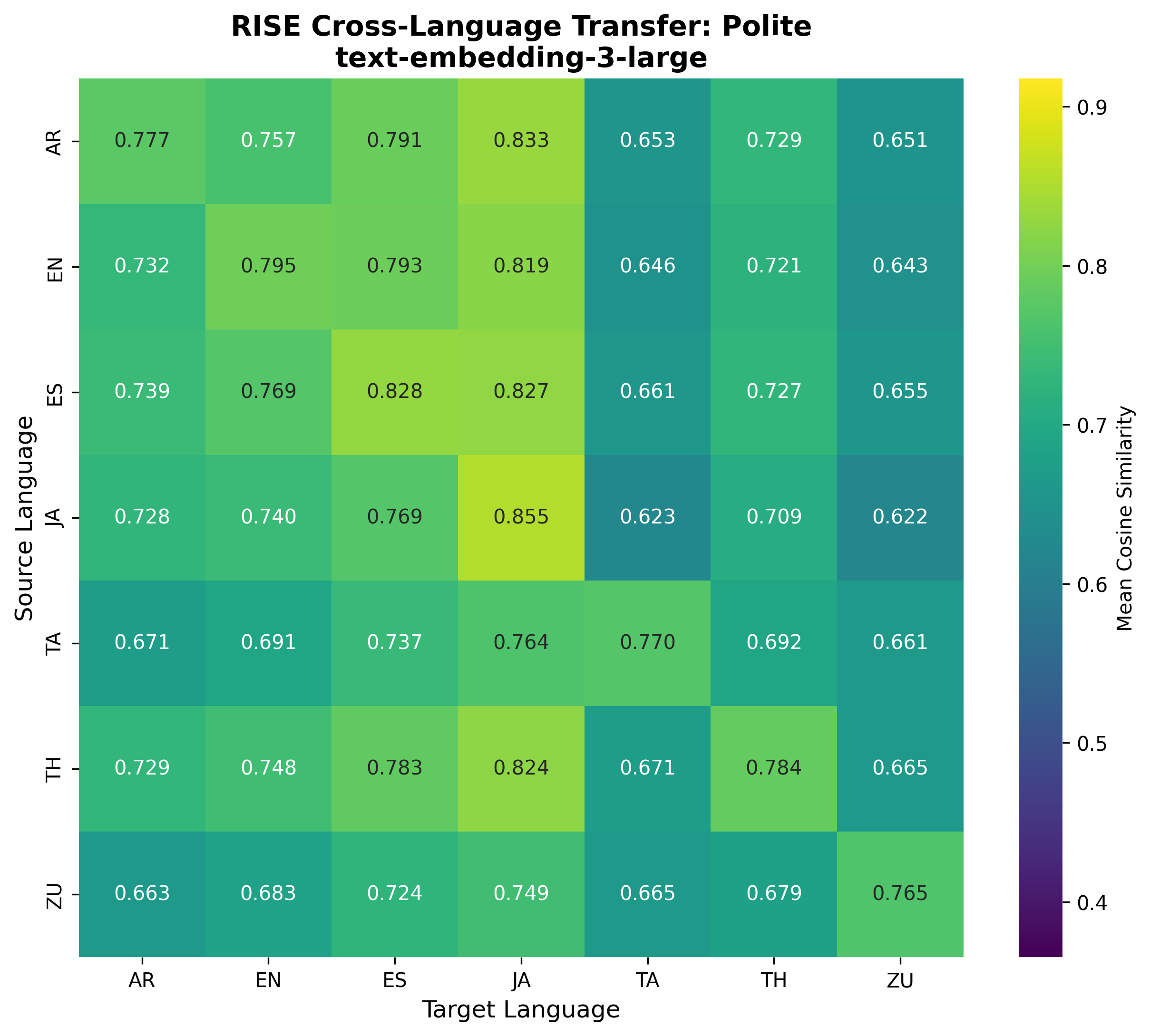}
  \caption{Cross-language transfer heatmaps for text-embedding-3-large showing RISE performance across all language pairs for conditionality, negation, and politeness transformations. Darker colors indicate higher cosine similarity between predicted and target embeddings.}
  \label{fig:text_embedding_3_large_heatmaps}
\end{figure}

\subsection{Cross-Model Transfer Comparison}
To evaluate RISE prototypes' robustness to transfer across different embedding architectures, we conducted cross-model mapping experiments using the method developed by \citet{morris2020linearity}. This approach learns statistical mappings between embedding spaces through principal component analysis (PCA) and distributional alignment, enabling transfer of learned RISE prototypes from one model to another.
We specifically examined transfer from text-embedding-3-large (3072-dimensional) to bge-m3 (1024-dimensional), demonstrating cross-model semantic transfer across different dimensionalities and training objectives. For each language pair and phenomenon, we learn RISE prototypes in text-embedding-3-large using 80\% of the data, map these prototypes and $e_1$ to bge-m3 space, and evaluate performance on native bge-m3 embeddings using the remaining 20\%.
Figure \ref{fig:cross_model_heatmaps} demonstrates comprehensive cross-model and cross-language transfer results. 

Cross-model transfer from text-embedding-3-large to bge-m3 reveals strong language-dependent performance. 
English achieves 0.80-0.82 similarity across all transformations, while other languages cluster around 0.70-0.75, and Zulu consistently scores 0.63-0.66. 
This 20\% performance gap persists across conditionality, negation, and politeness transformations.
These results suggest rotations can transfer between architecturally different models, but their effectiveness depends critically on source language, indicating that learned transformations are not architecture-independent.
The consistent English advantage across models suggests these embedding spaces share more robust geometric structures for English, likely reflecting training data imbalances (Anglo-centric bias in the composition of the model's training data).
The consistent language ranking across different semantic transformations (conditionality, negation, politeness) suggests the bias is structural rather than semantic.
In conclusion, RISE successfully captures semantic patterns that perform consistently in a cross-model comparison. 

\begin{figure}[htbp]

\centering
\includegraphics[width=0.30\textwidth]{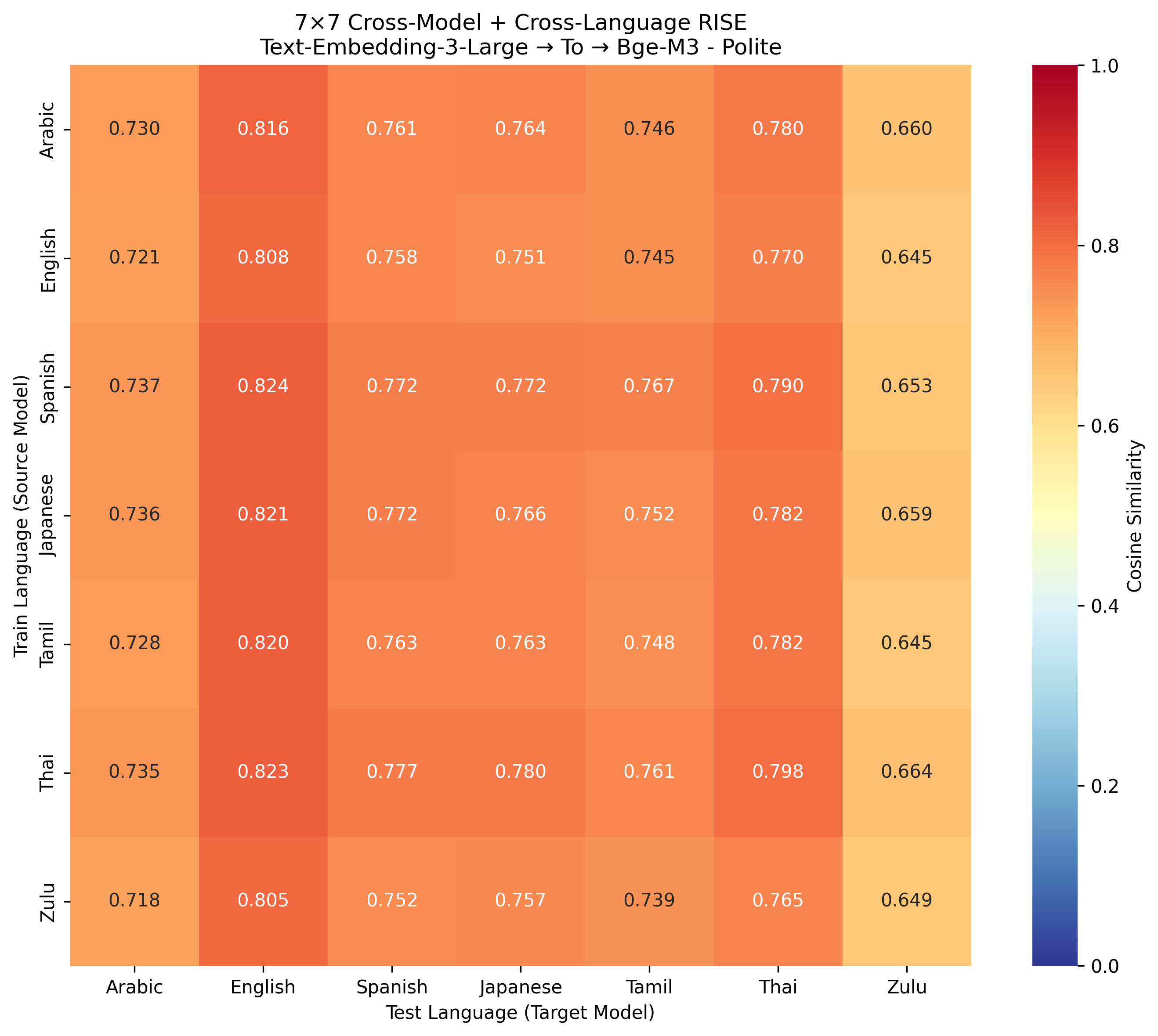}
\includegraphics[width=0.30\textwidth]{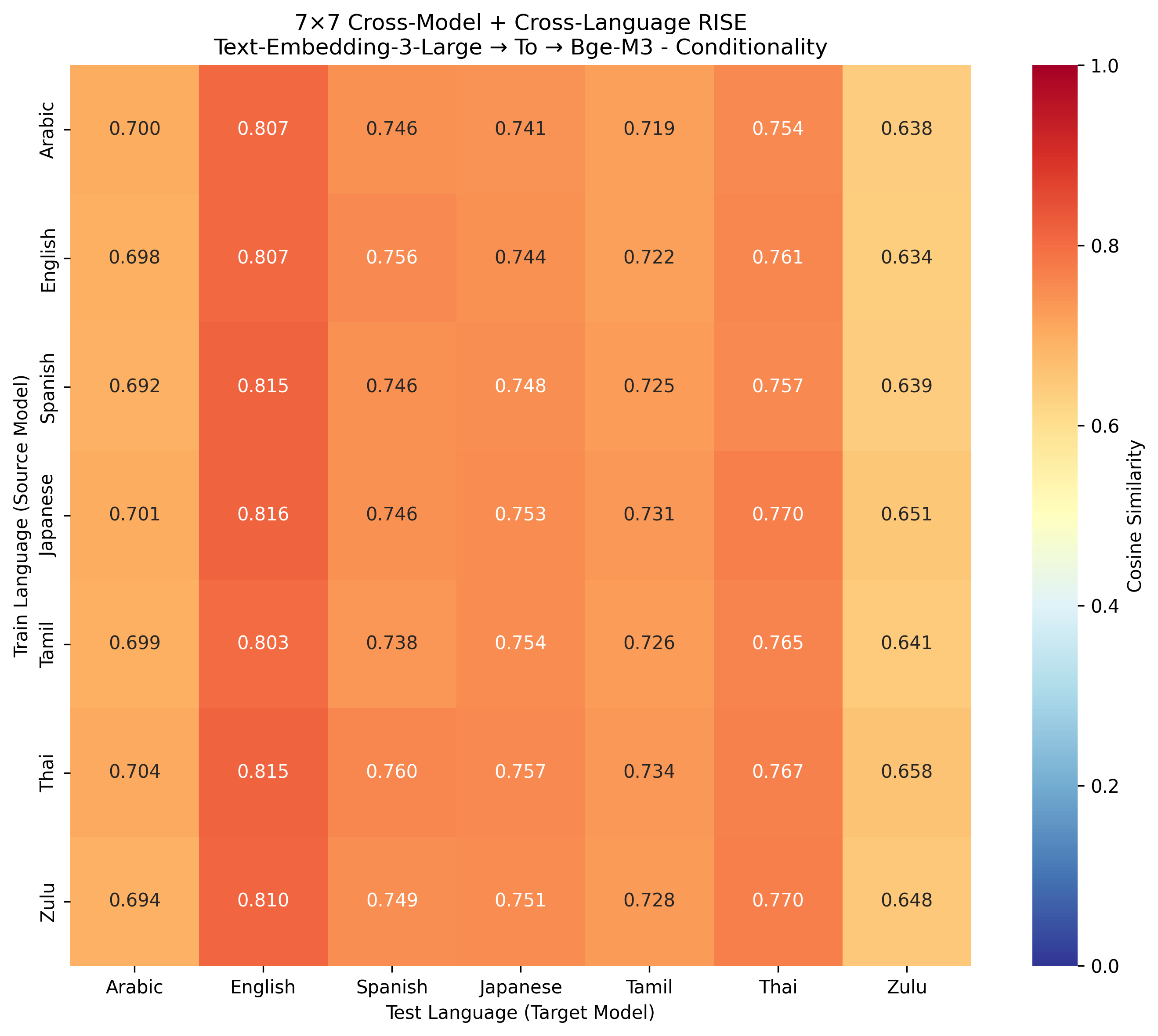}
\includegraphics[width=0.30\textwidth]{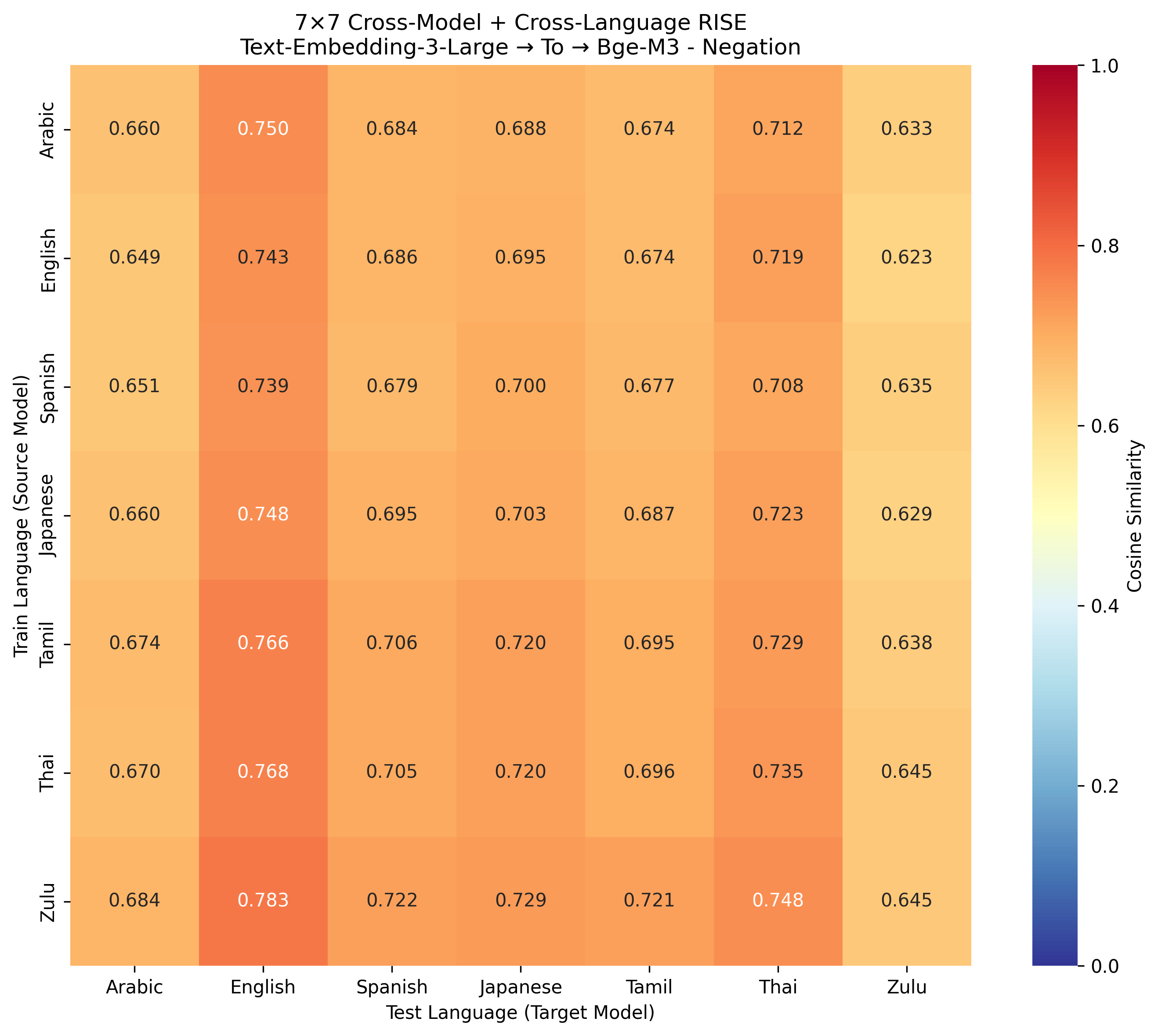}
\caption{Cross-Model Semantic Transfer: text-embedding-3-large → bge-m3. Each cell shows transfer performance from source language prototype (text-embedding-3-large) to target language test set (bge-m3). Diagonal elements represent pure cross-model transfer, while off-diagonal elements show combined cross-model and cross-language transfer using Morris statistical mapping \citep{morris2020linearity}.}
\label{fig:cross_model_heatmaps}
\end{figure}

\subsection{English Task-Based Comparison}
Our main investigation is how well RISE peforms in in multi-lingual settings. However there are limited external datasets for evaluating the performance discourse-level, \rev{semantic-syntactic} transformation tasks. Due to the limited resources, we had to select the most related datasets, BLiMP and SICK. BLiMP is a paired sentence dataset for major grammatical phenomena in English, and SICK is a dataset with paired sentences with entailment, contradiction, and neutral labels. 

Table \ref{tab:master_validation} summarizes RISE performance across the three selected datasets. The results confirm that all models achieve strong performance, with particular strengths varying by dataset: mBERT excels on grammatical tasks (BLiMP) and contradiction detection (SICK), while bge-m3 shows the most consistent performance across synthetic multilingual data.
The dramatic performance gap between BLiMP ($>$0.92) and SICK (0.62-0.74) suggests that RISE rotations might be capturing something more specific than general semantic transformations.

The high BLiMP performance indicates RISE excels at preserving grammatical/syntactic structure, while the moderate SICK performance suggests these same rotations don't preserve semantic relatedness as well. 
These results show that benchmark choice dramatically affects relative model ranking. 
Robustness depends on whether the task prioritizes cross-lingual consistency (favoring bge-m3) or raw performance on specific phenomena (favoring text-embedding-3-large for negation, mBERT for grammatical tasks).

\begin{table}[H]
\centering
\caption{RISE Performance Across Three Datasets: The performance is measured with the rotor alignment score between RISE-steered embeddings and target embeddings where bold values indicate best performance per dataset. text-embedding-3-large is abbreviated as TE3L.}
\label{tab:master_validation}
\begin{tabular}{lccc}
\toprule
\textbf{Model} & \textbf{Synthetic Multilingual} & \textbf{BLiMP Benchmark} & \textbf{SICK Dataset} \\
\midrule
TE3L (3072d) & 0.771 & 0.929 & 0.623 \\
bge-m3 (1024d) & \textbf{0.782} & 0.956 & 0.631 \\
mBERT (768d) & 0.709 & \textbf{0.961} & \textbf{0.736} \\
\midrule
\textbf{Average} & 0.754 & 0.949 & 0.663 \\
\bottomrule
\end{tabular}

\end{table}

\subsection{Linear Baseline Comparisons}
The full results presented in Appendix~\ref{app:lin_baselines} compare RISE against two standard baselines, Mean Difference Vectors (MDV) and Procrustes alignment, across the same three datasets. MDV is not Euclidean. 
MDV preserves spherical structure and naturally resembles RISE more closely than Procrustes. This distinction is directly reflected in the results: MDV and RISE transfers best across languages where Procrustes fails.

The strongest performance appears in monolingual English evaluation (BLiMP), while performance drops substantially for Procrustes on semantic relatedness (SICK) shown in Table~\ref{tab:appendixC_summary}.
This shift in performance reflects Procrustes' inability to identify a generalizable semantic–syntactic relationship as expected by method. Procrustes fits a single global rotation which is too rigid for the cross-lingual and cross model analysis
In contrast, RISE maintains stable cross-lingual and cross-model performance (e.g., App.~\ref{app:heatmaps}. Figures 5–7), indicating that geometric operations on the manifold better capture discourse-level semantic structure than Euclidean differences.

The MDV vs.~RISE vs.~Procrustes results reinforce our earlier claim that methods operating on the curved manifold (where sentence embeddings inherently reside) perform better than Euclidean/linear methods. 
Most steering and probing techniques operate in linear space, and we conjecture that this geometric mismatch helps explain why linear methods struggle to generalize. 
In short, Procrustes fits a single global rotation which is too rigid for the cross-lingual and cross model analysis. Geometric transformations, like RISE and MDV, are better suited for semantic-syntactic analysis and cross-lingual stability.

\begin{table}[htbp]
\centering
\small
\begin{tabular}{lccc}
\toprule
\textbf{Method} & 
\textbf{Monolingual Syntactic} & 
\textbf{Monolingual Semantic} &
\textbf{Cross-Language Transfer} \\
& (BLiMP) & (SICK) & (All Phenomena) \\
\midrule

\textbf{RISE} 
& \textbf{Strong} (0.97) 
& \textbf{Strong} (0.84) 
& \textbf{Moderate--Strong} (0.74--0.89) \\

\textbf{MDV} 
& \textbf{Strong} (0.97) 
& \textbf{Strong} (0.83) 
& \textbf{Moderate--Strong} (0.72--0.91) \\

\textbf{Procrustes} 
& \textbf{Strong} (0.99) 
& \textbf{Moderate} (0.67) 
& \textbf{Failing--Weak} (0.25--0.62) \\
\bottomrule
\end{tabular}
\caption{\rev{Condensed summary of baseline comparisons from Appendix~C using the cosine-similarity interpretation scale from Table~\ref{tab:cosine_interpretation}. RISE and MDV show Strong monolingual and Moderate--Strong cross-language structure, whereas Procrustes drops to Weak or Failing consistency outside syntactic, same-language settings.}}
\label{tab:appendixC_summary}
\end{table}

\section{Discussion \& Future Work}
Our findings demonstrate that meaningful \rev{semantic-syntactic} operations can be recovered as geometric transformations in modern language model representations. 
RISE successfully identifies consistent geometric structure for discourse-level \rev{semantic-syntactic} changes, primarily for text-embedding-3-large and negation in multilingual settings. 
\rev{The results demonstrating spherical methods, RISE and MDV, out perform linear methods, Procrustes alignment, provide positive results for extending the LRH to spherical spaces.}

Evaluation benchmarks (Table \ref{tab:master_validation}) reveal task-dependent effectiveness. RISE achieves near-perfect performance on syntactic acceptability (BLiMP: 0.93-0.96) but only moderate performance on semantic similarity (SICK: 0.62-0.74), suggesting better alignment with grammatical rather than semantic transformations.
Section 6.1 shows that negation and conditionality are the most generalizable discourse-level, \rev{semantic-syntactic} changes captured by RISE and best applied cross-lingually in text-embedding-3-large.
\rev{Our cross-model transfer experiments expose an English-centric bias, with English achieving 20\% higher transfer scores than languages like Zulu. 
This English-centric bias persists across all semantic transformations, indicating that current multilingual models encode geometric structures that prioritize English.
Future work should focus on developing more equitable multilingual representations and investigating which language-specific geometric structures are an inherent feature of the models.}

Together these results support that RISE is most successful at identifying semantic transformation with distinct grammatical factors, but more work is needed to justify semantic transformations in multilingual models are universal geometric operations.
First, our analysis focuses on three specific linguistic transformation types.
Future work should expand to additional semantic and pragmatic phenomena to test the generality of geometric consistency principles. 
Second, while our experiments used three diverse embedding models (text-embedding-3-large, bge-m3, and mBERT), validation across additional architectures would strengthen claims about the universality of geometric semantic structure. 
\rev{Third, the reliance on GPT-4.5 for data generation may introduce subtle biases toward English-centric conceptualizations of semantic phenomena. Future work should incorporate more diverse data sources and validation by native speakers.}


\section{Conclusion}
The ability to learn geometric transformations for discourse changes relates to work on text generation and steering vectors \citep{turner2023activation, li2023inference}. 
Our rotor-based approach, RISE, provides a geometric framework for understanding and improving interpretability in language models.
This work investigated whether discourse-level \rev{semantic-syntactic} transformations in multilingual embedding spaces correspond to intrinsic geometric operations, specifically rotations identified through the RISE method. 
\rev{Our comprehensive evaluation across multiple baselines, models, languages, and datasets reveals a more complex reality than initially hypothesized.}
This work demonstrates that modern language model representations maintain \rev{interpretable geometric structure for some semantic-syntactic} transformations, extending the promise of geometric semantics from early word embeddings to contemporary transformer models. 
We show that:

\begin{enumerate}
    \item \rev{Semantic transformations with clear syntactic mapping demonstrate the most consistent geometric structure.}
    \item \rev{RISE successfully identifies semantically meaningful geometric structure in high-dimensional embedding spaces that generalizes cross-lingually and across model architecture.}
\end{enumerate}

As language models continue to evolve, understanding these geometric foundations will be crucial for developing more \rev{interpretable} AI systems.
By revealing transferable geometric structure in semantic transformations (e.g. negation and conditionality), this work opens new possibilities for \rev{understanding} language model behavior through geometric interventions. 
\rev{Our work promotes geometric methods as more appropriate  approaches to cross-lingual semantic interpretation, achieving 77\%-95\% cross-language transfer effectiveness across typologically diverse languages.}
By developing RISE, we demonstrate that \rev{interpretable} structure exists for some grammatically distinct semantic transformations, providing a tools for understanding how these systems encode semantic knowledge.
\rev{While RISE remains valuable for analyzing model-specific semantic structures, claims about universal geometric operations require substantial qualification.}


\bibliography{main}

@article{achiam2023gpt,
  title={Gpt-4 technical report},
  author={Achiam, Josh and Adler, Steven and Agarwal, Sandhini and Ahmad, Lama and Akkaya, Ilge and Aleman, Florencia Leoni and Almeida, Diogo and Altenschmidt, Janko and Altman, Sam and Anadkat, Shyamal and others},
  journal={arXiv preprint arXiv:2303.08774},
  year={2023}
}

@inproceedings{alain2017understanding,
  title={Understanding intermediate layers using linear classifier probes},
  author={Alain, Guillaume and Bengio, Yoshua},
  booktitle={International Conference on Learning Representations},
  year={2017},
  url={https://openreview.net/forum?id=ryF7rTqgl}
}

@article{arora2016latent,
  title={A latent variable model approach to pmi-based word embeddings},
  author={Arora, Sanjeev and Li, Yuanzhi and Liang, Yingyu and Ma, Tengyu and Risteski, Andrej},
  journal={Transactions of the Association for Computational Linguistics},
  volume={4},
  pages={385--399},
  year={2016}
}

@article{belinkov2022probing,
  title={Probing classifiers: Promises, shortcomings, and advances},
  author={Belinkov, Yonatan},
  journal={Computational Linguistics},
  volume={48},
  number={1},
  pages={207--219},
  year={2022}
}

@inproceedings{bolukbasi2016man,
  title={Man is to Computer Programmer as Woman is to Homemaker? Debiasing Word Embeddings},
  author={Bolukbasi, Tolga and Chang, Kai-Wei and Zou, James and Saligrama, Venkatesh and Kalai, Adam},
  booktitle={Advances in Neural Information Processing Systems},
  year={2016}
}

@inproceedings{chang2022geometry,
  title={The Geometry of Multilingual Language Model Representations},
  author={Chang, Tyler and Tu, Zhuowen and Bergen, Benjamin},
  booktitle={Proceedings of the 2022 Conference on Empirical Methods in Natural Language Processing},
  pages={119--136},
  year={2022}
}

@article{chen2021probing,
  title={Probing BERT in hyperbolic spaces},
  author={Chen, Boli and Fu, Yao and Xu, Guangwei and Xie, Pengjun and Tan, Chuanqi and Chen, Mosha and Jing, Liping},
  journal={International Conference on Learning Representations},
  year={2021}
}

@inproceedings{chen2024bge,
  title={M3-embedding: Multi-linguality, multi-functionality, multi-granularity text embeddings through self-knowledge distillation},
  author={Chen, Jianlyu and Xiao, Shitao and Zhang, Peitian and Luo, Kun and Lian, Defu and Liu, Zheng},
  booktitle={Findings of the association for computational linguistics: ACL 2024},
  pages={2318--2335},
  year={2024}
}

@inproceedings{devlin2019bert,
  title={BERT: Pre-training of Deep Bidirectional Transformers for Language Understanding},
  author={Devlin, Jacob and Chang, Ming-Wei and Lee, Kenton and Toutanova, Kristina},
  booktitle={Proceedings of the 2019 Conference of the North American Chapter of the Association for Computational Linguistics: Human Language Technologies, Volume 1 (Long and Short Papers)},
  pages={4171--4186},
  year={2019},
  publisher={Association for Computational Linguistics}
}

@article{elhage2022toy,
  title={Toy models of superposition},
  author={Elhage, Nelson and Hume, Tristan and Olsson, Catherine and Schiefer, Nicholas and Henighan, Tom and Kravec, Shauna and Hatfield-Dodds, Zac and Lasenby, Robert and Drain, Dawn and Chen, Carol and others},
  journal={arXiv preprint arXiv:2209.10652},
  year={2022}
}

@inproceedings{ethayarajh2019towards,
  title={Towards understanding linear word analogies},
  author={Ethayarajh, Kawin and Duvenaud, David and Hirst, Graeme},
  booktitle={Proceedings of the 57th annual meeting of the association for computational linguistics},
  pages={3253--3262},
  year={2019}
}

@inproceedings{ethayarajh2019contextual,
  title={How contextual are contextualized word representations? Comparing the geometry of BERT, ELMo, and GPT-2 embeddings},
  author={Ethayarajh, Kawin},
  booktitle={Proceedings of EMNLP-IJCNLP},
  pages={55--65},
  year={2019}
}

@inproceedings{geva2022transformer,
  title={Transformer feed-forward layers build predictions by promoting concepts in the vocabulary space},
  author={Geva, Mor and Caciularu, Avi and Wang, Kevin and Goldberg, Yoav},
  booktitle={Proceedings of the Conference on Empirical Methods in Natural Language Processing},
  pages={30--45},
  year={2022}
}

@inproceedings{hewitt2019structural,
  title={A structural probe for finding syntax in word representations},
  author={Hewitt, John and Manning, Christopher D},
  booktitle={Proceedings of the 2019 Conference of the North American Chapter of the Association for Computational Linguistics: Human Language Technologies, Volume 1 (Long and Short Papers)},
  pages={4129--4138},
  year={2019}
}

@inproceedings{hirota2020emu,
  title     = {Emu: Enhancing Multilingual Sentence Embeddings with L2 Constrained Softmax Loss},
  author    = {Hirota, Wataru and Tanaka, Masahiro and Takase, Sho and Okazaki, Naoaki and Inui, Kentaro},
  booktitle = {Proceedings of the AAAI Conference on Artificial Intelligence},
  volume    = {34},
  pages     = {7904--7911},
  year      = {2020},
  doi       = {10.1609/aaai.v34i05.6301}
}

@article{hupkes2023taxonomy,
  title={A taxonomy and review of generalization research in NLP},
  author={Hupkes, Dieuwke and Giulianelli, Mario and Dankers, Verna and Artetxe, Mikel and Elazar, Yanai and Pimentel, Tiago and Christodoulopoulos, Christos and Lasri, Karim and Saphra, Naomi and Sinclair, Arabella and others},
  journal={Nature Machine Intelligence},
  volume={5},
  number={10},
  pages={1161--1174},
  year={2023},
  publisher={Nature Publishing Group}
}

@article{im2025unified,
  title={A unified understanding and evaluation of steering methods},
  author={Im, Shawn and Li, Yixuan},
  journal={arXiv preprint arXiv:2502.02716},
  year={2025}
}

@inproceedings{jha2025,
  title={Harnessing the Universal Geometry of Embeddings},
  author={Jha, Rishi Dev and Zhang, Collin and Shmatikov, Vitaly and Morris, John Xavier},
  booktitle={The Thirty-ninth Annual Conference on Neural Information Processing Systems},
  year={2025}
}

@article{jiang2023uncovering,
  title={Uncovering meanings of embeddings via partial orthogonality},
  author={Jiang, Yibo and Aragam, Bryon and Veitch, Victor},
  journal={Advances in Neural Information Processing Systems},
  volume={36},
  pages={31988--32005},
  year={2023}
}

@inproceedings{jorgensen2024improving,
  title={Improving Activation Steering in Language Models with Mean-Centring},
  author={Jorgensen, Ole and Cope, Dylan and Schoots, Nandi and Shanahan, Murray},
  booktitle={Responsible Language Models Workshop at AAAI-24},
  year={2024}
}

@inproceedings{kim2018interpretability,
  title={Interpretability beyond feature attribution: Quantitative testing with concept activation vectors (TCAV)},
  author={Kim, Been and Wattenberg, Martin and Gilmer, Justin and Cai, Carrie and Wexler, James and Viegas, Fernanda and others},
  booktitle={International Conference on Machine Learning},
  pages={2668--2677},
  year={2018},
  organization={PMLR}
}

@inproceedings{levy2014linguistic,
  title={Linguistic regularities in sparse and explicit word representations},
  author={Levy, Omer and Goldberg, Yoav},
  booktitle={Proceedings of CoNLL},
  pages={171--180},
  year={2014}
}

@inproceedings{li2020sentence,
  title={On the sentence embeddings from pre-trained language models},
  author={Li, Bohan and Zhou, Hao and He, Junxian and Wang, Mingxuan and Yang, Yiming and Li, Lei},
  booktitle={Proceedings of the 2020 Conference on Empirical Methods in Natural Language Processing (EMNLP)},
  pages={9119--9130},
  year={2020}
}

@inproceedings{li2022emergent,
  title={Emergent World Representations: Exploring a Sequence Model Trained on a Synthetic Task},
  author={Li, Kenneth and Hopkins, Aspen K and Bau, David and Vi{\'e}gas, Fernanda and Pfister, Hanspeter and Wattenberg, Martin},
  booktitle={The Eleventh International Conference on Learning Representations},
  year={2022}
}

@article{li2023inference,
  title={Inference-time intervention: Eliciting truthful answers from a language model},
  author={Li, Kenneth and Patel, Oam and Vi{\'e}gas, Fernanda and Pfister, Hanspeter and Wattenberg, Martin},
  journal={Advances in Neural Information Processing Systems},
  volume={36},
  pages={41451--41530},
  year={2023}
}

@InProceedings{liu2023context,
  title = 	 {In-context Vectors: Making In Context Learning More Effective and Controllable Through Latent Space Steering},
  author =       {Liu, Sheng and Ye, Haotian and Xing, Lei and Zou, James Y.},
  booktitle = 	 {Proceedings of the 41st International Conference on Machine Learning},
  pages = 	 {32287--32307},
  year = 	 {2024},
  editor = 	 {Salakhutdinov, Ruslan and Kolter, Zico and Heller, Katherine and Weller, Adrian and Oliver, Nuria and Scarlett, Jonathan and Berkenkamp, Felix},
  volume = 	 {235},
  series = 	 {Proceedings of Machine Learning Research},
  month = 	 {21--27 Jul},
  publisher =    {PMLR},
  url = 	 {https://proceedings.mlr.press/v235/liu24bx.html}
}

@inproceedings{marelli2014sick,
  author    = {Marelli, Marco and Menini, Stefano and Baroni, Marco and Bentivogli, Luisa and Bernardi, Raffaella and Zamparelli, Roberto},
  title     = {A SICK cure for the evaluation of compositional distributional semantic models},
  booktitle = {Proceedings of the Ninth International Conference on Language Resources and Evaluation (LREC'14)},
  year      = {2014},
  address   = {Reykjavik, Iceland},
  publisher = {European Language Resources Association (ELRA)}, 
  pages     = {216--223}
}

@inproceedings{merullo2024language,
  title={Language models implement simple word2vec-style vector arithmetic},
  author={Merullo, Jack and Eickhoff, Carsten and Pavlick, Ellie},
  booktitle={Proceedings of the 2024 Conference of the North American Chapter of the Association for Computational Linguistics: Human Language Technologies (Volume 1: Long Papers)},
  pages={5030--5047},
  year={2024}
}

@inproceedings{mikolov2013efficient,
  title={Efficient estimation of word representations in vector space},
  author={Mikolov, Tomas and Chen, Kai and Corrado, Greg and Dean, Jeffrey},
  booktitle={Proceedings of Workshop at ICLR},
  year={2013}
}

@inproceedings{mikolov2013linguistic,
  title={Linguistic regularities in continuous space word representations},
  author={Mikolov, Tomas and Yih, Wen-tau and Zweig, Geoffrey},
  booktitle={Proceedings of the 2013 Conference of the North American Chapter of the Association for Computational Linguistics: Human Language Technologies},
  pages={746--751},
  year={2013}
}

@inproceedings{mimno2017strange,
  title={The strange geometry of skip-gram with negative sampling},
  author={Mimno, David and Thompson, Laure},
  booktitle={Conference on Empirical Methods in Natural Language Processing},
  year={2017}
}

@inproceedings{morris2020linearity,
  title={The Linearity of Cross-Lingual Word Embeddings: A Geometric Analysis},
  author={Morris, John X and Bommasani, Rishi and Naik, Aakanksha and Rush, Alexander M},
  booktitle={Proceedings of the 2020 Conference on Empirical Methods in Natural Language Processing (EMNLP)},
  pages={7955--7964},
  year={2020},
  publisher={Association for Computational Linguistics},
  doi={10.18653/v1/2020.emnlp-main.641}
}

@inproceedings{nanda2023emergent,
  title={Emergent linear representations in world models of self-supervised sequence models},
  author={Nanda, Neel and Lee, Andrew and Wattenberg, Martin},
  booktitle={Proceedings of the 6th BlackboxNLP Workshop: Analyzing and Interpreting Neural Networks for NLP},
  pages={16--30},
  year={2023}
}

@misc{openai-text-embedding-3-large-2024,
  title        = {text‐embedding‐3‐large},
  author       = {{OpenAI}},
  howpublished = {OpenAI API models announcement},
  note         = {announced January 25, 2024; 3072 dimensions, improved performance on MIRACL and MTEB benchmarks. Available from OpenAI API documentation} ,
  year         = {2024},
}

@inproceedings{park2025,
  title={The Geometry of Categorical and Hierarchical Concepts in Large Language Models},
  author={Park, Kiho and Choe, Yo Joong and Jiang, Yibo and Veitch, Victor},
  booktitle={The Thirteenth International Conference on Learning Representations},
  year={2025}
}

@inproceedings{parklinear,
  title={The Linear Representation Hypothesis and the Geometry of Large Language Models},
  author={Park, Kiho and Choe, Yo Joong and Veitch, Victor},
  booktitle={International Conference on Machine Learning},
  year={2024}
}

@inproceedings{pham2024householder,
  title={Householder Pseudo-Rotation: A Novel Approach to Activation Editing in LLMs with Direction-Magnitude Perspective},
  author={Pham, Van-Cuong and Nguyen, Thien},
  booktitle={Proceedings of the 2024 Conference on Empirical Methods in Natural Language Processing},
  pages={13737--13751},
  year={2024}
}

@inproceedings{reif2019visualizing,
  title={Visualizing and measuring the geometry of BERT},
  author={Reif, Emily and Yuan, Ann and Wattenberg, Martin and Viegas, Fernanda B and Coenen, Andy and Pearce, Adam and Kim, Been},
  booktitle={Advances in Neural Information Processing Systems},
  volume={32},
  year={2019}
}

@inproceedings{reimers2019sentencebert,
  title     = {Sentence-BERT: Sentence Embeddings using Siamese BERT-Networks},
  author    = {Reimers, Nils and Gurevych, Iryna},
  booktitle = {Proceedings of the 2019 Conference on Empirical Methods in Natural Language Processing (EMNLP)},
  pages     = {3980--3990},
  year      = {2019},
  doi       = {10.18653/v1/D19-1410}
}

@article{rimsky2023steering,
  title={Steering Llama 2 via Contrastive Activation Addition},
  author={Rimsky, Nina and Gabrieli, Nick and Schulz, Julian and Tong, Meg and Hubinger, Evan and Turner, Alexander Matt},
  journal={arXiv preprint arXiv:2312.06681},
  year={2023}
}

@article{rogers2021primer,
  title={A primer in BERTology: What we know about how BERT works},
  author={Rogers, Anna and Kovaleva, Olga and Rumshisky, Anna},
  journal={Transactions of the association for computational linguistics},
  volume={8},
  pages={842--866},
  year={2021},
  publisher={MIT Press One Rogers Street, Cambridge, MA 02142-1209, USA journals-info~…}
}

@article{tan2024analysing,
  title={Analysing the generalisation and reliability of steering vectors},
  author={Tan, Daniel and Chanin, David and Lynch, Aengus and Paige, Brooks and Kanoulas, Dimitrios and Garriga-Alonso, Adri{\`a} and Kirk, Robert},
  journal={Advances in Neural Information Processing Systems},
  volume={37},
  pages={139179--139212},
  year={2024}
}

@article{touvron2023llama,
  title={Llama 2: Open foundation and fine-tuned chat models},
  author={Touvron, Hugo and Martin, Louis and Stone, Kevin and Albert, Peter and Almahairi, Amjad and Babaei, Yasmine and Bashlykov, Nikolay and Batra, Soumya and Bhargava, Prajjwal and Bhosale, Shruti and others},
  journal={arXiv preprint arXiv:2307.09288},
  year={2023}
}

@inproceedings{trager2023linear,
  title={Linear spaces of meanings: compositional structures in vision-language models},
  author={Trager, Matthew and Perera, Pramuditha and Zancato, Luca and Achille, Alessandro and Bhatia, Parminder and Soatto, Stefano},
  booktitle={Proceedings of the IEEE/CVF International Conference on Computer Vision},
  pages={15395--15404},
  year={2023}
}

@article{turner2023activation,
  title={Activation Addition: Steering Language Models Without Optimization},
  author={Turner, Alexander Matt and Thiergart, Lisa and Udell, David and Leech, Gavin and Mini, Ulisse and MacDiarmid, Monte},
  journal={arXiv preprint arXiv:2308.10248},
  year={2023}
}

@inproceedings{wang2023concept,
  title={Concept algebra for (score-based) text-controlled generative models},
  author={Wang, Zihao and Gui, Lin and Negrea, Jeffrey and Veitch, Victor},
  booktitle={Advances in Neural Information Processing Systems},
  volume={36},
  pages={35331--35349},
  year={2023}
}

@article{warstadt2020blimp,
    author = {Warstadt, Alex and Parrish, Alicia and Liu, Haokun and Mohananey, Anhad and Peng, Wei and Wang, Sheng-Fu and Bowman, Samuel R.},
    title = {BLiMP: The Benchmark of Linguistic Minimal Pairs for English},
    journal = {Transactions of the Association for Computational Linguistics},
    volume = {8},
    pages = {377-392},
    year = {2020},
    doi = {10.1162/tacl_a_00321},
    URL = {https://doi.org/10.1162/tacl_a_00321},
    eprint = {https://doi.org/10.1162/tacl_a_00321}
}

@article{zou2023representation,
  title={Representation Engineering: A Top-Down Approach to AI Transparency},
  author={Zou, Andy and Phan, Long and Chen, Sarah and Campbell, James and Guo, Phillip and Ren, Richard and Pan, Alexander and Yin, Xuwang and Mazeika, Mantas and Dombrowski, Ann-Kathrin and others},
  journal={arXiv preprint arXiv:2310.01405},
  year={2023}
}

@book{Absil2008,
  title={Optimization Algorithms on Matrix Manifolds},
  author={Absil, P.-A. and Mahony, R. and Sepulchre, R.},
  year={2008},
  publisher={Princeton University Press}
}

@inproceedings{mitchell-lapata-2008-vector,
    title = "Vector-based Models of Semantic Composition",
    author = "Mitchell, Jeff  and
      Lapata, Mirella",
    editor = "Moore, Johanna D.  and
      Teufel, Simone  and
      Allan, James  and
      Furui, Sadaoki",
    booktitle = "Proceedings of ACL-08: HLT",
    month = jun,
    year = "2008",
    address = "Columbus, Ohio",
    publisher = "Association for Computational Linguistics",
    url = "https://aclanthology.org/P08-1028/",
    pages = "236--244"
}

@inproceedings{baroni-zamparelli-2010-nouns,
    title = "Nouns are Vectors, Adjectives are Matrices: Representing Adjective-Noun Constructions in Semantic Space",
    author = "Baroni, Marco  and
      Zamparelli, Roberto",
    editor = "Li, Hang  and
      M{\`a}rquez, Llu{\'i}s",
    booktitle = "Proceedings of the 2010 Conference on Empirical Methods in Natural Language Processing",
    month = oct,
    year = "2010",
    address = "Cambridge, MA",
    publisher = "Association for Computational Linguistics",
    url = "https://aclanthology.org/D10-1115/",
    pages = "1183--1193"
}

@inproceedings{artetxe2018robust,
  title={A robust self-learning method for fully unsupervised cross-lingual mappings of word embeddings},
  author={Artetxe, Mikel and Labaka, Gorka and Agirre, Eneko},
  booktitle={Proceedings of the 56th Annual Meeting of the Association for Computational Linguistics (Volume 1: Long Papers)},
  pages={789--798},
  year={2018}
}

@inproceedings{conneau2018you,
  title={What you can cram into a single vector: Probing sentence embeddings for linguistic properties},
  author={Conneau, Alexis and Kruszewski, German and Lample, Guillaume and Barrault, Lo{\"\i}c and Baroni, Marco},
  booktitle={ACL 2018-56th Annual Meeting of the Association for Computational Linguistics},
  volume={1},
  pages={2126--2136},
  year={2018},
  organization={Association for Computational Linguistics}
}
\bibliographystyle{iclr2026_conference}

\appendix
\section{Mathematical Properties of RISE}
\label{app:rise_theory}
These mathematical results support our main claims in the paper. 
Lemma~\ref{lem:exp_log} provides the explicit exponential and logarithmic map formulas that underlie RISE’s use of geodesics on the unit hypersphere. 
Theorem~\ref{thm:rise_commute} formalizes that sequential RISE edits commute up to second order, showing that different discourse-level transformations can be applied in any order without significant distortion. This result highlights the local geometric consistency of RISE transformations, rather than implying global additive steering.
Proposition~\ref{prop:complexity} shows that each RISE transformation can be applied in $O(d)$ time and memory, demonstrating the method’s scalability to modern high-dimensional embeddings. 
Together, these results provide theoretical grounding for both the geometric consistency and the practical efficiency reported in the main text.

\subsection{Geometry Preliminaries on the Sphere}

We work on the unit sphere $\mathbb{S}^{d-1}\subset\mathbb{R}^d$ with the standard round metric. 
For $n\in\mathbb{S}^{d-1}$, the tangent space is $T_n\mathbb{S}^{d-1}=\{x\in\mathbb{R}^d:\langle x,n\rangle=0\}$. 
The exponential map $\exp_n:T_n\mathbb{S}^{d-1}\to\mathbb{S}^{d-1}$ is defined for all tangent vectors, while the logarithmic map $\log_n$ is well-defined for all $v\in\mathbb{S}^{d-1}$ except the antipode $v=-n$. 
For each $n$, fix an orthogonal map $R(n)\in O(d)$ such that $R(n)n=e_1$, where $e_1=(1,0,\dots,0)^\top$. 
When analyzing local behavior (e.g., Theorem~\ref{thm:rise_commute}), we take $R(\cdot)$ to be any $C^1$ (continuously differentiable) choice on a neighborhood of the geodesic segment(s) under consideration; such a local choice always exists.

\begin{lemma}[Exponential and logarithmic maps on the unit sphere]
\label{lem:exp_log}
For $n\in\mathbb{S}^{d-1}$, tangent vector $\xi\in T_n\mathbb{S}^{d-1}$, and point $v\in\mathbb{S}^{d-1}\setminus\{-n\}$,
\[
\exp_n(\xi) = \cos(\|\xi\|)\,n + \sin(\|\xi\|)\,\frac{\xi}{\|\xi\|}, 
\qquad
\log_n(v) = \arccos(\langle n,v\rangle)\,\frac{v-\langle n,v\rangle n}{\|v-\langle n,v\rangle n\|}.
\]
\end{lemma}

\begin{proof}
These formulas follow from the fact that geodesics on $\mathbb{S}^{d-1}$ are great circles in $\mathbb{R}^d$ (unit-radius sphere). 
See, e.g., \citet[Sec.~5.4]{Absil2008}.
\end{proof}

\subsection{Rotor Construction and Implementation}
\label{app:rotor}

In Clifford algebra terms, a \emph{rotor} is an element of $\mathrm{Spin}(d)$ that rotates vectors by the sandwich product $x \mapsto r x \tilde r$, where $\tilde r$ denotes reversion. For our purposes, we only require an orthogonal operator $R(n)\in O(d)$ with $R(n)n=e_1$ that depends smoothly on $n$. One closed-form rotor mapping $n\mapsto e_1$ (valid when $n\neq -e_1$) is
\[
r(n) \;=\; \frac{1 + e_1 n}{\sqrt{2(1+\langle e_1,n\rangle)}}, \qquad
r(n)\,n\,\tilde r(n) = e_1.
\]
In practice we realize this as a standard linear operator without explicit
Clifford algebra structures. Two efficient $O(d)$ realizations are:

\begin{itemize}
\item \textbf{Householder reflection:}
$H(n) = I - 2\frac{ww^\top}{\|w\|^2}$ with $w=n-e_1$, which satisfies
$H(n)n=e_1$ (determinant $-1$).
\item \textbf{Givens rotation:} a $2\times 2$ rotation in the plane spanned
by $\{n,e_1\}$, extended by the identity elsewhere, with determinant $+1$.
\end{itemize}

Both satisfy the required conditions $R(n)n=e_1$ and local $C^1$ smoothness,
and are numerically stable away from $n\approx -e_1$. In the antipodal case
($n\approx -e_1$) we use a two-step construction: map $n$ to an auxiliary
orthogonal vector $u\perp e_1$, then $u$ to $e_1$. In all cases, applying
$R(n)$ or $R(n)^\top$ to a vector costs $O(d)$ operations.

\subsection{Commutativity Properties of Sequential RISE Operations}

\subsubsection{The RISE Sequential Procedure}

Given $n_0\in\mathbb{S}^{d-1}$ and prototypes $\vec{p}_A,\vec{p}_B\in T_{e_1}\mathbb{S}^{d-1}$:
\[
\textbf{Apply A:}\;\;\xi_A=R(n_0)^\top \vec{p}_A,\;n_1=\exp_{n_0}(\xi_A),
\qquad
\textbf{Apply B:}\;\;\xi_B=R(n_1)^\top \vec{p}_B,\;n_2=\exp_{n_1}(\xi_B).
\]

\subsubsection{First-Order Commutativity Analysis}

\begin{theorem}[RISE commutativity to first order]
\label{thm:rise_commute}
For small prototype magnitudes $\|\vec{p}_A\|,\|\vec{p}_B\|\ll 1$,
\[
d\!\left(\text{result of }A\circ B,\;\text{result of }B\circ A\right)
= O(\|\vec{p}_A\|\cdot\|\vec{p}_B\|).
\]
\end{theorem}

\begin{proof}
Using Lemma~\ref{lem:exp_log}, expand $\exp_{n_0}(\xi_A)=n_0+\xi_A+O(\|\xi_A\|^2)$. 
Let $\eta_A=\xi_A$. 
Canonicalization at $n_1=n_0+\eta_A+O(\|\eta_A\|^2)$ differs from that at $n_0$ by $O(\|\eta_A\|)$. 

Let $P_{n_1\to n_0}:T_{n_1}\mathbb{S}^{d-1}\to T_{n_0}\mathbb{S}^{d-1}$ denote parallel transport along the short geodesic from $n_1$ to $n_0$. 
On the unit sphere, $\|P_{n_1\to n_0}-I\|=O(\|n_1-n_0\|)=O(\|\eta_A\|)$, where $I$ denotes the identity operator on the tangent space. 
With a $C^1$ choice of $R(\cdot)$, $\|R(n_1)^\top-R(n_0)^\top\|=O(\|n_1-n_0\|)=O(\|\eta_A\|)$. 
Therefore,
\[
P_{n_1\to n_0}\,R(n_1)^\top \vec p_B \;=\; R(n_0)^\top \vec p_B \;+\; O(\|\eta_A\|\,\|\vec p_B\|).
\]

Now expand the second step:
\[
n_2 = n_0 + R(n_0)^\top(\vec{p}_A+\vec{p}_B) + O(\|\vec{p}_A\|\|\vec{p}_B\|) + O(\|\vec{p}_A\|^2+\|\vec{p}_B\|^2).
\]
Swapping roles of $A$ and $B$ gives the same expansion with $\vec{p}_A,\vec{p}_B$ reversed. 
Subtracting yields a difference of order $\|\vec{p}_A\|\|\vec{p}_B\|$.
\end{proof}

\paragraph{Geometric interpretation.}
Re-canonicalization is equivalent (to first order) to parallel-transporting the next step’s vector back to the initial tangent space. 
On $\mathbb{S}^{d-1}$ with constant curvature, order effects are second order.

\subsection{Computational Complexity}

\begin{proposition}[Per-transformation complexity]
\label{prop:complexity}
Each RISE transformation can be implemented in $O(d)$ time and $O(d)$ memory:
\begin{enumerate}
    \item Canonicalization: applying $R(n)$ or $R(n)^\top$ costs $O(d)$.
    \item Logarithmic map $\log_n(v)$: $O(d)$ using Lemma~\ref{lem:exp_log}.
    \item Exponential map $\exp_n(\xi)$: $O(d)$ using Lemma~\ref{lem:exp_log}.
    \item Storage: prototype $\hat{\vec{p}}\in T_{e_1}\mathbb{S}^{d-1}$ costs $O(d)$.
\end{enumerate}
\end{proposition}

\paragraph{Comparison with matrix methods.}  
Dense $d\times d$ rotations require $O(d^2)$ time and memory. 
RISE achieves equivalent updates in $O(d)$.

\paragraph{Implementation note (Householder).}  
A practical canonicalization is the Householder reflection
\[
H(n)=I-2\frac{ww^\top}{\|w\|^2},\quad w=n-e_1,
\]
which maps $n\mapsto e_1$ in $O(d)$. 
Since $H(n)$ is a reflection ($\det=-1$), it suffices for canonicalization. 
Near $n\approx e_1$, one may switch to a numerically stable alternative.

\section{Cross-Language Transfer Analysis and Results}
\label{app:heatmaps}

To test whether geometric transformations generalize across languages, we conducted comprehensive cross-language transfer experiments. This section reports detailed results across three models and three semantic phenomena, analyzing both quantitative performance and geometric properties of learned transformations.

\begin{figure}[H]
    \label{fig:bge_m3_heatmaps}
    \centering
    \includegraphics[width=0.32\textwidth]{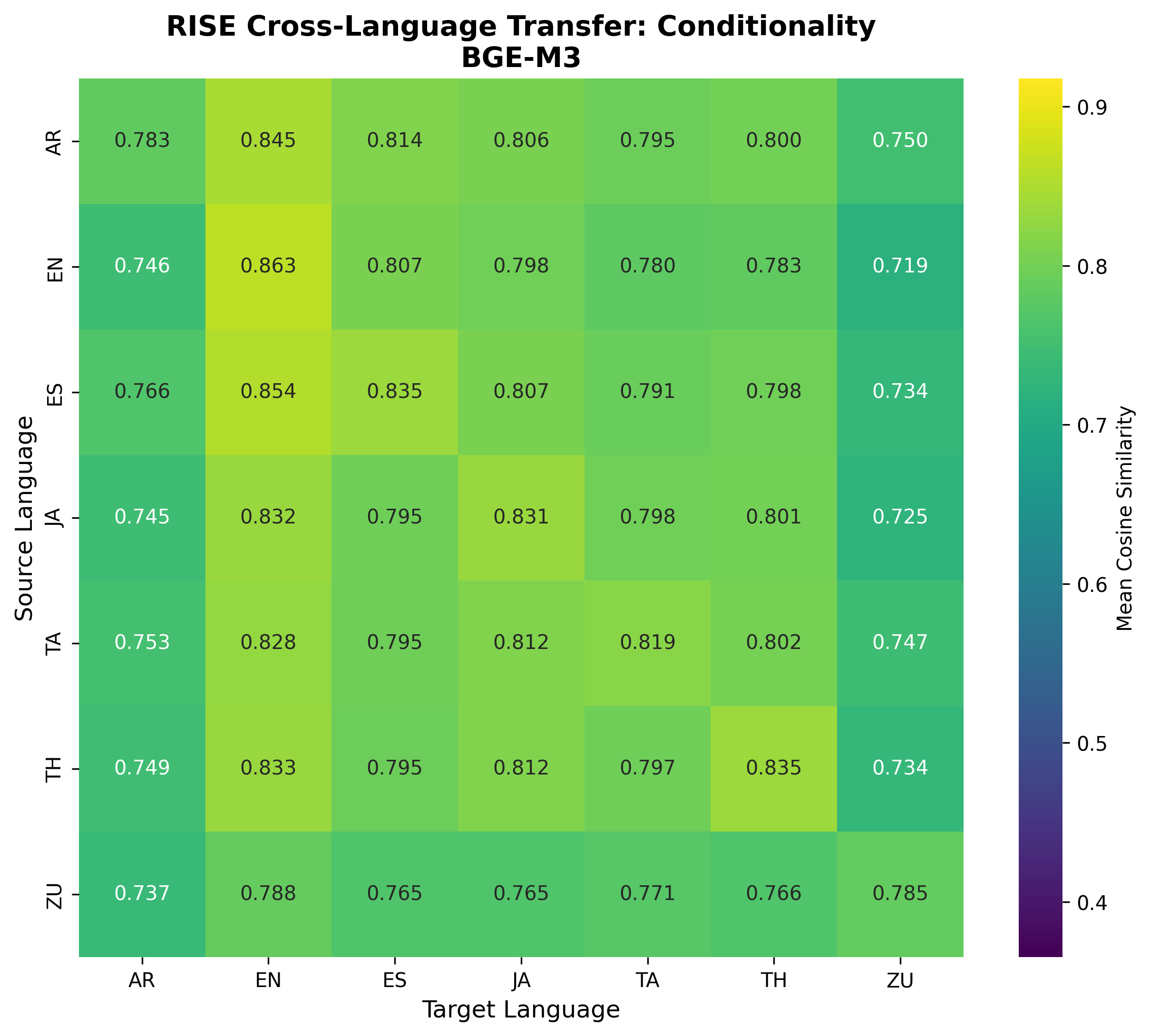}
    \includegraphics[width=0.32\textwidth]{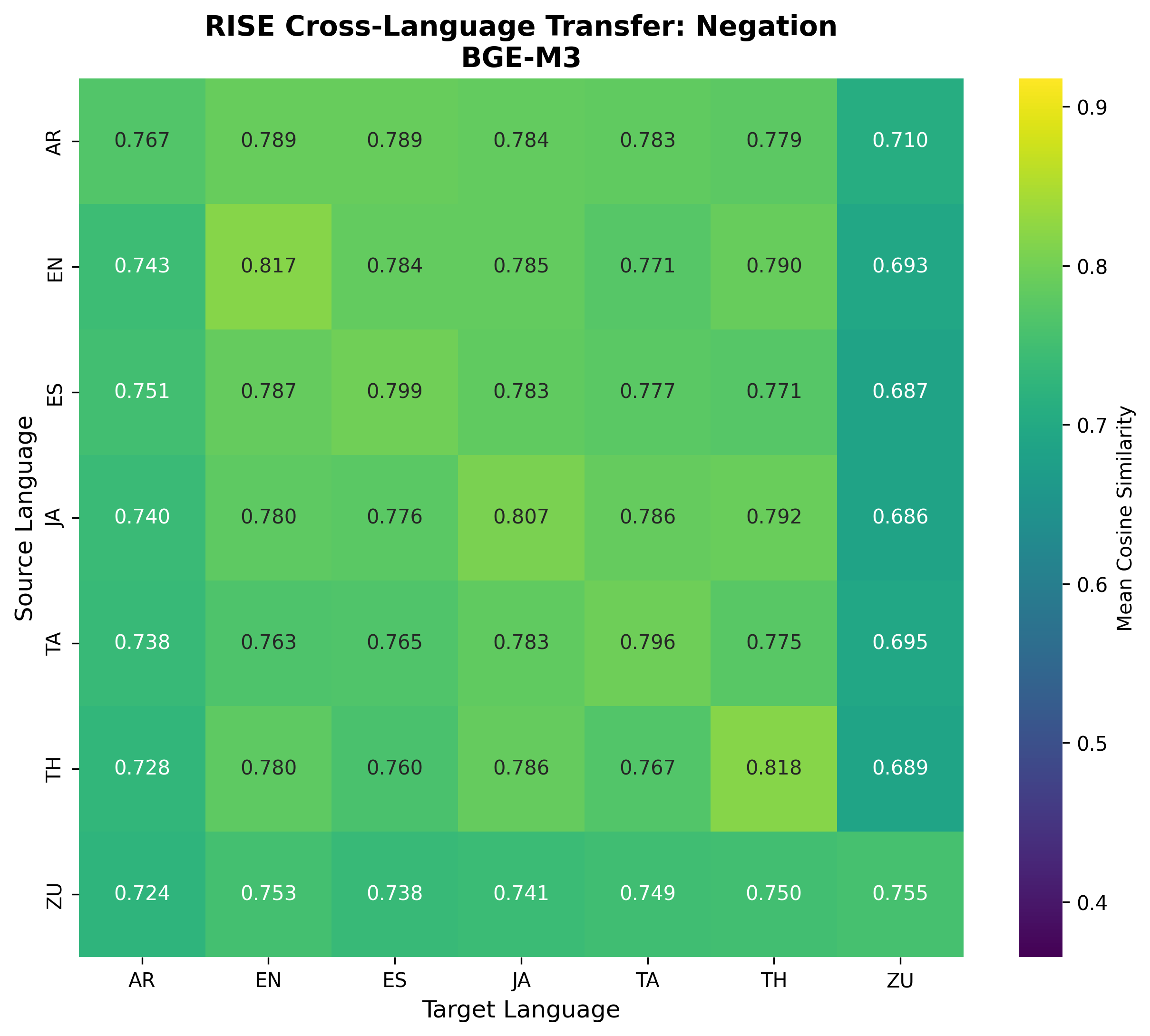}
    \includegraphics[width=0.32\textwidth]{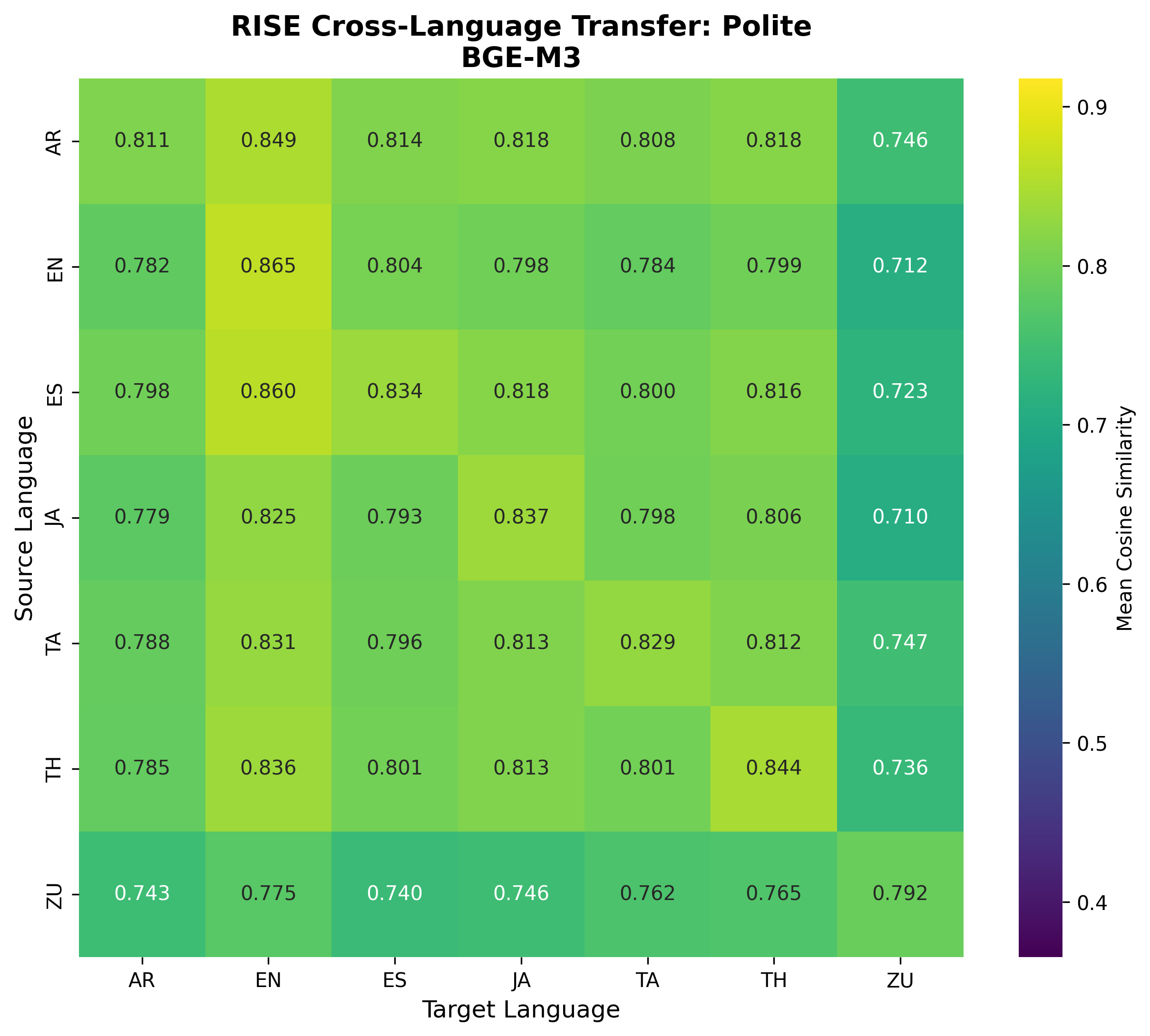}
    \caption{Cross-language transfer heatmaps for bge-m3 model showing RISE performance across all language pairs for conditionality, negation, and politeness transformations. Darker colors indicate higher cosine similarity between predicted and target embeddings.}
\end{figure}

\begin{figure}[H]
    \centering
    \includegraphics[width=0.32\textwidth]{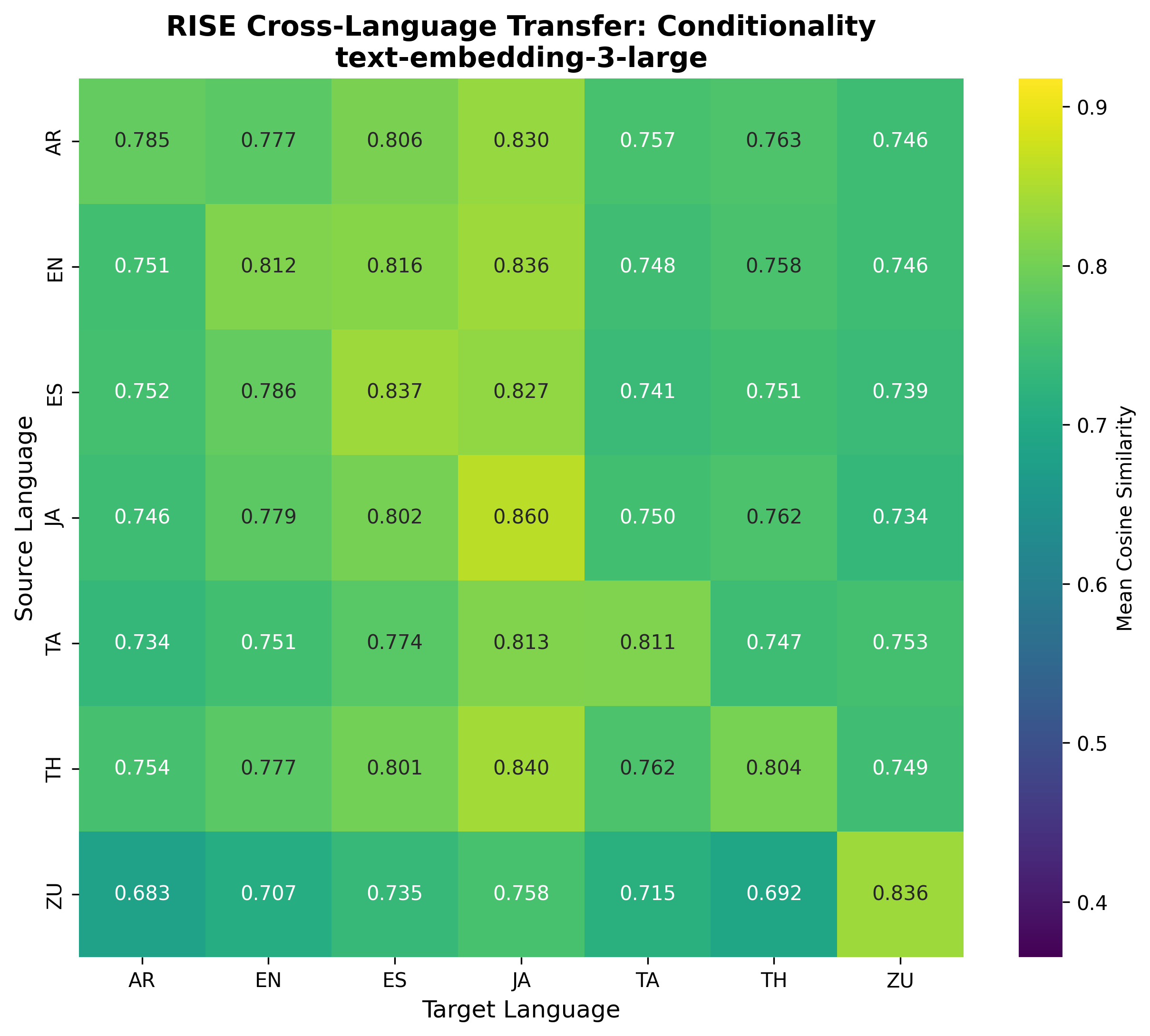}
    \includegraphics[width=0.32\textwidth]{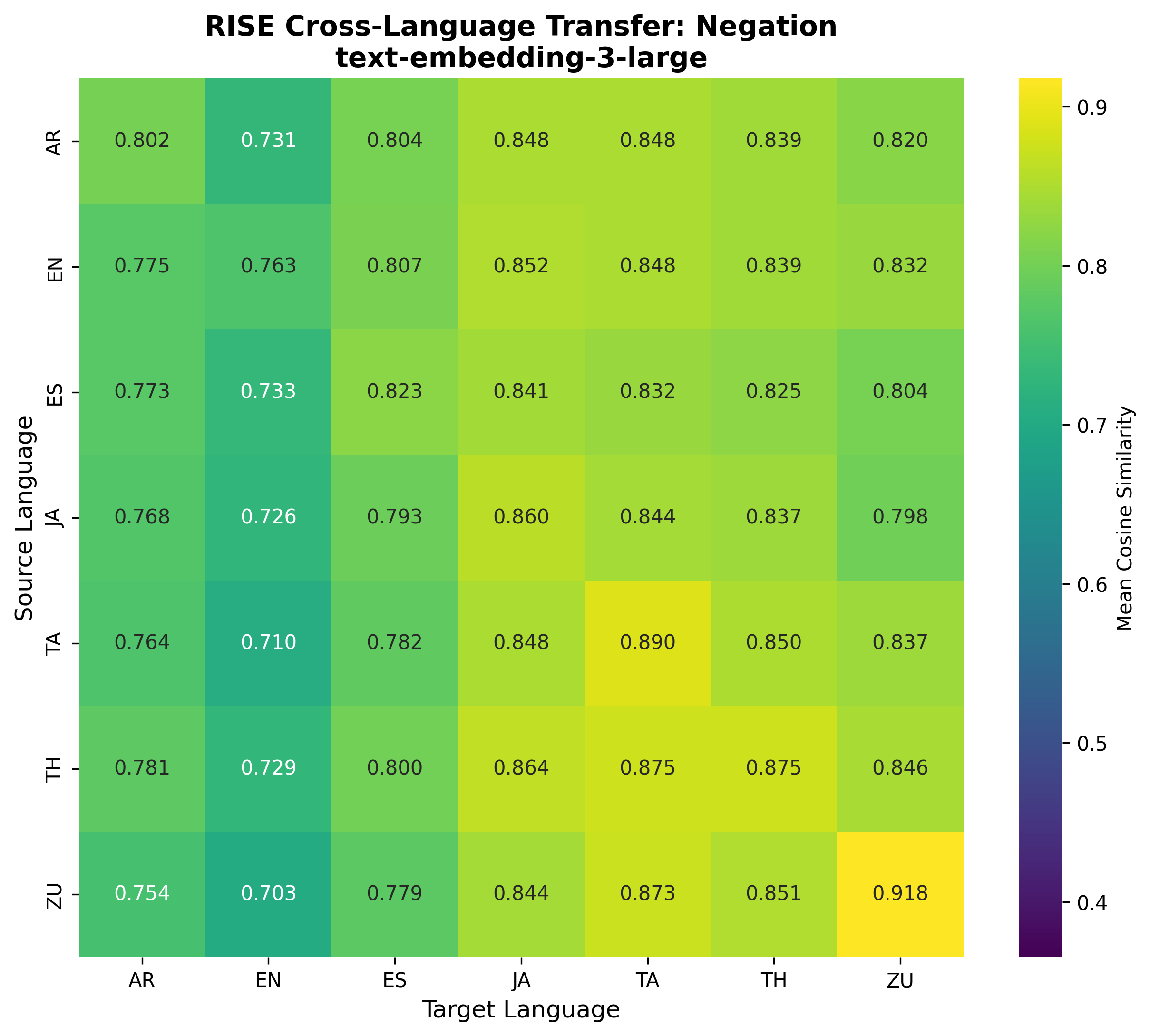}
    \includegraphics[width=0.32\textwidth]{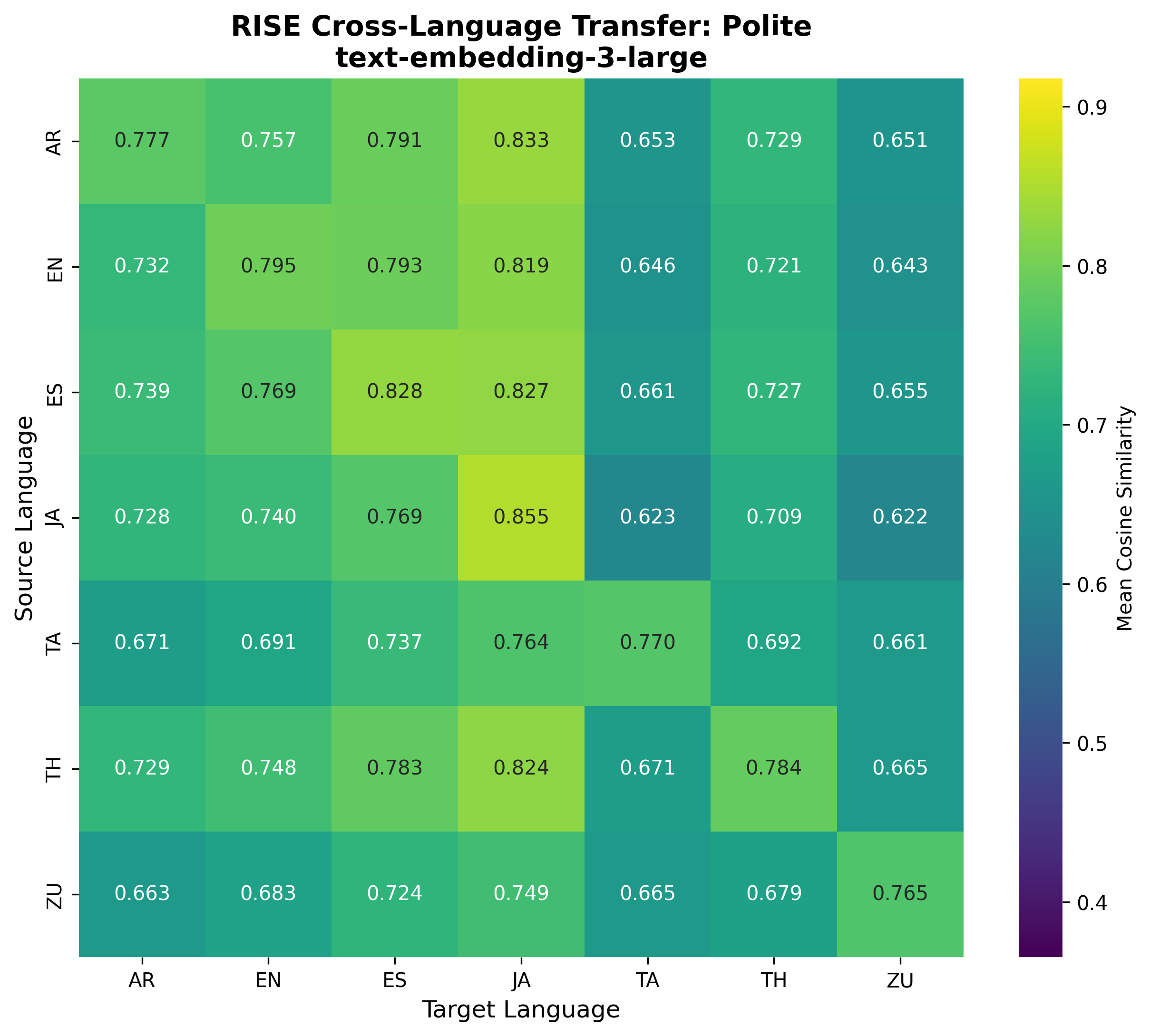}
    \caption{Cross-language transfer heatmaps for text-embedding-3-large model showing RISE performance across all language pairs for conditionality, negation, and politeness transformations. Darker colors indicate higher cosine similarity between predicted and target embeddings.}
\end{figure}

\begin{figure}[H]
    \label{fig:mbert_heatmaps}
    \centering
    \includegraphics[width=0.32\textwidth]{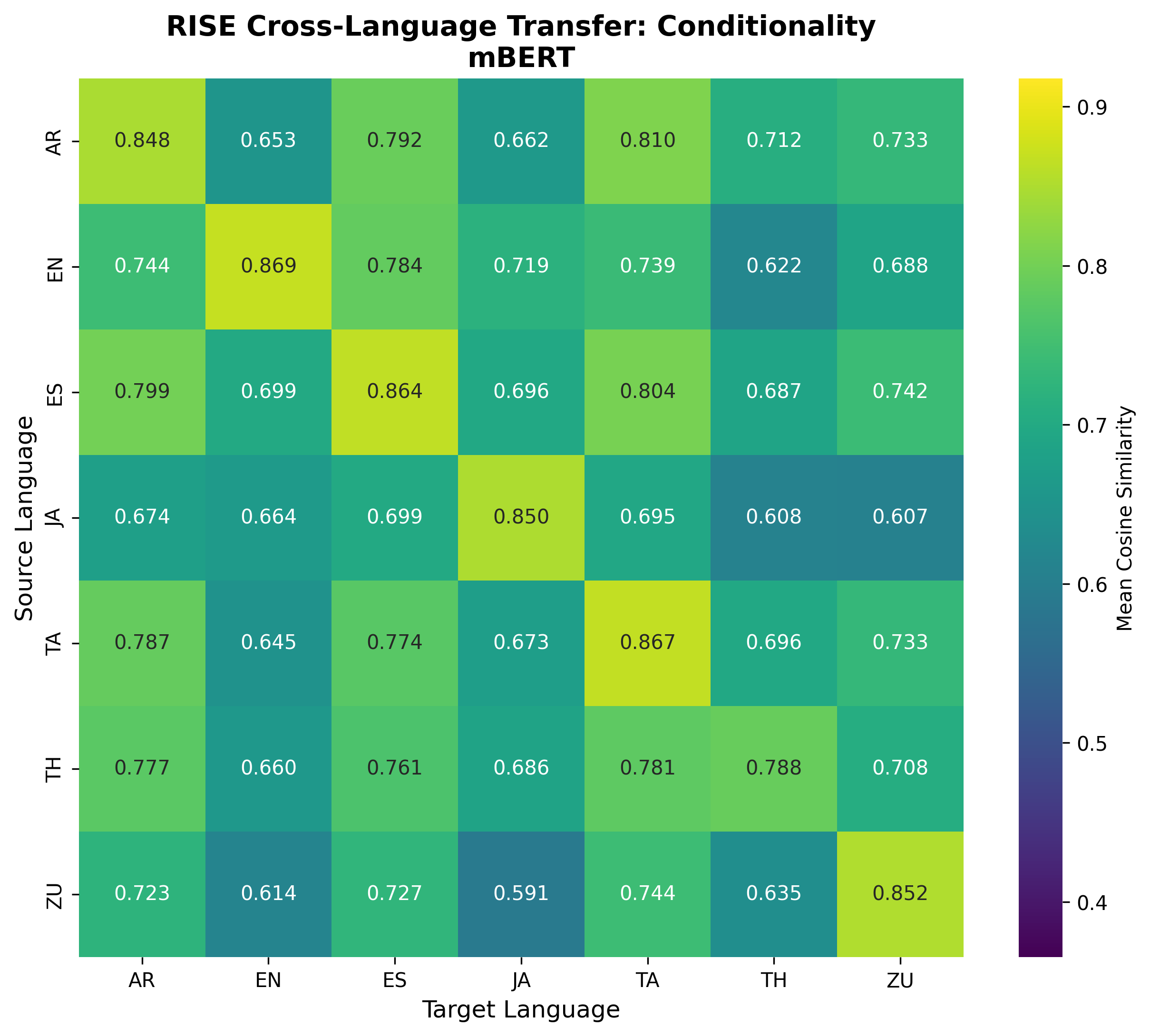}
    \includegraphics[width=0.32\textwidth]{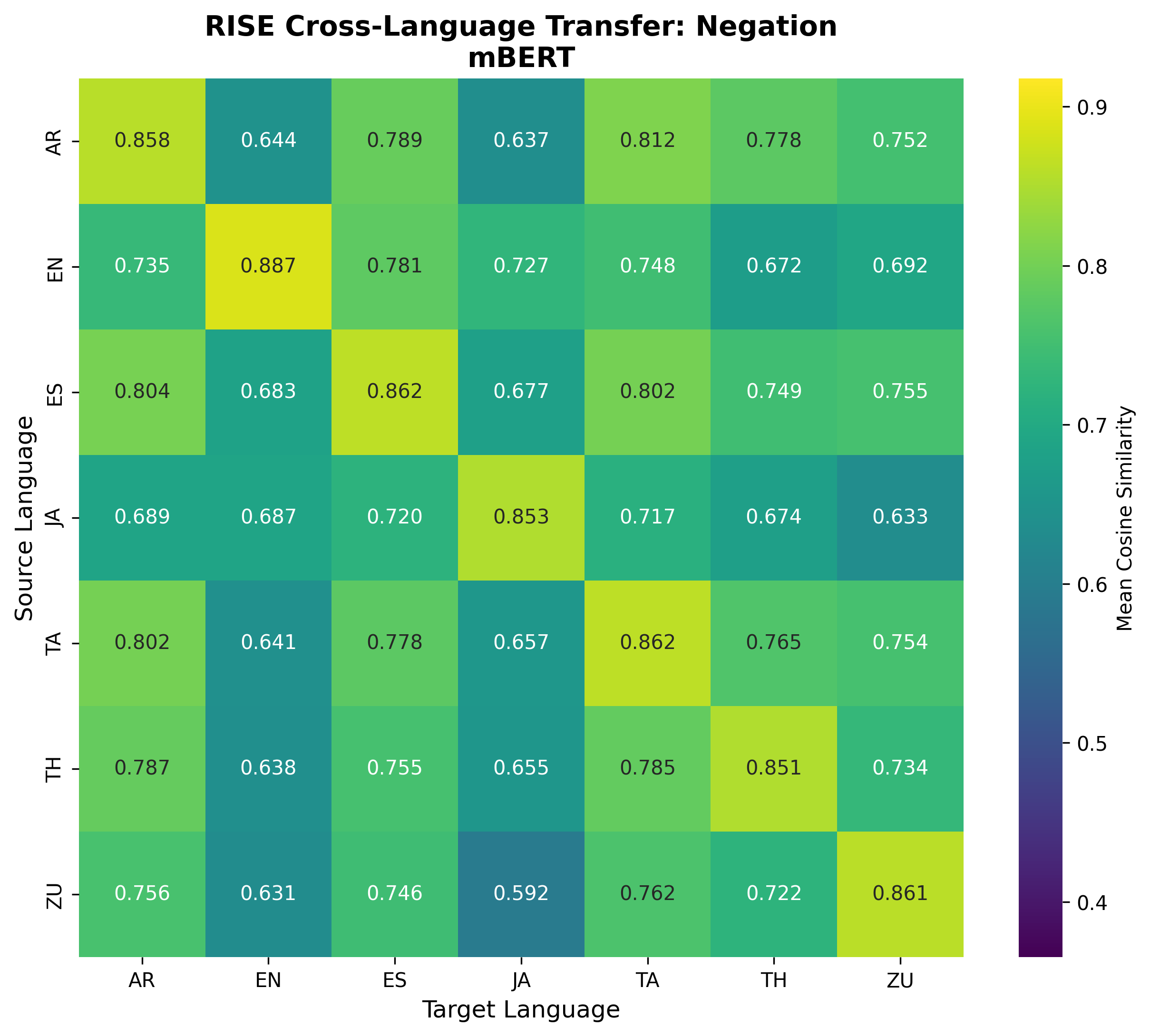}
    \includegraphics[width=0.32\textwidth]{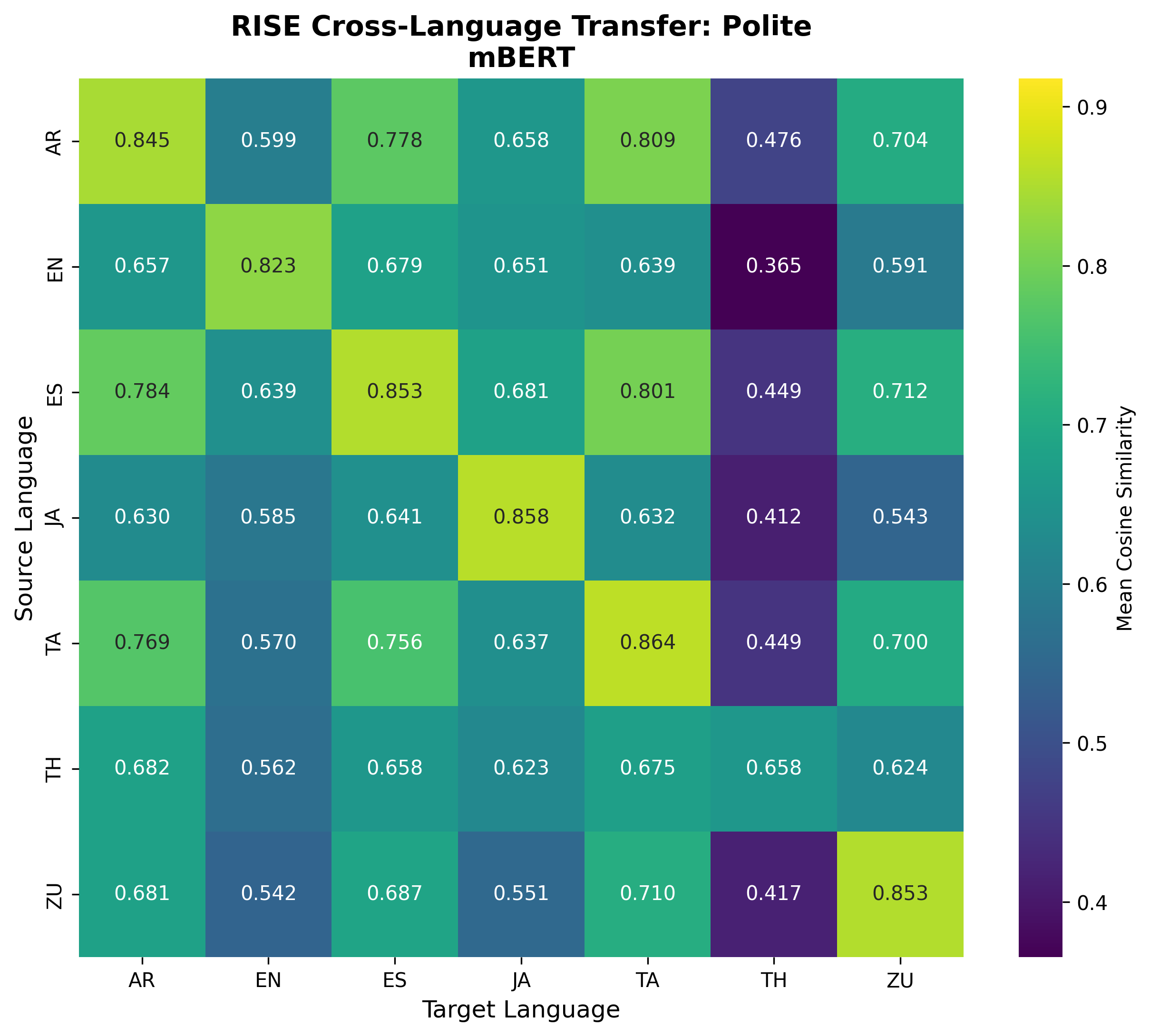}
    \caption{Cross-language transfer heatmaps for mBERT model showing RISE performance across all language pairs for conditionality, negation, and politeness transformations. Darker colors indicate higher cosine similarity between predicted and target embeddings.}
\end{figure}

\subsection{Cross-Language Transfer Performance}

The above heatmaps demonstrate comprehensive cross-language transfer results across our three models. Training rotor prototypes on one language and evaluating on others reveals promising cross-linguistic performance, particularly for negation and conditionality. Most language pairs show transfer scores above 0.70, with negation achieving particularly strong off-diagonal performance (most scores $>$ 0.80). 

\textbf{Negation} emerges as the most performant transformation, achieving the highest mean cross-language transfer scores (0.788 across all model-language combinations) with performance ranging from 0.686 to 0.918. 

\textbf{Conditionality} demonstrates the highest stability and consistency across cross-language transfers, with the lowest performance variability (0.038) and most stable individual measurements (0.056 average std deviation). Mean performance of 0.780 places it second overall.

\textbf{Politeness} shows more variation but still achieves substantial cross-linguistic success (most scores $>$ 0.70). 

\subsection{Geometric Analysis of Cross-Language Centroids}

Analysis of the learned centroids reveals additional insights into the geometric structure of semantic transformations. For each phenomenon, we computed ``ideal'' transformation vectors by averaging canonicalized transformed embeddings across languages.

For \textbf{negation}, the centroids show high similarity across languages (pairwise cosines $>$ 0.95). 

\textbf{Conditionality} centroids maintain high geometric consistency, supporting the observed stability in transfer performance across all model-language combinations.

\textbf{Politeness} centroids cluster more loosely but still maintain substantial similarity (pairwise cosines $>$ 0.87).

\subsection{Quantitative Cross-Language Analysis}

\begin{table}[htbp]
\centering
\caption{Complete Cross-Language Transfer Matrix: Statistical Summary}
\label{tab:cross_language_matrix}
\begin{adjustbox}{width=\textwidth,center}
\begin{tabular}{lccccc}
\toprule
\textbf{Model} & \textbf{Phenomenon} & \textbf{All Transfers} & \textbf{Monolingual} & \textbf{Cross-Lang} & \textbf{Ratio} \\
\midrule
\multirow{3}{*}{TE3L (3072d)} 
& Conditionality & 13.6× ± 0.7 & 14.5× & 13.5× & 0.93 \\
& Negation & 19.7× ± 1.2 & 20.6× & 19.6× & 0.95 \\
& Politeness & 23.1× ± 1.9 & 25.3× & 22.8× & 0.90 \\
\midrule
\multirow{3}{*}{bge-m3 (1024d)} 
& Conditionality & 13.9× ± 0.6 & 14.5× & 13.8× & 0.95 \\
& Negation & 18.5× ± 0.8 & 19.3× & 18.4× & 0.95 \\
& Politeness & 25.2× ± 1.2 & 26.4× & 25.1× & 0.95 \\
\midrule
\multirow{3}{*}{mBERT (768d)} 
& Conditionality & 12.8× ± 1.3 & 15.0× & 12.5× & 0.83 \\
& Negation & 18.0× ± 1.8 & 20.9× & 17.5× & 0.84 \\
& Politeness & 20.8× ± 3.9 & 26.1× & 20.0× & 0.77 \\
\bottomrule
\end{tabular}
\end{adjustbox}
\begin{tablenotes}
\small
\item Statistics computed across complete 7×7 language transfer matrix (49 language pairs per phenomenon).
\item Values show advantage ratios ± standard deviation across all language pairs.
\item Ratio indicates relative cross-language transfer effectiveness (Cross-Lang/Monolingual).
\item All models maintain strong cross-language performance (77\%--95\% of monolingual performance).
\end{tablenotes}
\end{table}

\begin{table}[htbp]
\centering
\caption{Model Architecture and Overall RISE Performance Summary}
\label{tab:model_summary}
\begin{adjustbox}{width=\textwidth,center}
\begin{tabular}{lcccc}
\toprule
\textbf{Model} & \textbf{Dims} & \textbf{Validation Avg} & \textbf{Cross-Lang Avg} & \textbf{Random Adv} \\
\midrule
TE3L & 3072 & 0.774 & 19.0× & 6.3× \\
bge-m3 & 1024 & \textbf{0.790} & \textbf{19.8×} & 11.7× \\
mBERT & 768 & \textbf{0.802} & 16.9× & \textbf{11.9×} \\
\bottomrule
\end{tabular}
\end{adjustbox}
\begin{tablenotes}
\small
\item Validation Avg: Mean performance across Synthetic Multilingual, BLiMP, and SICK datasets.
\item Cross-Lang Avg: Mean advantage ratio across English→Spanish and Japanese→English transfers.
\item Random Adv: Mean advantage ratio over random baselines in monolingual English scenarios.
\item Bold values indicate best performance in each category.
\end{tablenotes}
\end{table}

Tables \ref{tab:cross_language_matrix} and \ref{tab:model_summary} provide comprehensive quantitative analysis of cross-language transfer performance. Notably, all models maintain strong cross-language performance (77\%--95\% of monolingual performance), with bge-m3 showing the most consistent cross-language effectiveness across all phenomena.

\newpage
\section{Linear Baselines Comparisons}
\label{app:lin_baselines}

This appendix reports the full results for the linear baseline comparisons requested by the reviewers. We thank the reviewers for this valuable suggestion as these results did strengthen our paper. We implemented two  baselines: Procrustes alignment and Mean Difference Vectors (MDV).
MDV is not truly Euclidean: it computes mean displacements using the manifold’s geometry (via log/exp maps), preserving spherical structure. 
Thus MDV functions naturally resembles RISE more closely than Procrustes. 
We evaluated them alongside RISE on three datasets: BLiMP, SICK, and our multilingual synthetic dataset.

The strongest performance appears in monolingual English evaluation (BLiMP), while performance drops substantially for Procrustes on semantic relatedness (SICK) shown in Table~\ref{tab:appendixC_summary}.
This shift in performance reflects Procrustes' inability to identify a generalizable semantic–syntactic relationship as expected by method. Procrustes fits a single global rotation which is too rigid for the cross-lingual and cross model analysis
In contrast, RISE maintains stable cross-lingual and cross-model performance (e.g., App.~\ref{app:heatmaps}. Figures 5–7), indicating that geometric operations on the manifold better capture discourse-level semantic structure than Euclidean differences.

The MDV vs.~RISE vs.~Procrustes results reinforce our earlier claim that methods operating on the curved manifold (where sentence embeddings inherently reside) perform better than Euclidean/linear methods. 
Most steering and probing techniques operate in linear space, and we conjecture that this geometric mismatch helps explain why linear methods struggle to generalize. 
In short, Procrustes fits a single global rotation which is too rigid for the cross-lingual and cross model analysis. Geometric transformations, like RISE and MDV, are better suited for semantic-syntactic analysis and cross-lingual stability. 

\subsection{Cross-Language Transfer Heatmaps}
\label{app:cross_lang_heatmaps}

Figures~\ref{fig:cond_cross_appendix}--\ref{fig:pol_cross_appendix} show
cross-language cosine similarity for the three semantic transformations
(Conditionality, Negation, Politeness) under Mean Difference Vectors (MDV),  Procrustes alignment, and RISE.

\begin{figure*}[htbp]
    \centering
    \begin{subfigure}{0.32\textwidth}
        \centering
        \includegraphics[width=\linewidth]{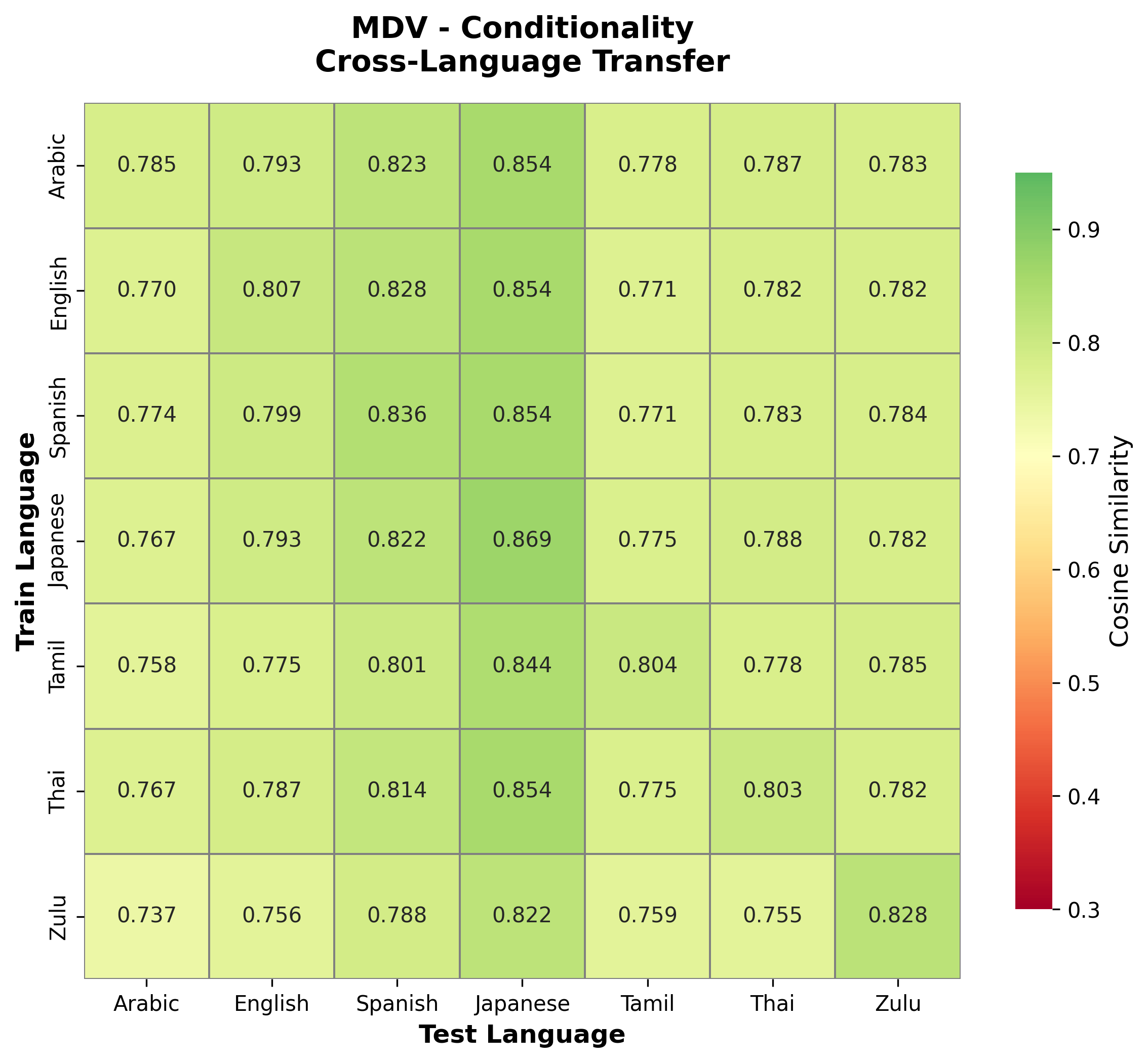}
        \caption{MDV}
    \end{subfigure}
    \hfill
    \begin{subfigure}{0.32\textwidth}
        \centering
        \includegraphics[width=\linewidth]{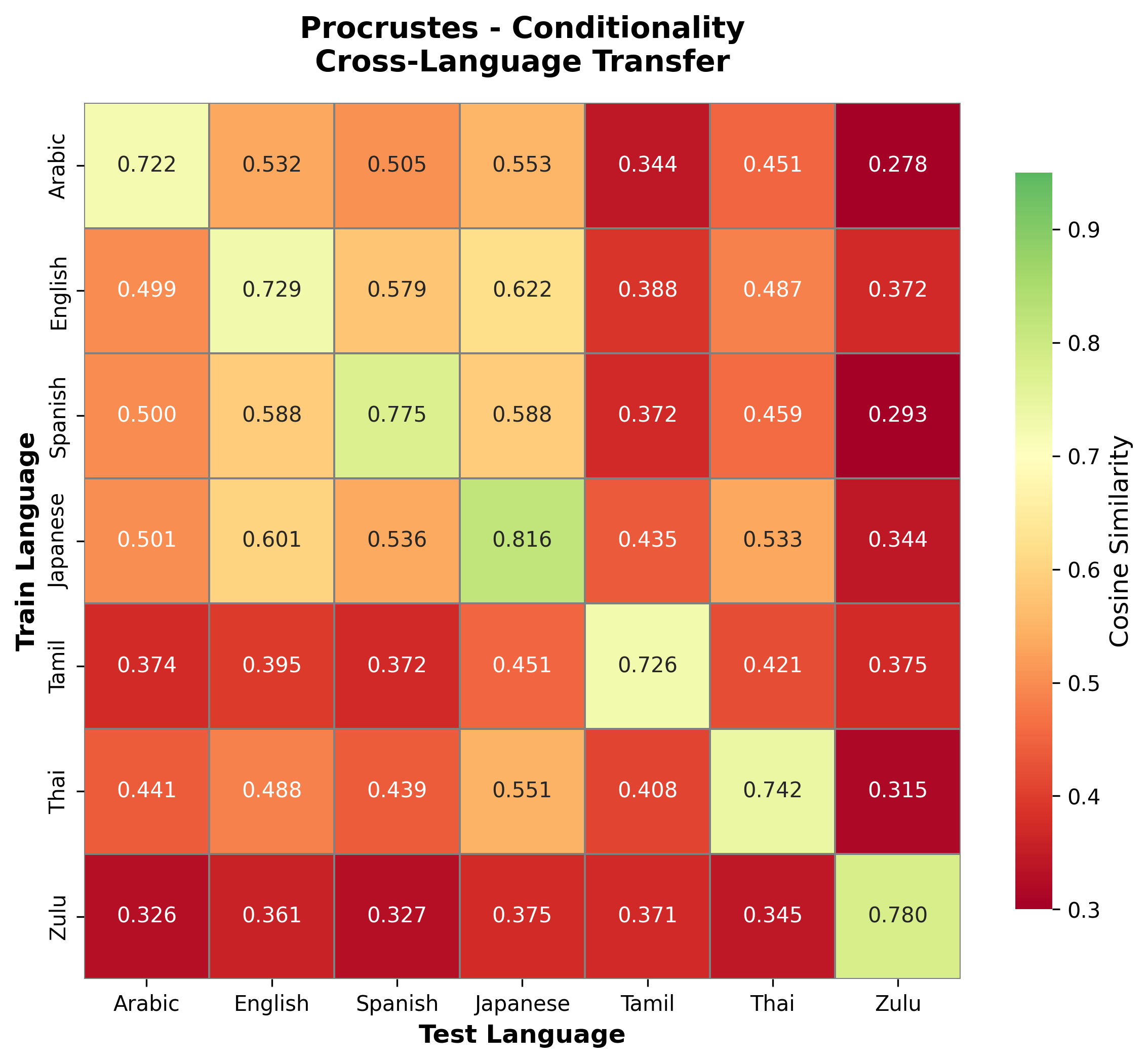}
        \caption{Procrustes}
    \end{subfigure}
    \hfill
    \begin{subfigure}{0.32\textwidth}
        \centering
        \includegraphics[width=\linewidth]{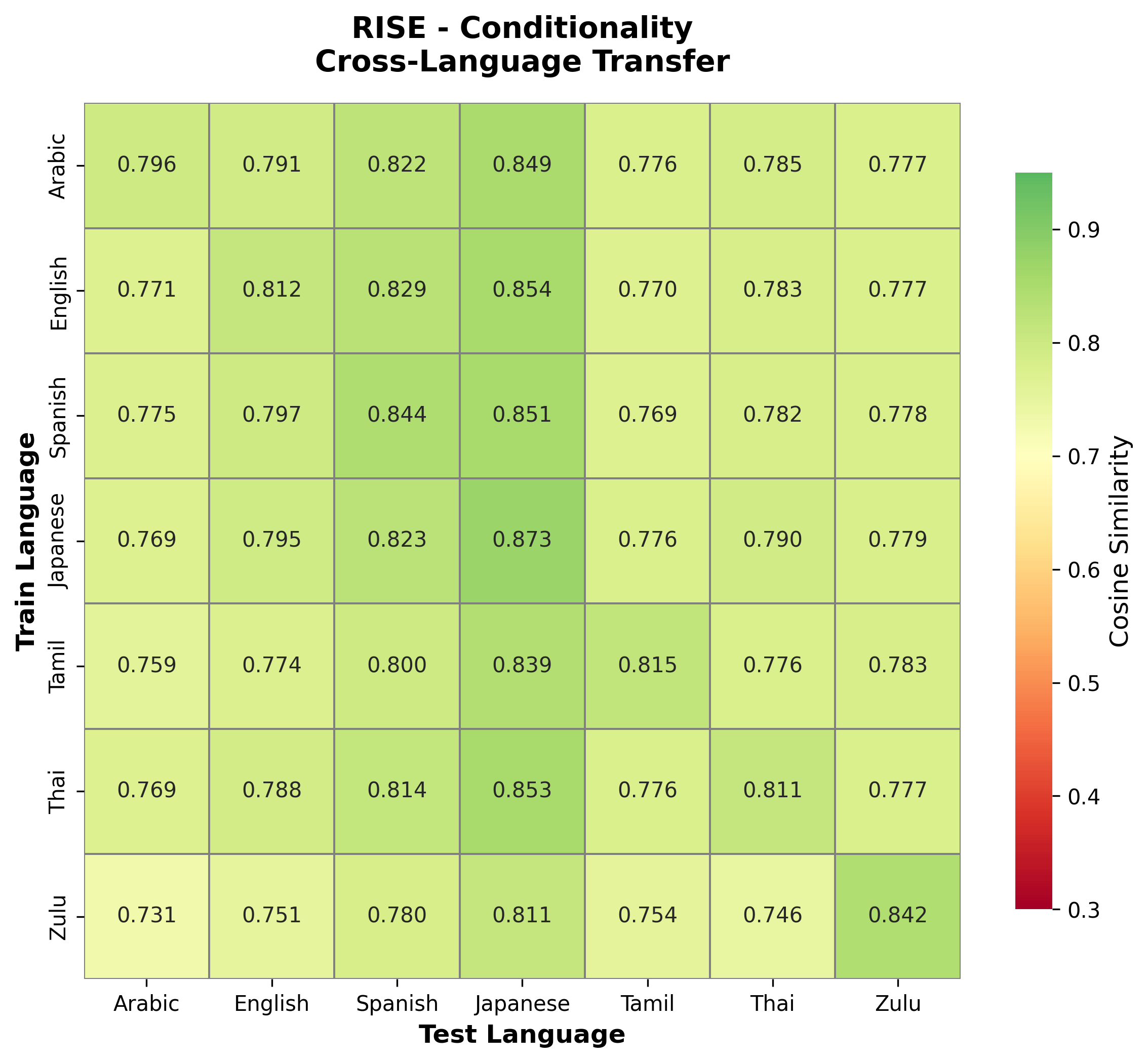}
        \caption{RISE}
    \end{subfigure}
    \caption{Cross-language transfer for \textbf{Conditionality} across seven languages.}
    \label{fig:cond_cross_appendix}
\end{figure*}

\begin{figure*}[htbp]
    \centering
    \begin{subfigure}{0.32\textwidth}
        \centering
        \includegraphics[width=\linewidth]{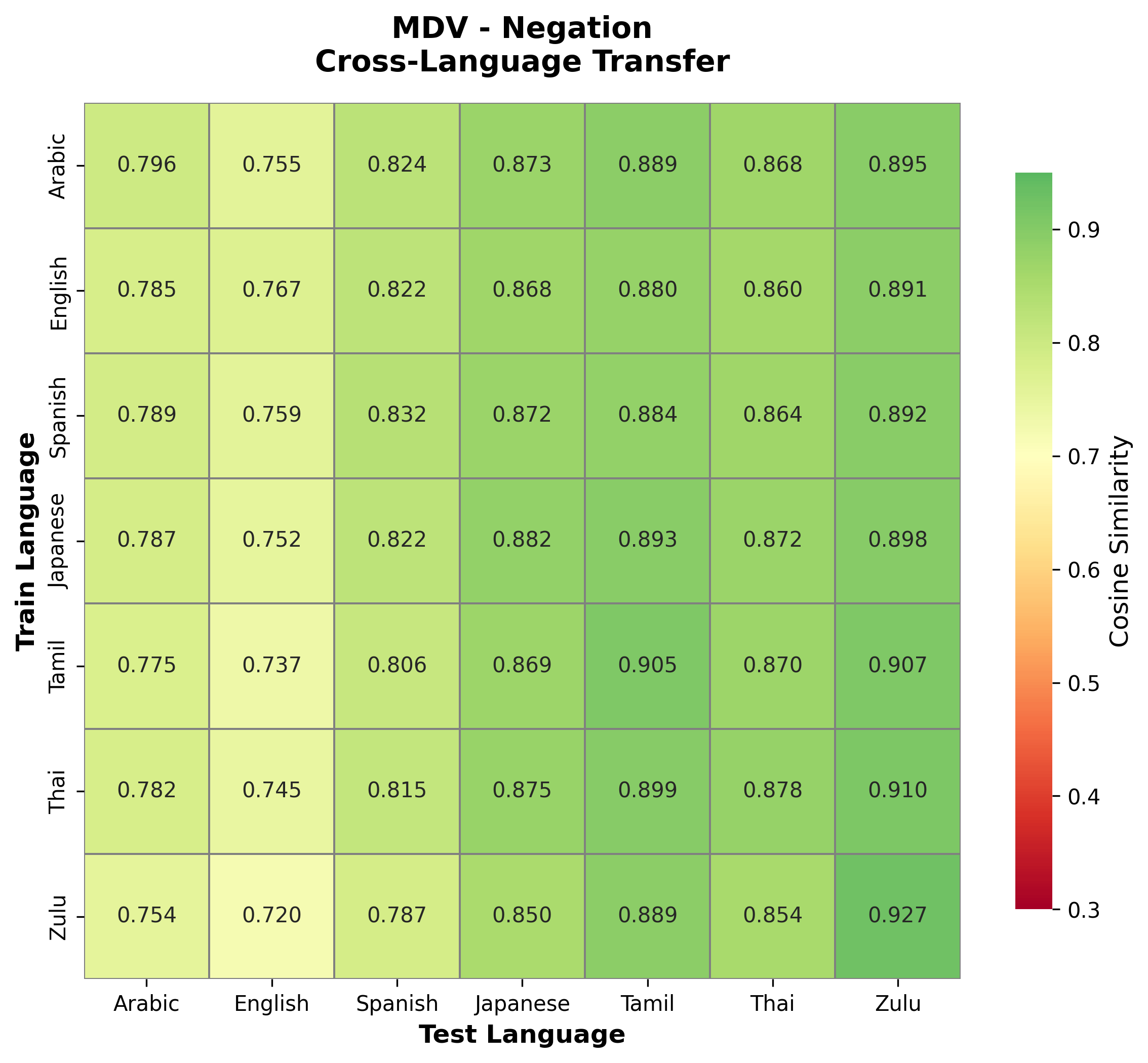}
        \caption{MDV}
    \end{subfigure}
    \hfill
    \begin{subfigure}{0.32\textwidth}
        \centering
        \includegraphics[width=\linewidth]{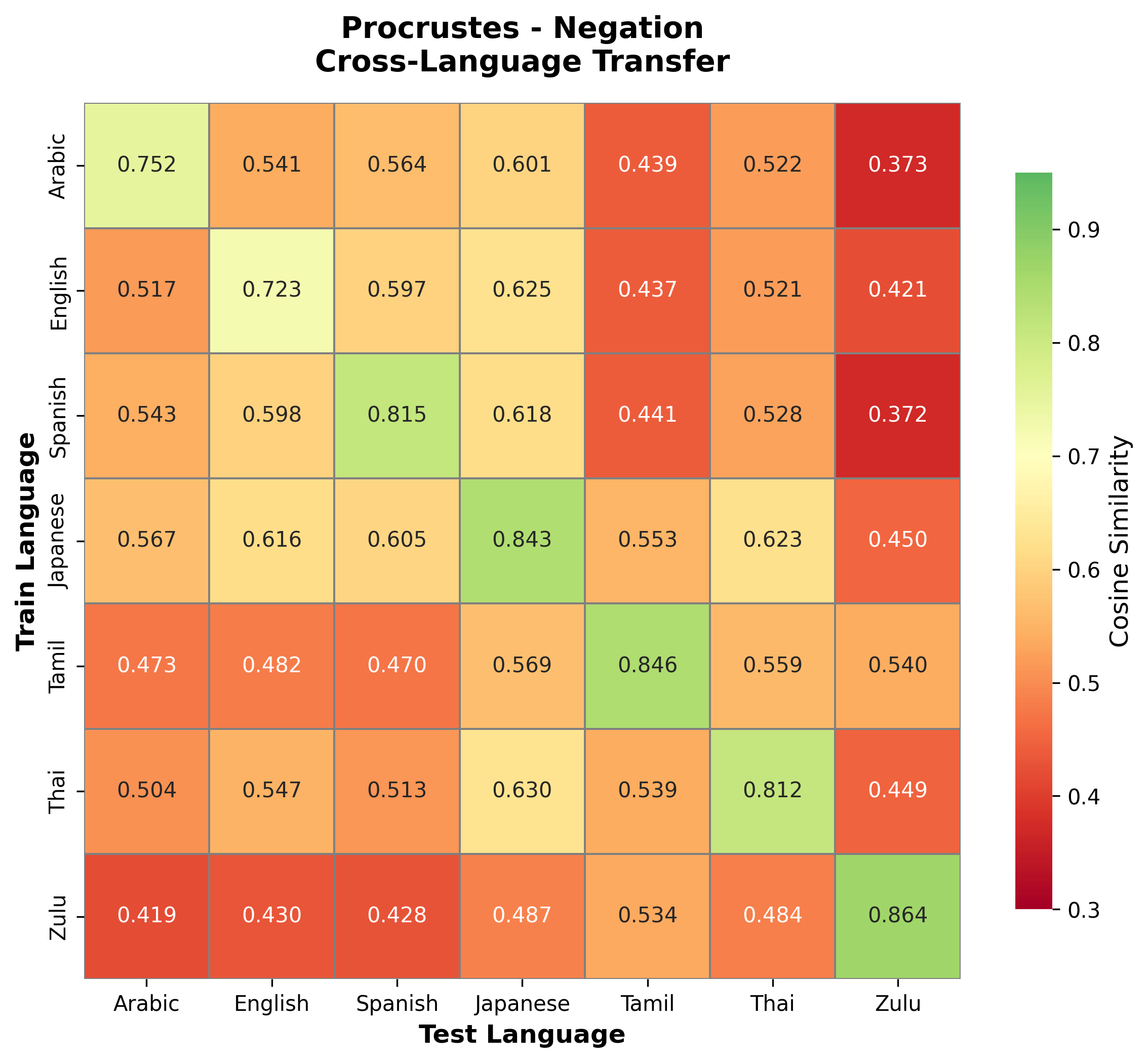}
        \caption{Procrustes}
    \end{subfigure}
    \hfill
    \begin{subfigure}{0.32\textwidth}
        \centering
        \includegraphics[width=\linewidth]{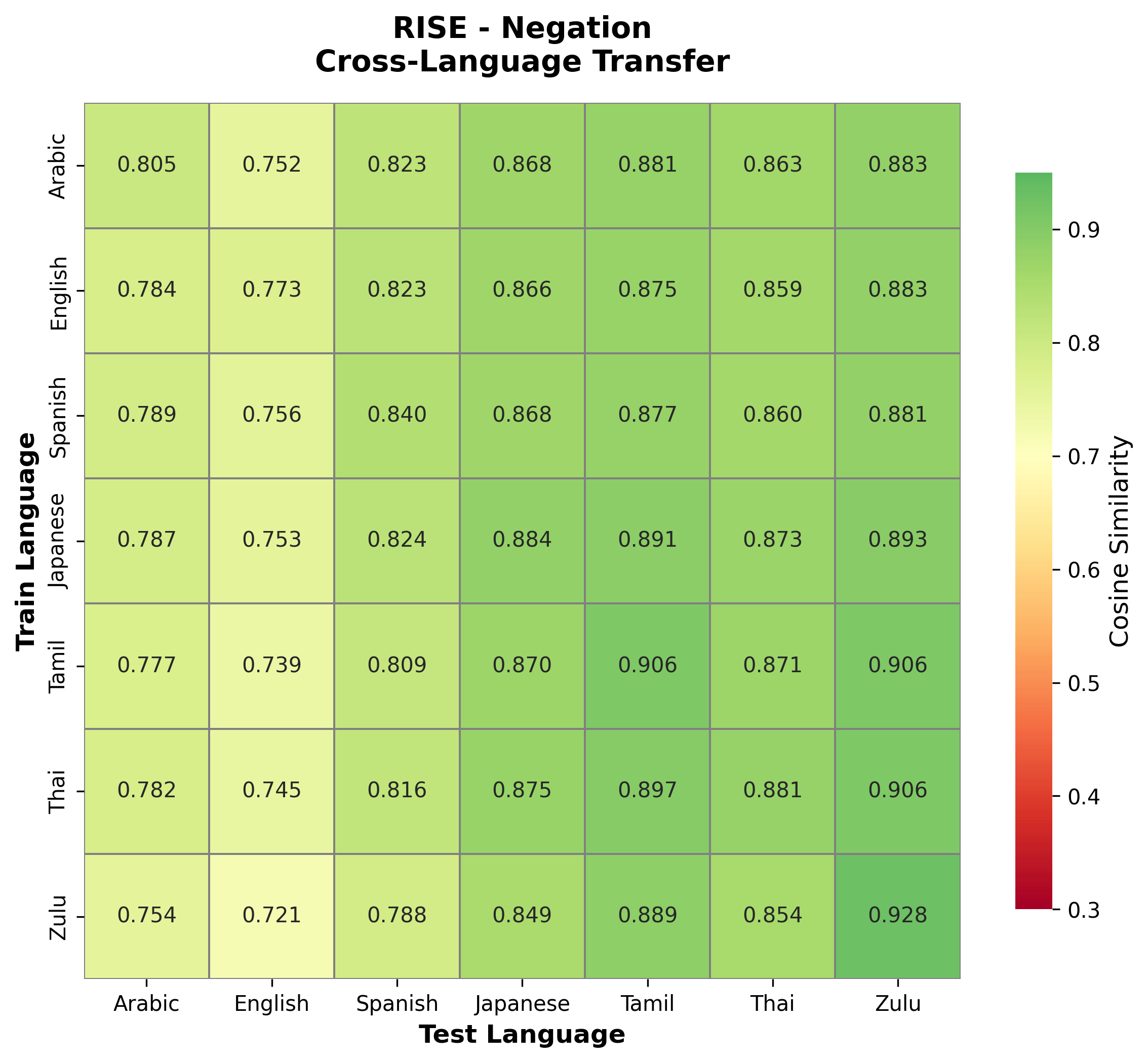}
        \caption{RISE}
    \end{subfigure}
    \caption{Cross-language transfer for \textbf{Negation} across seven languages.}
    \label{fig:neg_cross_appendix}
\end{figure*}

\begin{figure*}[htbp]
    \centering
    \begin{subfigure}{0.32\textwidth}
        \centering
        \includegraphics[width=\linewidth]{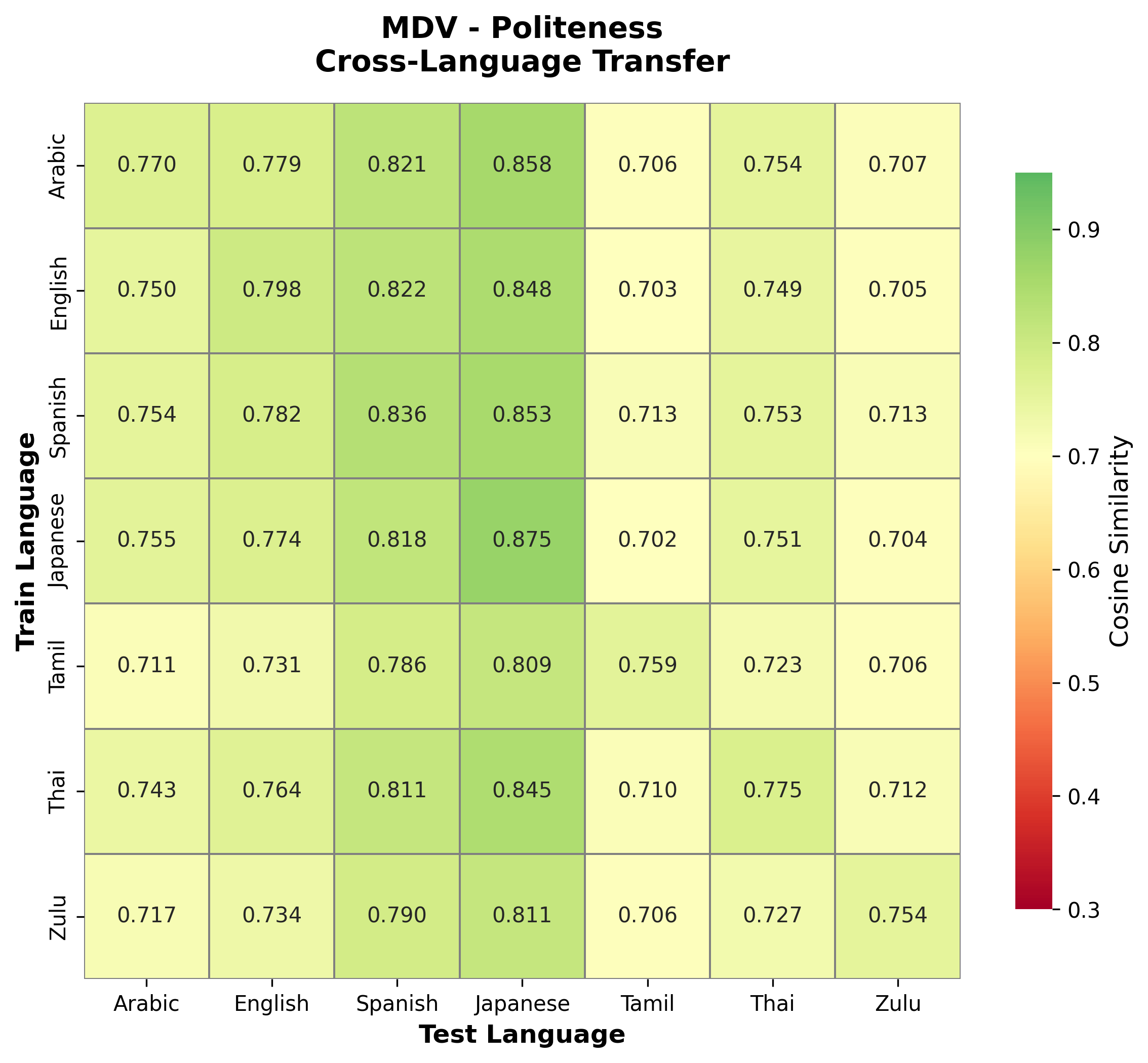}
        \caption{MDV}
    \end{subfigure}
    \hfill
    \begin{subfigure}{0.32\textwidth}
        \centering
        \includegraphics[width=\linewidth]{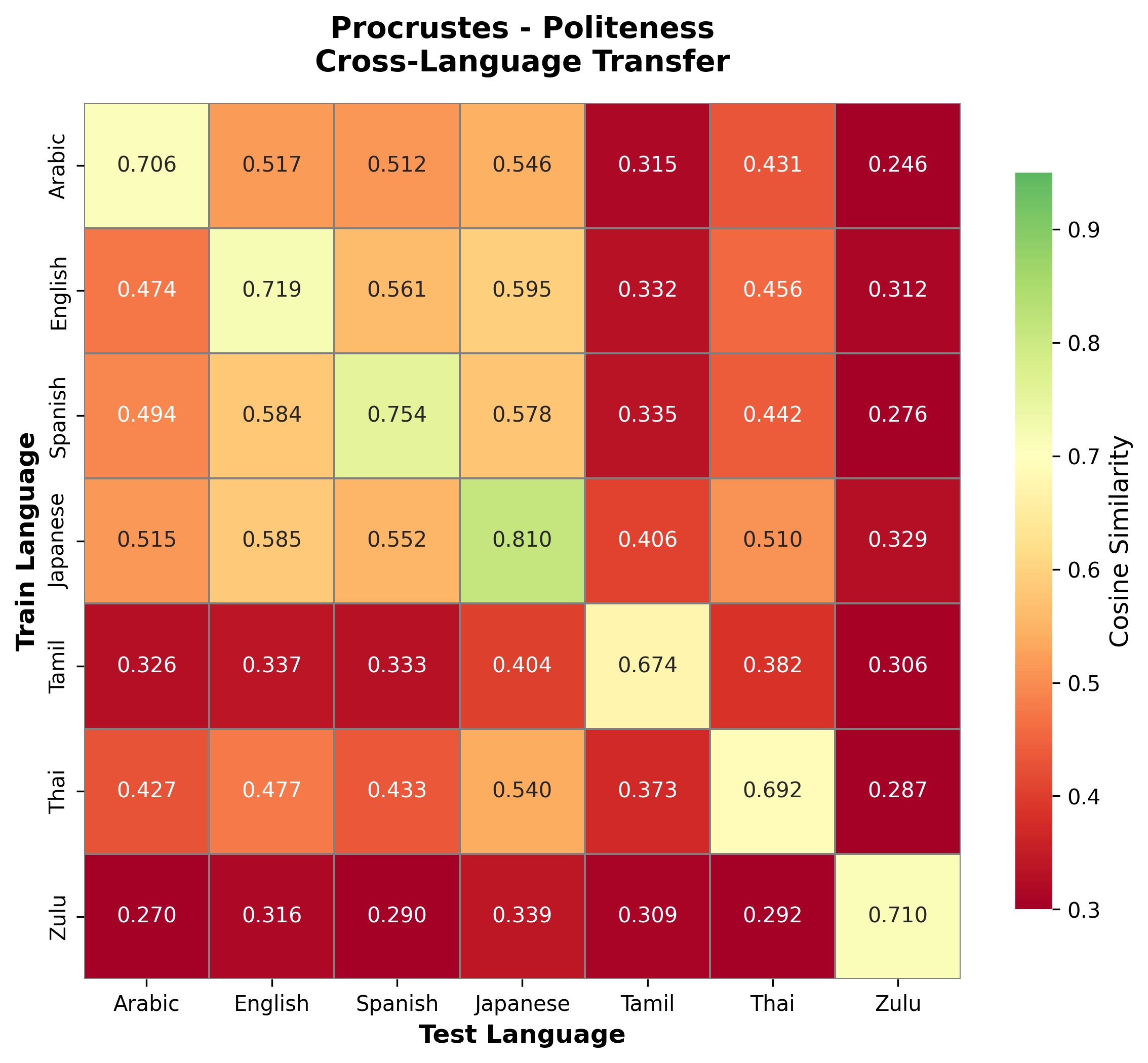}
        \caption{Procrustes}
    \end{subfigure}
    \hfill
    \begin{subfigure}{0.32\textwidth}
        \centering
        \includegraphics[width=\linewidth]{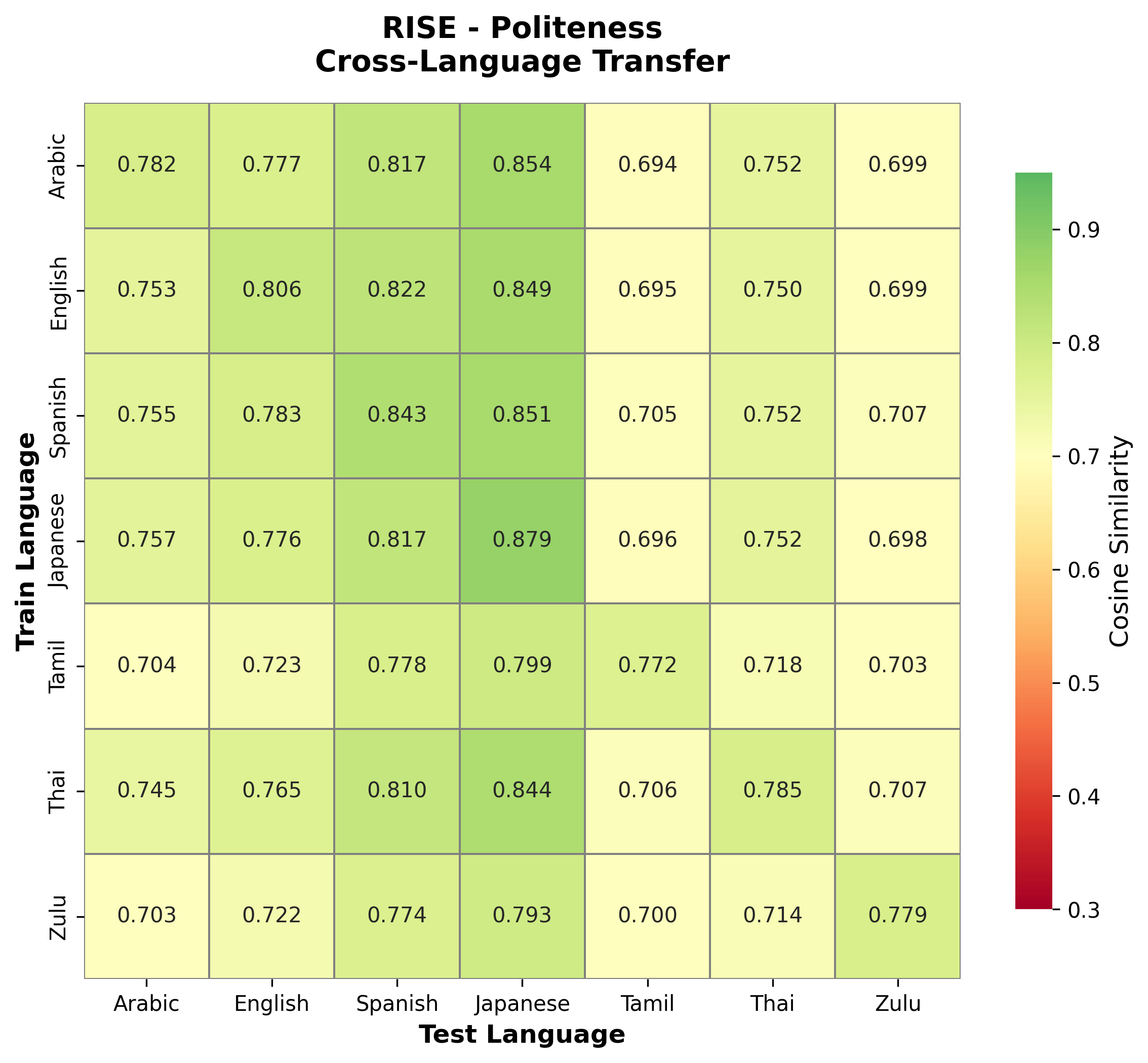}
        \caption{RISE}
    \end{subfigure}
    \caption{Cross-language transfer for \textbf{Politeness} across seven languages.}
    \label{fig:pol_cross_appendix}
\end{figure*}

\subsection{Natural-Language Validation: BLiMP and SICK}
\label{app:BLiMP_sick}

Figure~\ref{fig:BLiMP_sick_appendix} reports mean cosine similarity on BLiMP
(syntactic) and SICK (semantic) for the three methods.

\begin{figure}[htbp]
    \centering
    \includegraphics[width=\linewidth]{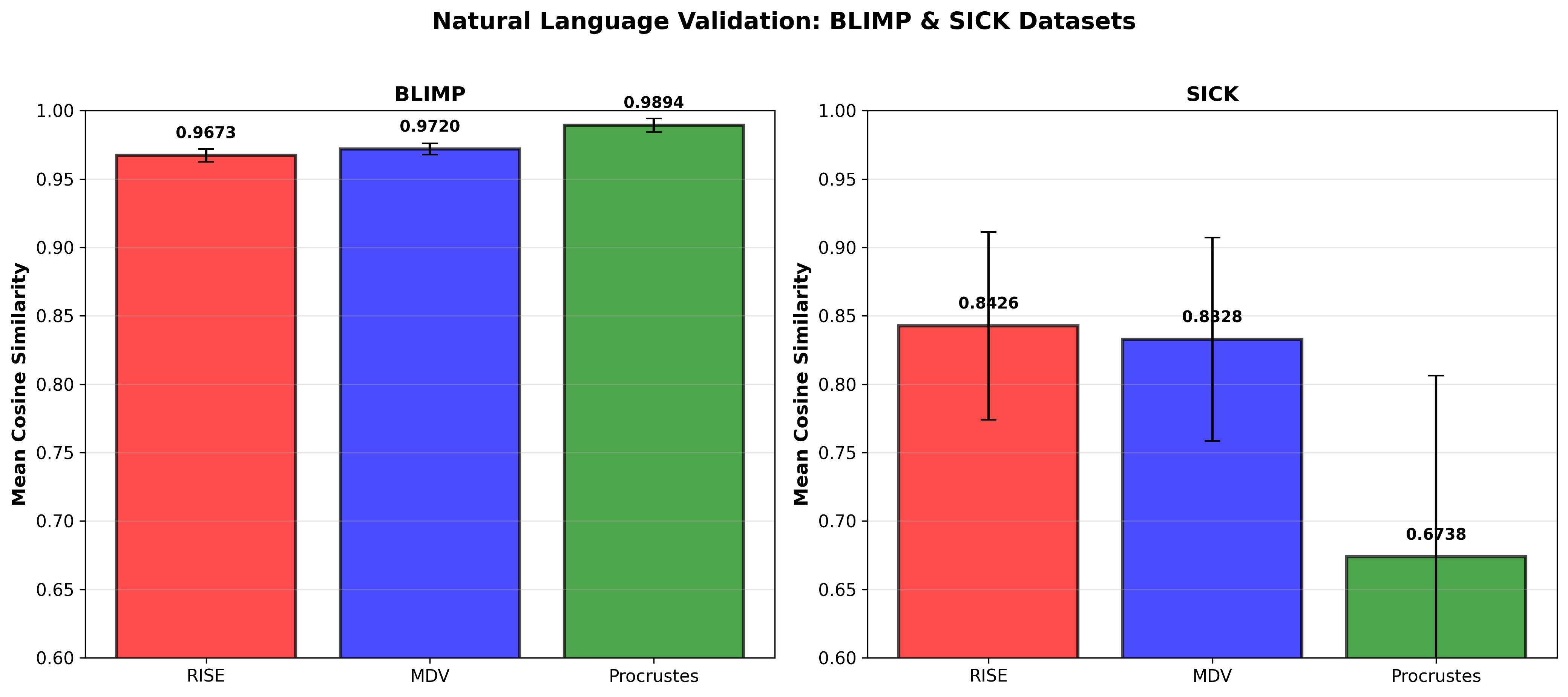}
    \caption{Natural language validation on BLiMP (syntactic acceptability) and
    SICK (semantic relatedness) for RISE, MDV, and Procrustes. Error bars denote
    standard deviation across examples.}
    \label{fig:BLiMP_sick_appendix}
\end{figure}

\newpage
\section{Prompt Templates}
\label{app:prompts}

We provide the exact prompt templates used to generate neutral sentences and their semantic variants. Each template is shown in monospace using the \texttt{lstlisting} environment for clarity and reproducibility.

\subsection{Neutral Sentence Generation}
\begin{lstlisting}[language={},basicstyle=\ttfamily\small,breaklines=true]
You are a linguistics assistant. Generate ONE terse, blunt English
sentence that is politeness-neutral: it must be neither explicitly
polite nor impolite. Keep it concise (8 to 12 words), direct, and
free of polite markers such as "please", honorifics, hedging,
or apologies, yet ensure it is not rude. If the situation contains
a placeholder (e.g., "a favor", "a cultural practice"), replace
it with a concrete, plausible example.

Context category: {category}
Detailed situation: {example}

Respond with ONLY the single sentence (no explanations, no quotation marks).
\end{lstlisting}

\subsection{Politeness Rephrasing}
\begin{lstlisting}[language={},basicstyle=\ttfamily\small,breaklines=true]
You are an expert translator and pragmatics specialist. Rewrite the
following sentence in {language_name} to make it more POLITE while
preserving its original meaning. Incorporate the given politeness
features.

Sentence: "{sentence}"

Politeness features (JSON): {features_json}

Respond ONLY with a JSON object in the exact format:
{"polite": "<rewritten sentence>"}
Do NOT add any other keys, explanations, or markdown.
\end{lstlisting}

\subsection{Negation}
\begin{lstlisting}[language={},basicstyle=\ttfamily\small,breaklines=true]
You are an expert translator and semantics specialist. Rewrite the
following sentence in {language_name} so that it expresses the
NEGATION of its original meaning while remaining natural and fluent.
Incorporate the given negation features.

Sentence: "{sentence}"

Negation features (JSON): {features_json}

Respond ONLY with a JSON object in the exact format:
{"negation": "<rewritten sentence>"}
Do NOT add any other keys, explanations, or markdown.
\end{lstlisting}

\subsection{Conditionality}
\begin{lstlisting}[language={},basicstyle=\ttfamily\small,breaklines=true]
You are an expert translator and syntax/pragmatics expert. Rewrite
the following sentence in {language_name} so that the statement
becomes CONDITIONAL (i.e., it only holds under a certain condition)
while preserving overall meaning and sounding natural. Incorporate
the provided conditionality features.

Sentence: "{sentence}"

Conditionality features (JSON): {features_json}

Respond ONLY with a JSON object in the exact format:
{"conditionality": "<rewritten sentence>"}
Do NOT add any other keys, explanations, or markdown.
\end{lstlisting}


\section{Data Generation Methodology}
\label{app:data-generation}

\subsection{Diversity Controls} To guard against artifacts that might arise from narrow lexical or topical coverage, we apply several sampling diversity control. (1) Each neutral sentence prompt draws its situation description from a randomly chosen context category and exemplar, yielding a wide topical spread before any transformation is applied. (2) Within every language we shuffle sentence–feature assignments so that no specific lexical field correlates with a particular transformation subtype. (3) For each transformation we uniformly sample property values (e.g., negation particle, politeness strategy) per language and sentence, guaranteeing that every combination of language and subtype appears the same number of times. (4) After generation we remove near-duplicates and enforce a 5–25 token length window, which empirically yields a near-uniform length distribution. Together these steps ensure that our corpus varies in topic, syntax, and lexical choice while remaining balanced across languages and transformation subtypes.
These controls ensure that observed geometric patterns reflect semantic properties rather than artifacts of lexical choice or sentence structure.
\begin{enumerate}
    \item \textbf{Topical Diversity:} Neutral sentences were drawn from varied context categories (social interactions, factual statements, requests, etc.)
    \item \textbf{Feature Balance:} Transformation features (e.g., negation particles, politeness strategies) were uniformly sampled to prevent correlation with specific lexical fields.
    \item \textbf{Length Normalization:} Sentences were filtered to 5-25 tokens to ensure comparable embedding properties.
    \item \textbf{Deduplication:} Near-duplicate outputs were removed to prevent repeated data.
\end{enumerate}

\subsection{Feature-based Transformation Methodology}

We generated sentence pairs systematically by first sampling neutral sentences in seven typologically diverse languages (English, Spanish, Tamil, Thai, Arabic, Japanese, and Zulu), and subsequently transforming each sentence using feature-controlled prompts. Each transformation was guided by uniformly sampling linguistic features from a predefined typological metadata set (illustrated below).

The full inventories of typological properties for politeness, negation, and conditionality are provided in Tables~\ref{tab:politeness_features}--\ref{tab:conditionality_features}.

\begin{table*}[t]
\centering
\small
\begin{tabular}{lp{3cm}p{4cm}}
\hline
\textbf{Language} & \textbf{Strategy Type} & \textbf{Grammatical/Lexical Devices} \\
\hline
English & Negative politeness & Modal conditional, hedging, idiomatic/proverbial, taboo avoidance \\
Spanish & Positive politeness & Modal conditional, morphological politeness, hedging, idiomatic/proverbial \\
Tamil & Relational/Kinship politeness & Morphological politeness \\
Thai & Positive politeness; Relational/Kinship & Morphological politeness, modal conditional \\
Arabic & Positive politeness; Relational/Kinship & Modal conditional, morphological politeness, idiomatic/proverbial \\
Japanese & Relational/Kinship politeness & Morphological politeness, modal conditional, hedging \\
Zulu & Relational/Kinship politeness & Morphological politeness \\
\hline
\end{tabular}
\caption{Typological features sampled uniformly for politeness transformations.}
\label{tab:politeness_features}
\end{table*}

\begin{table*}[t]
\centering
\small
\begin{tabular}{lp{2.2cm}p{4cm}}
\hline
\textbf{Language} & \textbf{Marker Position} & \textbf{Morphological Realization} \\
\hline
English & Clause-medial & Negative particle; negative aux/modal; negative affix \\
Spanish & Clause-medial; concord & Negative particle \\
Tamil & Clause-final & Negative particle; verb-internal negation \\
Thai & Clause-medial & Negative particle \\
Arabic & Clause-initial / medial & Negative particle; negative affix \\
Japanese & Clause-final & Verb-internal negation \\
Zulu & Clause-medial & Negative particle \\
\hline
\end{tabular}
\caption{Typological features sampled uniformly for negation transformations.}
\label{tab:negation_features}
\end{table*}

\begin{table*}[t]
\centering
\small
\begin{tabular}{lp{2.2cm}p{4cm}}
\hline
\textbf{Language} & \textbf{Clause Structure} & \textbf{Morphological Marking} \\
\hline
English & Initial; final; embedded & Explicit marker; conditional tense/aspect \\
Spanish & Initial; final; embedded & Conditional mood; explicit marker \\
Tamil & Final & Explicit marker; conditional mood \\
Thai & Initial & Explicit marker \\
Arabic & Initial; final & Conditional mood; explicit marker \\
Japanese & Final; embedded & Conditional mood; explicit marker \\
Zulu & Initial & Conditional mood; explicit marker \\
\hline
\end{tabular}
\caption{Typological features sampled uniformly for conditionality transformations.}
\label{tab:conditionality_features}
\end{table*}

\subsubsection{Transformation Procedure}

For each neutral sentence, we uniformly sampled exactly one set of feature values from the typological metadata and prompted the language model (GPT-4.5) to generate the transformed variant adhering to these specifications. By uniformly sampling across multiple typological dimensions—strategy types, morphological realizations, and pragmatic contexts—we ensured comprehensive coverage of each language's linguistic variability. This methodology supports cross-linguistic embedding analysis and ensures that observed embedding-space transformations reflect typological distinctions accurately.

\subsection{Feature-Controlled Prompting}

To generate each transformation in a systematic and reproducible manner, we employ a feature-controlled prompting strategy with a large language model (LLM). Each prompt is carefully templated to specify the source language, the desired transformation type, and a set of fine-grained feature tags that guide the model's output. For example, a prompt might indicate the language code (``[TA]'' for Tamil), the transformation (``Politeness Rephrase''), and a particular strategy or keyword (such as ``add honorific'') relevant to that transformation. By explicitly encoding these features, we ensure that the LLM produces the intended variation—whether a more polite rephrasing, a negated statement, or a conditional construction—in a consistent and transparent way.

To further guarantee balanced coverage, we maintain a metadata table that enumerates all possible sub-types or strategies for each transformation. This enables us to stratify the sampling of transformation features across languages and sentences, ensuring that every variant type is equally represented. For instance, multiple politeness strategies (e.g., adding honorifics, using indirect language) or different negation words (``no'' vs. ``not'') are distributed uniformly across the dataset. This controlled coverage is critical for fair comparisons: it prevents any language from being overrepresented by a particular style of rephrasing or negation, and minimizes inadvertent correlations between language and transformation realization. Our stratified sampling approach follows established principles of controlled experimental design, providing a robust foundation for cross-lingual embedding analysis.

All transformed sentences are generated using a single, consistent LLM (GPT-4.5) with a temperature of 1.0 and a maximum token limit of 128 per prompt. The relatively high temperature encourages diversity in phrasing, while the one-shot generation policy (taking the first model output without retries or manual curation) avoids selection bias. With carefully constructed prompts, the model reliably produces valid transformations on the first attempt, and all outputs remain in the target language specified by the prompt. This procedure ensures that our dataset is both systematically varied and reproducible, supporting rigorous downstream analysis.

\subsection{Quality Control and Deduplication}

To ensure the integrity and uniqueness of our dataset, we implemented a rigorous two-level deduplication process. At the first level, we removed any transformed sentence that was exactly identical to another within the same category and language. This step addresses the possibility that the LLM might produce identical outputs for different inputs, especially for short or formulaic sentences. At the second level, we ensured that each (neutral, variant) pair was unique across the entire dataset. In rare cases where two different source sentences yielded the same transformed output, we treated this as a collision and regenerated a new variant using a slightly altered prompt. Through this process, every neutral sentence in our dataset is paired one-to-one with three distinct transformed sentences (one per transformation type), with no overlaps. The result is a clean set of sentence pairs, each exhibiting a unique, transformation-driven difference.

Beyond deduplication, we applied a suite of diversity controls to guard against artifacts arising from narrow lexical or topical coverage. Each neutral sentence prompt was drawn from a wide range of context categories and exemplars, ensuring topical breadth before any transformation was applied. Within each language, we shuffled sentence–feature assignments so that no specific lexical field correlated with a particular transformation subtype. For each transformation, we uniformly sampled property values (such as negation particles or politeness strategies) per language and sentence, guaranteeing that every combination of language and subtype appeared the same number of times. After generation, we removed near-duplicates and enforced a 5–25 token length window, which empirically yielded a near-uniform length distribution. Together, these steps ensure that our corpus varies in topic, syntax, and lexical choice while remaining balanced across languages and transformation subtypes, providing a robust foundation for subsequent embedding analysis.

\subsection{Embedding Generation}

With our dataset of neutral and transformed sentences in hand, we next obtain high-dimensional vector representations using a state-of-the-art multilingual sentence encoder. Specifically, we employ OpenAI's \texttt{text-embedding-3-large} model, which produces 3072-dimensional embeddings aligned semantically across more than 90 languages.\footnote{\url{https://platform.openai.com/docs/guides/embeddings}} All embeddings are generated in a frozen (non-fine-tuned) setting, with a single API call per sentence. According to the model card, each sentence embedding is computed by mean-pooling the token-level hidden states, followed by layer normalization. This means that every token—including short functional items like negation particles—contributes proportionally to the final vector.

Our approach assumes that all sentence embeddings reside in a shared semantic space where linear structure is meaningful. We adopt the perspective that this space forms a latent manifold encoding universal semantic features, as hypothesized by \citet{Jha2025}. In this framework, certain directions in the embedding space correspond to specific attributes, such as politeness or negation. If sentence transformations truly correspond to adding or subtracting a semantic attribute, we expect the difference vector (variant minus source) to be relatively consistent across examples. This aligns with the ``universal geometry for embeddings'' framework, in which multilingual embeddings from different models or languages can be brought to a common representation where semantic differences are captured by geometric translations. While our work stays within a single encoder's space, we leverage a similar idea: analyzing whether the transformation ``rotors'' (difference vectors) cluster for similar transformations across languages. This methodology sets the stage for validating whether these quasi-linear transformations indeed behave like translations in a Riemannian semantic space \citep{Jha2025}, which we explore in the next section via rotor-based analysis of the embedding differences.

It is important to note that applying a single global rotation or principal component analysis (PCA) can distort other dimensions and is not adaptive to individual vectors. Because the base embedding is already a mean across tokens, edits that insert or replace a handful of tokens translate to small but coherent rotations of the global vector—precisely the kind of local, content-independent shift that our rotor method is designed to capture.

\subsection{Final Dataset Statistics}

The resulting corpus comprises 1,000 neutral sentences in each of the seven languages, totaling 7,000 examples. For English neutral sentences, the mean token length is 9.1 tokens (with a median of 9.0 tokens), with token counts ranging from three to 12 tokens and an average character length of 54.4 characters. This distribution confirms that our generation process produced concise, natural sentences suitable for semantic transformation analysis across languages and transformation types.

To further validate the diversity and balance of our dataset, we analyzed the distribution of sentence lengths per language, which reveals broadly similar profiles with a peak around 10–15 tokens. Additionally, we examined the distribution of word frequencies, confirming a typical long-tail distribution in each language. These statistics affirm that our corpus is both balanced and rich in content, providing a solid empirical foundation for the cross-lingual transformation analysis in the subsequent sections.

\section{LLM Usage Disclosure}
\label{app:LLM_usage}

Large language models (LLMs) were used to assist with multiple aspects of this research, including: ideation, writing, programming, and implementation of experimental code, and identification of related work and literature. 
All LLM-generated content, code, and references were subject to human review, testing, and verification to ensure accuracy, functionality, and relevance. 
Any claims, results, experimental implementations, and citations presented in this work have been reviewed by the authors. 
The authors take responsibility for all content, including any errors or inaccuracies that may remain despite our review process.


\section{Downstream Task Analysis}
\label{app:downstream}
\rev{As requested by reviewers, we completed a downstream classification analysis. Due to time constraints, we focused on a single well-defined task: detecting negation in the English subset of the Synthetic Multilingual dataset. We evaluated how well a classifier trained on MDV-transformed and RISE-transformed sentences performed on a held-out test set of 1919 unpaired sentences (961 with negation, 958 without). The test set was generated with the same specifications described in Appendix \ref{app:prompts} \& \ref{app:data-generation}.} 

\rev{Now, both methods perform well on this task. MDV achieves strong recall (92.1\%) and overall accuracy (87.2\%), showing that even a simple mean displacement vector captures meaningful geometric regularities in the transformation. Yet, RISE yields a stronger downstream performance and outperforms MDV across all metrics (93.0\% accuracy, 92.1\% precision, 94.0\% recall, and 93.0\% F1). The positive results of both methods reinforces the broader claim that spherical, non-linear techniques are effective tools for capturing semantic-syntactic transformations in high-dimensional embedding spaces. }

\begin{table}[h]
\centering
\small
\begin{tabular}{lcccc}
\toprule
\textbf{Method} & \textbf{Accuracy} & \textbf{Precision} & \textbf{Recall} & \textbf{F1} \\
\midrule
MDV  & 0.872 & 0.840 & 0.921 & 0.878 \\
RISE & \textbf{0.930} & \textbf{0.921} & \textbf{0.940} & \textbf{0.930} \\
\bottomrule
\end{tabular}
\caption{\rev{Downstream negation classification performance for MDV and RISE transformations.}}
\label{tab:negation_classification}
\end{table}


\section{RISE vs Random Baseline Comparisons}
\label{app:rise_baselines}
This section presents comprehensive comparisons between RISE and random baseline prototypes to validate that RISE learns meaningful semantic directions rather than benefiting from arbitrary vector orientations. The following figures show detailed graphs, heatmaps, and tables comparing RISE performance against random prototypes of equivalent magnitude across all language pairs and phenomena. Each comparison uses 10,000 random trials to ensure statistical robustness.

Figure \ref{fig:rise_vs_random_comparisons} highlights three select language transfer scenarios and Figures \ref{fig:rise_vs_random_text_embedding_3_large}--\ref{fig:rise_vs_random_mbert} demonstrate the baseline validity of RISE by comparing it against random prototypes across multiple language transfer scenarios. The consistent and substantial advantages (ranging from 5.1× to 26.2×) across all models and phenomena provide crucial validation that RISE learns meaningful semantic directions rather than exploiting statistical artifacts. Notably, cross-language transfers often maintain or even exceed monolingual performance relative to random baselines, confirming that RISE captures universal semantic patterns that generalize across language boundaries.
Overall, RISE analyses show that embedding models encode some transformations as universal operators, but others remain highly culture- and resource-dependent. 
Future research should refine evaluation benchmarks to account for phenomenon-specific variability and investigate training regimes that promote balanced universality without sacrificing discriminative capacity.

\begin{figure}[htbp]
\caption{RISE vs Random Baseline Comparisons across select language transfer scenarios. 
\\
\textbf{Top:} English monolingual analysis.
\\
\textbf{Middle:} English prototype → Spanish target cross-language transfer.
\\
\textbf{Bottom:} Japanese prototype → English target cross-language transfer.}
\label{fig:rise_vs_random_comparisons}
\centering
\includegraphics[width=0.85\textwidth]{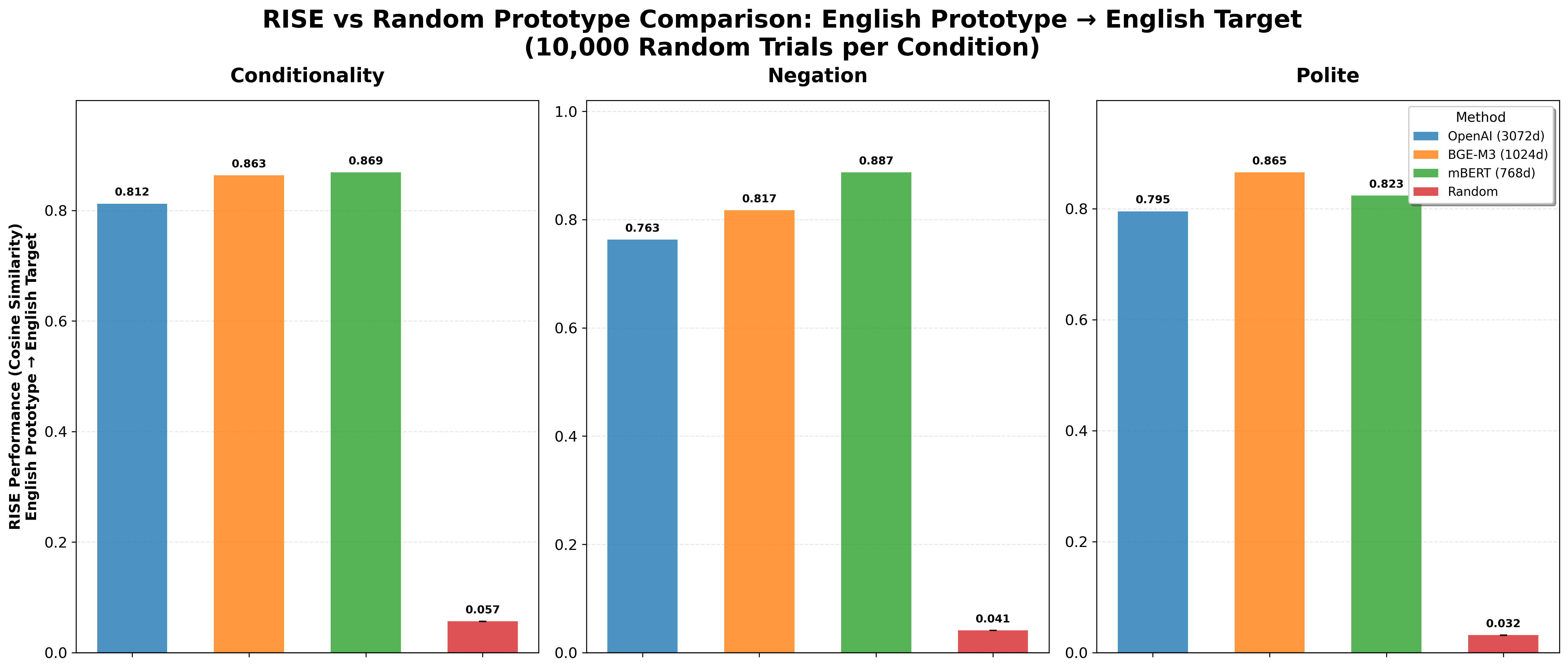}\\[0.5em]
\includegraphics[width=0.85\textwidth]{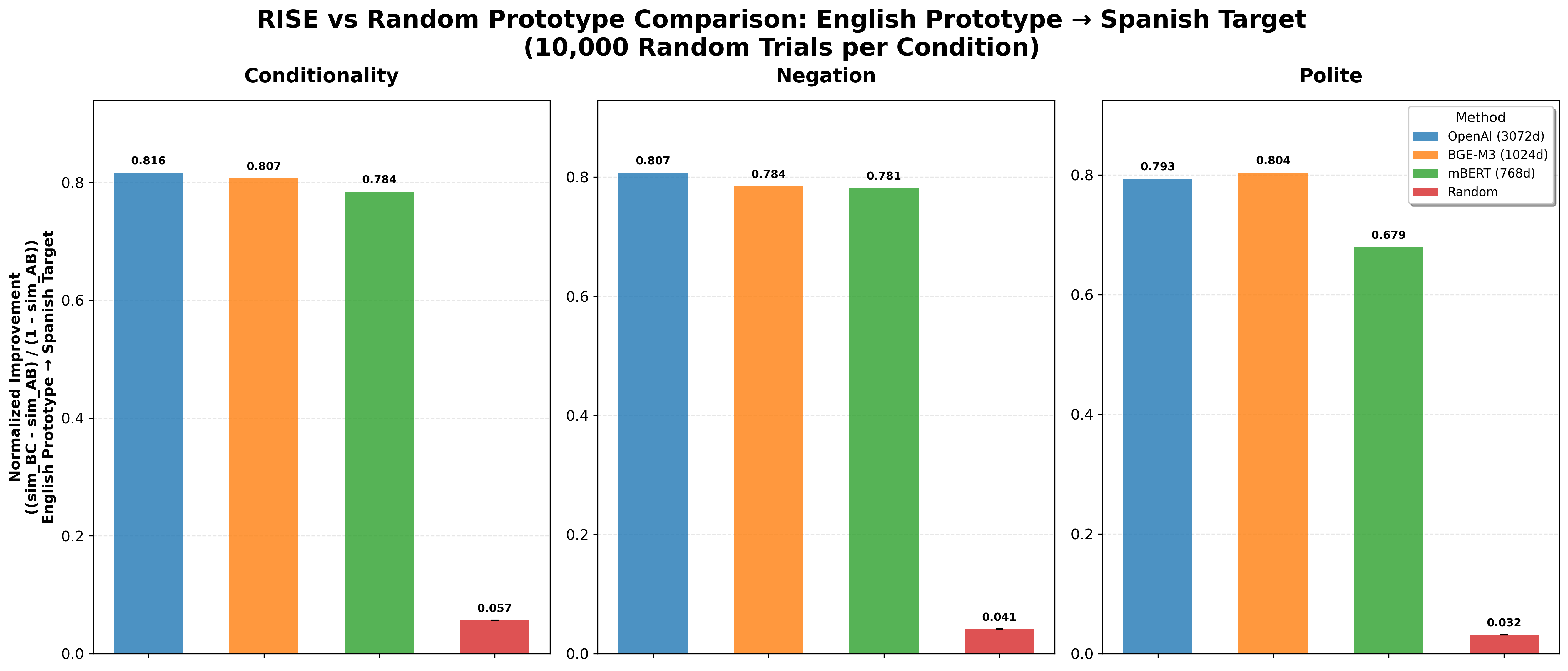}\\[0.5em]
\includegraphics[width=0.85\textwidth]{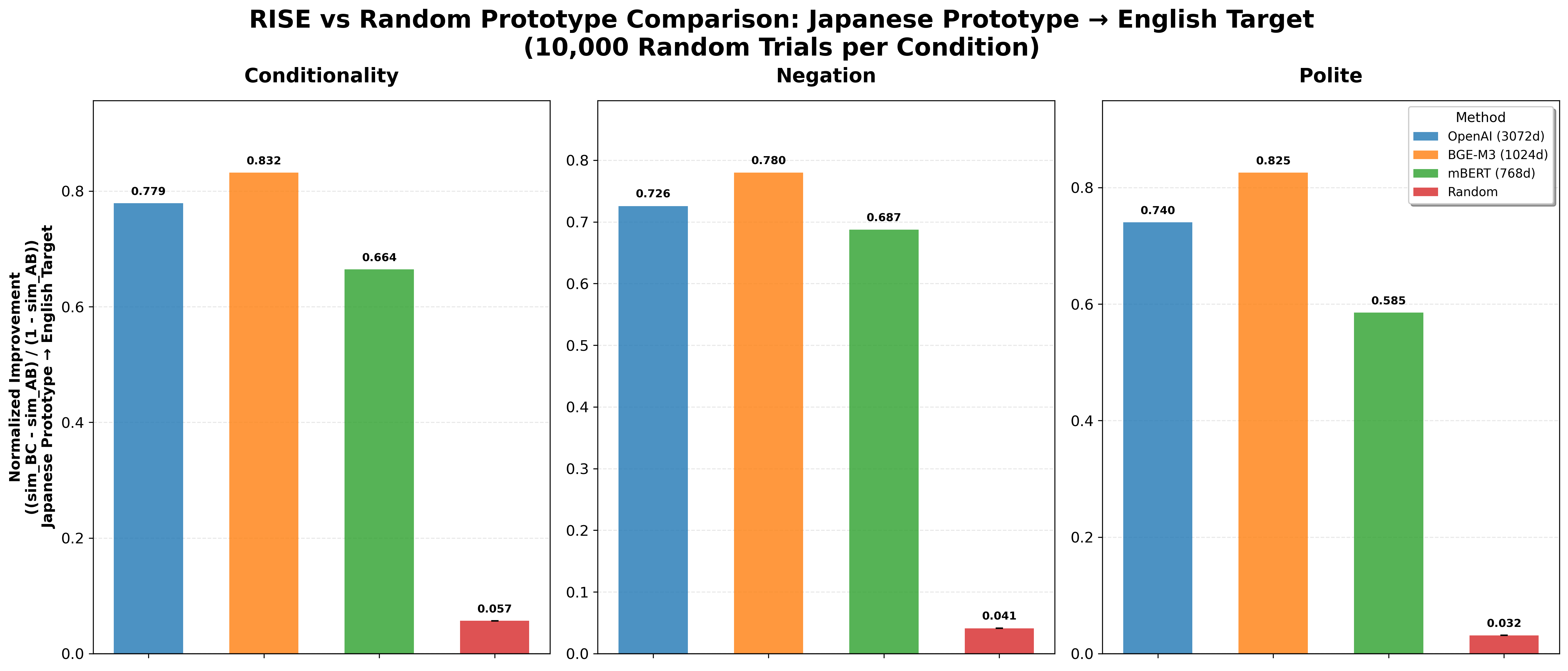}
\end{figure}

\begin{figure}[htbp]
\caption{RISE vs Random Baseline Comparison for text-embedding-3-large. Top row shows RISE performance, bottom row shows random baseline performance (averaged over 10,000 trials). The dramatic performance gap demonstrates that RISE learns meaningful semantic directions rather than benefiting from arbitrary vector orientations.}
\label{fig:rise_vs_random_text_embedding_3_large}
\centering
\includegraphics[width=\textwidth]{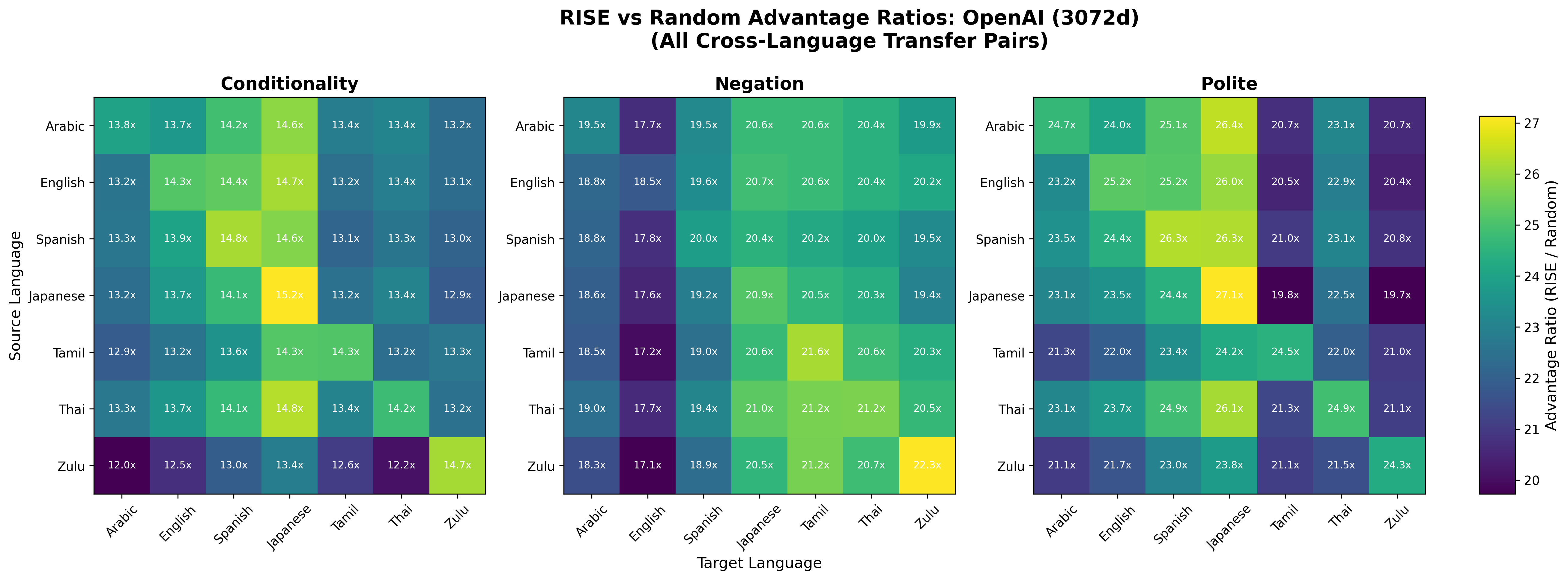}
\end{figure}

\begin{figure}[htbp]
\caption{RISE vs Random Baseline Comparison for bge-m3. Top row shows RISE performance, bottom row shows random baseline performance (averaged over 10,000 trials). Bge-m3 shows remarkably consistent RISE performance across all phenomena and language pairs, with random baselines consistently near zero.}
\label{fig:rise_vs_random_bge_m3}
\centering
\includegraphics[width=\textwidth]{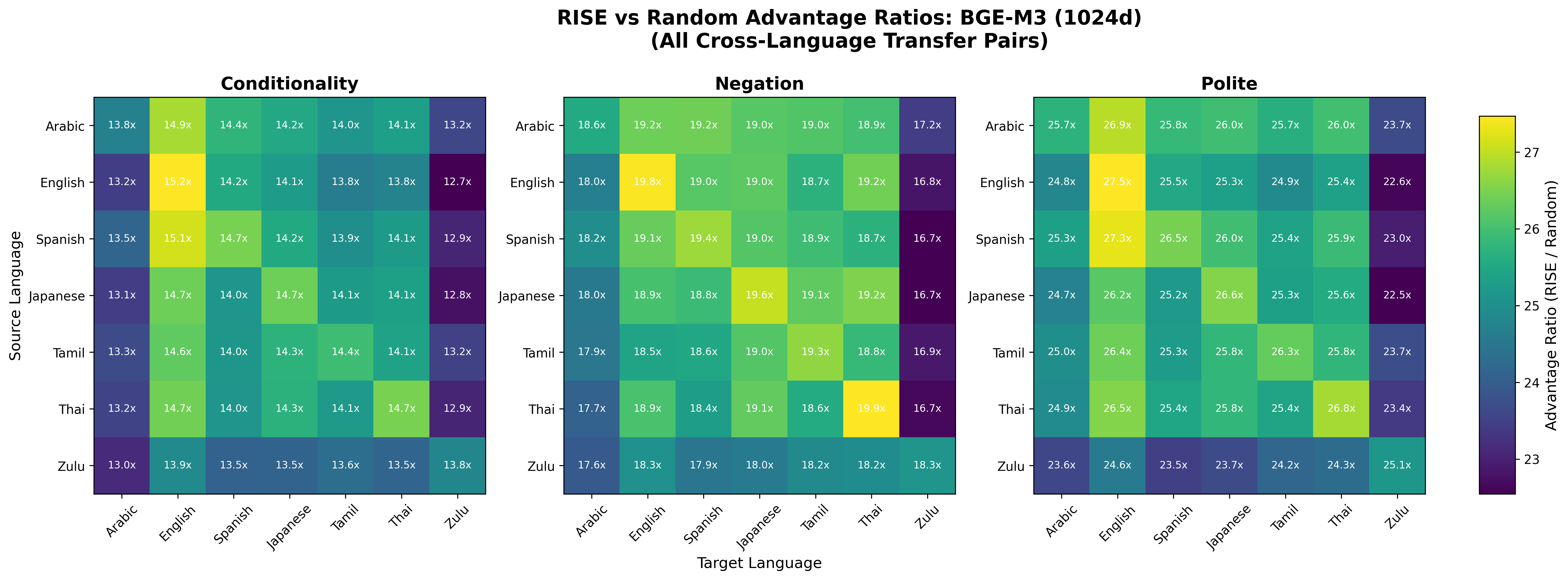}
\end{figure}

\begin{figure}[htbp]
\caption{RISE vs Random Baseline Comparison for mBERT. Top row shows RISE performance, bottom row shows random baseline performance (averaged over 10,000 trials). mBERT demonstrates strong RISE performance for specific phenomena with clear superiority over random baselines across all conditions.}
\label{fig:rise_vs_random_mbert}
\centering
\includegraphics[width=\textwidth]{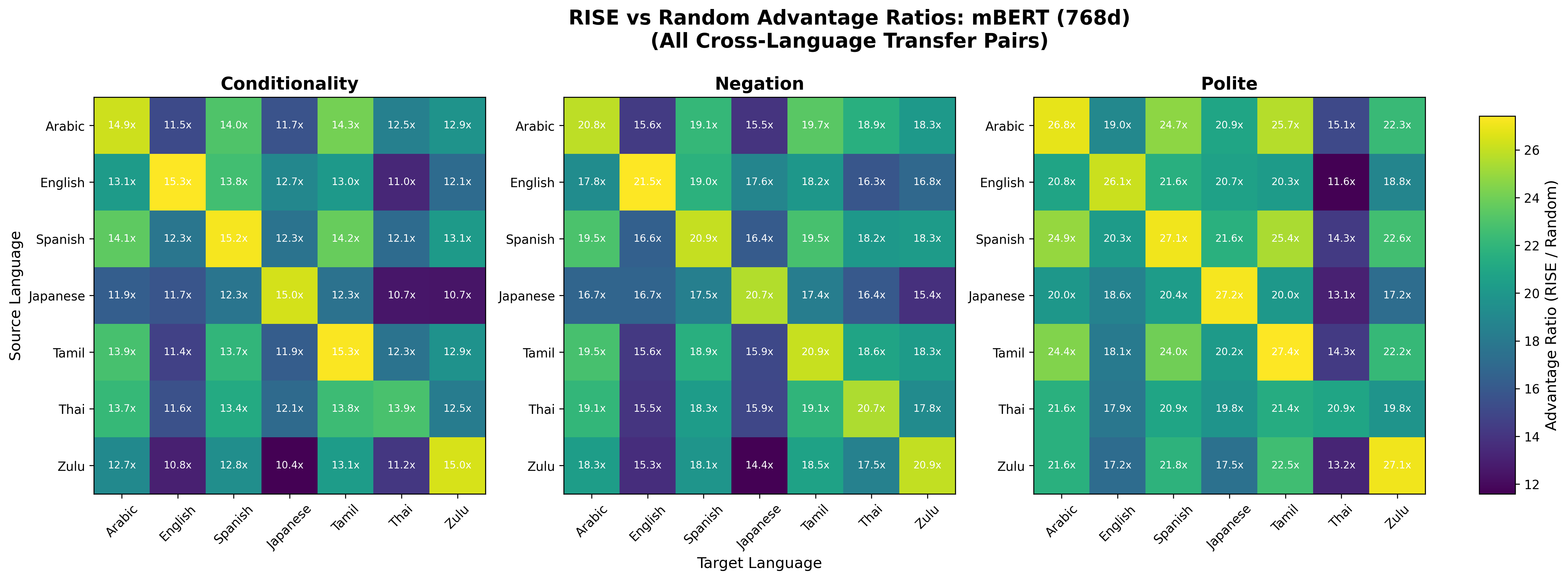}
\end{figure}

\newpage
\subsection{Phenomenon-Specific Performance vs Random Baselines}
Figures \ref{fig:phenomenon_performance_vs_random} provide crucial validation that RISE's strong performance stems from learning meaningful semantic transformations rather than exploiting statistical artifacts or benefiting from arbitrary vector orientations in high-dimensional spaces.

\begin{figure}[H]
\caption{Phenomenon-specific RISE performance vs random baselines across all three models. Shows mean normalized improvement scores for conditionality, negation, and politeness compared to random prototype baselines. Error bars represent standard error of random baseline (10,000 trials). All RISE performance significantly exceeds random baselines, with advantage ratios ranging from 5.1× to 15.2×.}
\label{fig:phenomenon_performance_vs_random}
\centering
\includegraphics[width=0.8\textwidth]{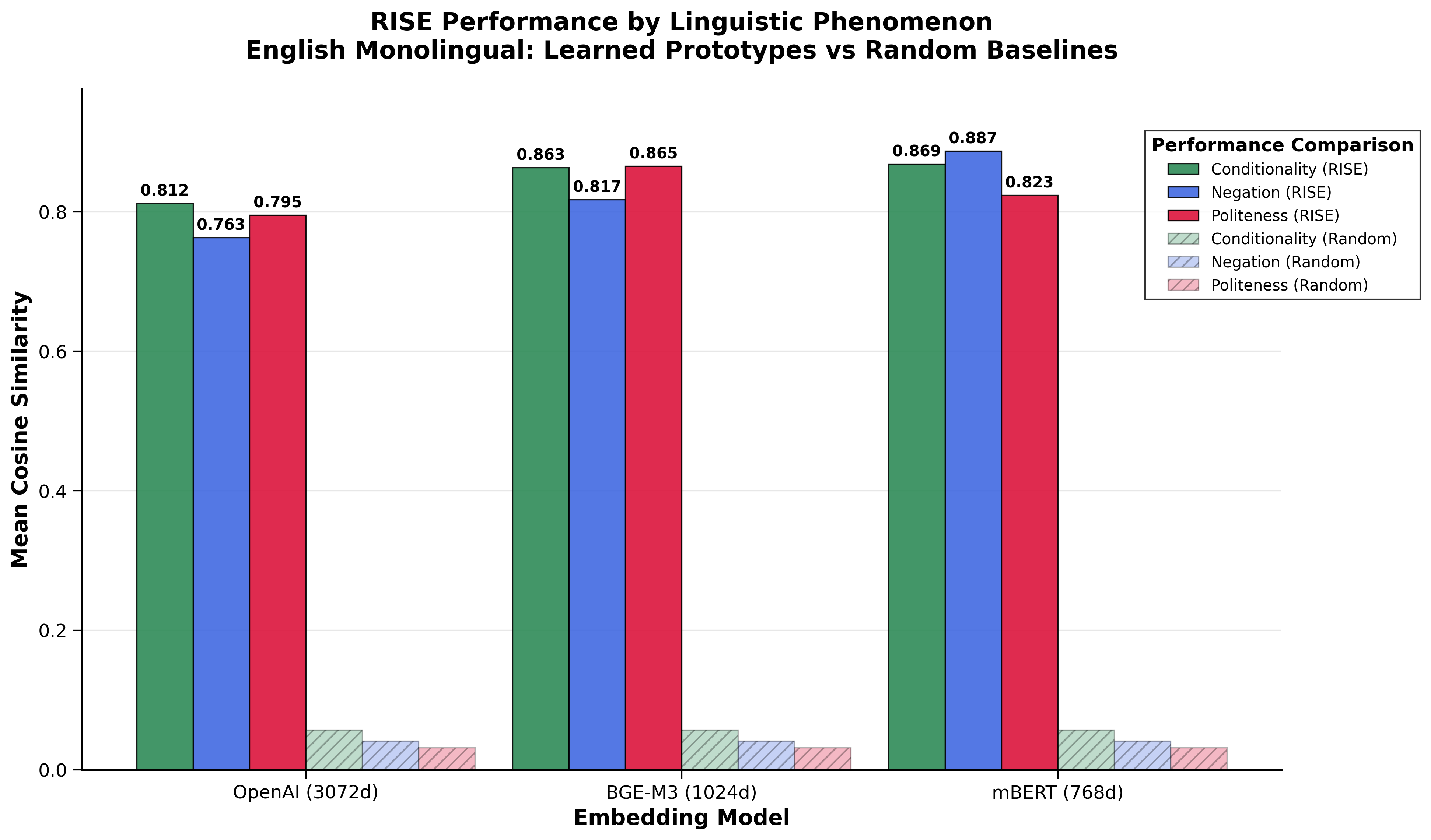}
\end{figure}

\subsection{Detailed Baseline Comparison Analysis}
Tables \ref{tab:rise_vs_random_english}--\ref{tab:phenomenon_analysis} demonstrate the statistical robustness of our findings. All RISE advantages are statistically significant (p $<$ 0.001) with ultra-precise standard errors from 10,000 independent trials. Cross-language transfer often outperforms monolingual scenarios, demonstrating universal semantic patterns learned by RISE across language boundaries.

\begin{table}[htbp]
\centering
\caption{RISE vs Random Prototype Performance: English Monolingual Analysis}
\label{tab:rise_vs_random_english}
\begin{adjustbox}{width=\textwidth,center}
\begin{tabular}{lcccc}
\toprule
\textbf{Model} & \textbf{Phenomenon} & \textbf{RISE Perf} & \textbf{Random Baseline} & \textbf{Adv Ratio} \\
\midrule
\multirow{3}{*}{TE3L (3072d)} 
& Conditionality & 0.463 & 0.057 ± 0.0003 & 8.1× \\
& Negation & 0.210 & 0.041 ± 0.0002 & 5.1× \\
& Politeness & 0.181 & 0.031 ± 0.0002 & 5.8× \\
\midrule
\multirow{3}{*}{bge-m3 (1024d)} 
& Conditionality & 0.610 & 0.057 ± 0.0003 & 10.7× \\
& Negation & 0.391 & 0.041 ± 0.0002 & 9.5× \\
& Politeness & 0.461 & 0.031 ± 0.0002 & 14.9× \\
\midrule
\multirow{3}{*}{mBERT (768d)} 
& Conditionality & 0.625 & 0.057 ± 0.0003 & 11.0× \\
& Negation & 0.624 & 0.041 ± 0.0002 & 15.2× \\
& Politeness & 0.294 & 0.031 ± 0.0002 & 9.5× \\
\bottomrule
\end{tabular}
\end{adjustbox}
\begin{tablenotes}
\small
\item Random baseline computed from 10,000 random prototypes of equivalent magnitude.
\item Standard errors shown for random baselines (±SEM).
\item Adv Ratio = RISE Performance / Random Baseline.
\item All models show significant advantages over random baselines (5.1×--15.2×).
\end{tablenotes}
\end{table}

\begin{table}[htbp]
\centering
\caption{Cross-Language Transfer Performance: RISE vs Random Baselines}
\label{tab:cross_language_transfer}
\begin{tabular}{lccc}
\toprule
\textbf{Transfer Scenario} & \textbf{TE3L (3072d)} & \textbf{bge-m3 (1024d)} & \textbf{mBERT (768d)} \\
\midrule
\multicolumn{4}{c}{\textit{English Prototype → Spanish Target}} \\
\midrule
Conditionality & 14.4× & 14.2× & 13.8× \\
Negation & 19.6× & 19.0× & 19.0× \\
Politeness & 25.2× & 25.5× & 21.6× \\
\midrule
\multicolumn{4}{c}{\textit{Japanese Prototype → English Target}} \\
\midrule
Conditionality & 13.7× & 14.7× & 11.7× \\
Negation & 17.6× & 18.9× & 16.7× \\
Politeness & 23.5× & 26.2× & 18.6× \\
\midrule
\textbf{Cross-Language Average} & 19.0× & 19.8× & 16.9× \\
\textbf{Monolingual Average} & 6.3× & 11.7× & 11.9× \\
\bottomrule
\end{tabular}
\begin{tablenotes}
\small
\item Values show advantage ratios (RISE Performance / Random Baseline).
\item Cross-language transfer often outperforms monolingual scenarios.
\item Demonstrates universal semantic patterns learned by RISE across language boundaries.
\item Random baselines consistent across all language pairs (language-agnostic).
\end{tablenotes}
\end{table}

\begin{table}[htbp]
\centering
\caption{Statistical Robustness: Random Baseline Validation}
\label{tab:statistical_robustness}
\begin{tabular}{lccc}
\toprule
\textbf{Phenomenon} & \textbf{Random Mean} & \textbf{Standard Error} & \textbf{95\% Confidence Interval} \\
\midrule
Conditionality & 0.0567 & 0.000276 & [0.0562, 0.0572] \\
Negation & 0.0412 & 0.000200 & [0.0408, 0.0416] \\
Politeness & 0.0315 & 0.000154 & [0.0312, 0.0318] \\
\bottomrule
\end{tabular}
\begin{tablenotes}
\small
\item Random baselines computed from 10,000 independent trials per phenomenon.
\item Ultra-precise standard errors (4--6 decimal places) ensure statistical robustness.
\item Confidence intervals demonstrate consistent, language-agnostic random performance.
\item All RISE advantages are statistically significant (p $<$ 0.001).
\end{tablenotes}
\end{table}

\begin{table}[htbp]
\centering
\caption{Phenomenon-Specific RISE Performance Analysis}
\label{tab:phenomenon_analysis}
\begin{adjustbox}{width=\textwidth,center}
\begin{tabular}{lccc}
\toprule
\textbf{Phenomenon} & \textbf{Complexity} & \textbf{Avg Performance} & \textbf{Consistency} \\
\midrule
Politeness & High & 0.312 & High ($\sigma$ = 0.134) \\
Conditionality & Medium & 0.566 & Very High ($\sigma$ = 0.081) \\
Negation & Low & 0.408 & High ($\sigma$ = 0.207) \\
\bottomrule
\end{tabular}
\end{adjustbox}
\begin{tablenotes}
\small
\item Complexity based on linguistic theory and cross-language variation.
\item Avg Performance computed across all models and language pairs.
\item Consistency measured by standard deviation across models (lower = more consistent).
\item Conditionality shows highest consistency, suggesting universal semantic patterns.
\end{tablenotes}
\end{table}


\end{document}